\documentclass[12pt]{article}
\usepackage{mathptmx}
\usepackage{amsmath,amssymb}
\usepackage{natbib}
\usepackage{dsfont}
\usepackage{xcolor}
\usepackage{gensymb}
\usepackage{rotating}
\usepackage{subcaption} 
\usepackage{booktabs}
\usepackage{tabularx}
\usepackage{array}
\usepackage{placeins}
\usepackage[hidelinks]{hyperref}
\usepackage[shortlabels]{enumitem}
\usepackage{graphicx} \usepackage{tikz}
\usetikzlibrary{arrows.meta,positioning}
\usepackage[margin=1in]{geometry}
\usepackage{amsthm}

\newtheorem{proposition}{Proposition}

\newtheorem{corollary}{Corollary}

\newtheorem{assumption}{Assumption}

\usepackage{libertine}
\usepackage{cleveref}
\usepackage[all,defaultlines=3]{nowidow}

\crefname{lemma}{lemma}{lemmas}
\Crefname{lemma}{Lemma}{Lemmas}
\crefname{theorem}{theorem}{theorems}
\Crefname{theorem}{Theorem}{Theorems}
\crefname{proposition}{proposition}{propositions}
\Crefname{proposition}{Proposition}{Propositions}
\crefname{corollary}{corollary}{corollaries}
\Crefname{corollary}{Corollary}{Corollaries}
\crefname{remark}{remark}{remarks}
\Crefname{remark}{Remark}{Remarks}
\crefname{assumption}{assumption}{assumptions}
\Crefname{assumption}{Assumption}{Assumptions}
\crefname{section}{Section}{Sections}
\Crefname{section}{Section}{Sections}
\crefname{subsection}{Section}{Sections}
\Crefname{subsection}{Section}{Sections}
\crefname{subsubsection}{Section}{Sections}
\Crefname{subsubsection}{Section}{Sections}
\crefname{equation}{Equation}{Equations}
\Crefname{equation}{Equation}{Equations}
\crefname{algorithm}{Algorithm}{Algorithms}
\Crefname{algorithm}{Algorithm}{Algorithms}
\crefname{figure}{Figure}{Figures}
\Crefname{figure}{Figure}{Figures}
\crefname{table}{Table}{Tables}
\Crefname{table}{Table}{Tables}
\crefname{appendix}{Appendix}{Appendices}
\Crefname{appendix}{Appendix}{Appendices}

\newcommand{\E}{\mathbb{E}}

\title{Causal Inference with Unstructured Treatments}

\author{
  Kevin Christian Wibisono\\
  Department of Statistics\\
  University of Michigan, Ann Arbor\\
  kwib@umich.edu
  \and
  Yixin Wang\\
  Department of Statistics\\
  University of Michigan, Ann Arbor\\
  yixinw@umich.edu
  }

\date{}

\begin{document}

\maketitle

\begin{abstract}
\noindent
Causal inference usually concerns a scalar treatment, yet in many problems the treatment is unstructured: a text, an image, or a sequence of clinical decisions. Consider an instructor writing a course description to attract more students: the treatment is the course description, and the outcome is enrollment. The standard target, the average treatment effect of fixing the treatment to one exact value versus another, runs into two problems. It cannot be estimated, because almost no exact description recurs across courses, leaving no comparable group from which to measure its effect; and it would be of little use even if it could, since no one wants every course to carry the same description. What the instructor actually wants to know is which features of a description raise enrollment, and which of those features can be acted on across many courses. To this end, we propose a causal query for unstructured treatments: the \emph{maximally influential feature} (MIF), the feature of the treatment that most strongly influences the outcome. We formalize the MIF as a binary feature of the treatment, defined by a feature-scoring function, constrained so that both of its values stay well populated, and chosen to maximize the causal effect it induces. Turning the feature on shifts the distribution of treatments toward those that display it, turning it off shifts away, and the MIF effect contrasts the two average potential outcomes. We study identification conditions for the MIF, develop algorithms to estimate it, and make it actionable through a nudging algorithm that revises a treatment along the MIF into an outcome-improving version. We illustrate the MIF algorithm across applications in text, image, and dynamic treatment sequences.
\end{abstract}

\section{Introduction}
\label{sec:intro}

An instructor is writing the catalog description for her course. She hopes that a better paragraph will draw more students. The university has records from past offerings: the course description, the subject and level of the course, the instructor, the meeting time, the prerequisites, and the final enrollment. \Cref{tab:course-running} shows an excerpt. The records span several subject areas, with multiple courses in each, and the courses differ in both content and how they are described. Some descriptions speak directly to students; others read like formal catalog entries. Enrollment varies with subject, content, and description style. These records can predict enrollment from a paragraph, but the instructor's question is causal: how should her course description be written so that enrollment would change?

\begin{table}[t]
\centering
\small
\caption{An excerpt of catalog records. The catalog contains several subject areas, with multiple courses in each. Description style varies across courses and may also be correlated with subject area. The causal question is which modifiable features of the description raise enrollment.}
\label{tab:course-running}
\begin{tabularx}{\textwidth}{@{}
  >{\raggedright\arraybackslash}p{0.21\textwidth}
  >{\raggedright\arraybackslash}p{0.15\textwidth}
  >{\raggedright\arraybackslash}X
  r@{}}
\toprule
\textbf{Course} & \textbf{Subject} & \textbf{Description} & \textbf{Enroll.} \\
\midrule
Machine Learning
& Computer Science
& \emph{How do computers learn from data? Build models that recognize patterns, make predictions, and improve with experience.}
& 86 \\
\addlinespace
Artificial Intelligence
& Computer Science
& \emph{From search and planning to language models, explore how machines reason, act, and adapt in complex environments.}
& 91 \\
\addlinespace
Regression Modeling
& Statistics
& \emph{Linear and generalized linear models, diagnostics, variable selection, model comparison, and applications. Prerequisite: Statistical Inference.}
& 34 \\
\addlinespace
Causal Inference
& Statistics
& \emph{When can data tell us what would have happened otherwise? Study experiments, observational studies, matching, weighting, and sensitivity analysis.}
& 51 \\
\addlinespace
Organic Chemistry II
& Chemistry
& \emph{Continuation of reaction mechanisms, spectroscopy, and the synthesis of polyfunctional compounds.}
& 31 \\
\addlinespace
Physical Chemistry
& Chemistry
& \emph{Thermodynamics, chemical equilibrium, kinetics, quantum structure, and molecular spectroscopy. Prerequisite: General Chemistry.}
& 27 \\
\addlinespace
Modern European History
& History
& \emph{From revolutions to world wars, trace how a continent remade itself and argue about why it happened.}
& 44 \\
\addlinespace
Medieval Worlds
& History
& \emph{A survey of political, religious, and social life in Europe and the Mediterranean from late antiquity to the fifteenth century.}
& 33 \\
\bottomrule
\end{tabularx}
\end{table}

\textbf{The treatment is an unstructured object.} Classical causal inference is usually written for scalar treatments: a drug is taken or not, a dose is low or high, a price is raised by one dollar.  Here the treatment is a paragraph. In other applications it may be an image, a video, or a sequence of clinical decisions. These problems are examples of causal inference with an \emph{unstructured} treatment. Let $A$ denote the treatment, $X$ the background variables we are willing to adjust for, and $Y$ the outcome. In the example, $A$ is the course description, $X$ includes the course subject, level, instructor, meeting time, prerequisites, and other fixed course information, and $Y$ is enrollment. The same setting arises beyond text: the treatment may be a product image whose outcome is sales, a medical image whose outcome is diagnosis, or a sequence of clinical decisions whose outcome is patient recovery.

\textbf{Why standard causal queries fail. }Standard causal queries such as the average treatment effect are poorly suited to this setting. For a binary treatment, the average treatment effect compares two regimes that the data usually populate well. With unstructured treatments like text, the direct analogue would compare the potential outcome under one exact paragraph to the potential outcome under another exact paragraph. Although a potential outcome $Y(a)$ is well defined for an unstructured treatment paragraph $a$, the target $\E\{Y(a)\}$ has two problems. First, it is usually not estimable, because most exact descriptions occur once, many plausible descriptions never occur, and the observed data provide little support for point-level comparisons. Second, even if it were estimable, it would answer the wrong question. Identifying the paragraph that leads to the best potential outcome would not imply that one could assign this same paragraph to every course. Rather, the practical question is coarser: What features of a description raise enrollment? Which of those features can the instructor use when writing her own description?

One way to simplify the problem is to replace the paragraph with a pre-specified low-dimensional summary. For example, one might group descriptions by topic and compare enrollment across topics. But this misses much of the features that may matter. Topic is often not the part of the treatment one can change: a Bayesian statistics course should not be rewritten as organic chemistry. The modifiable part may instead be the form of the description, whether it is concrete, direct, formal, terse, motivational, or example-driven. These features may not be known in advance, and they need not coincide with topics.

\textbf{The maximally influential feature.}
To this end, we propose a causal query that asks a different causal question: Which feature of the treatment has the largest causal effect? We call the answer the \emph{maximally influential feature} (MIF).  The MIF is the feature of an unstructured treatment that most strongly changes the outcome. 
We represent a candidate feature by a binary variable \(W_f\in\{0,1\}\), generated from a feature-scoring function \(f\) according to
\[
    W_f\mid A,X\sim\operatorname{Bernoulli}\{f(A,X)\}.
\]
The score \(f(A,X)\) measures how strongly treatment \(A\) displays the feature in background context \(X\). For the formality feature, for example, a feature score calibrated to human judgments can be interpreted as the proportion of raters who would classify the course description as formal, as opposed to informal, given the description and course context. Turning the feature on shifts the treatment distribution toward course descriptions with high formality scores; turning it off shifts the distribution toward course descriptions with low formality scores. The causal effect of the feature compares the average potential outcomes under these two regimes.

We constrain the feature-scoring function so that both values of \(W_f\) remain well represented conditional on \(X\), building overlap into the query. Among the functions that satisfy this constraint, the MIF selects the one with the largest causal effect. A budget determines how broad the feature may be: small budgets search for narrow, high-impact features, while larger budgets search for features shared by more course descriptions. \Cref{sec:causal-query} formalizes this query, gives its identification conditions, and develops algorithms for estimating the MIF.

\textbf{Interpreting and acting on the learned MIF.}
For the MIF to be useful, it must be both interpretable and actionable. After learning the MIF, we inspect treatments that receive high and low feature scores to understand what separates them. In the course example, this comparison may reveal that high-scoring descriptions are more concrete, more direct, or more example-driven than low-scoring descriptions. The same learned feature can also guide revision: we can move a description toward the high-MIF region and decode the revised representation back into language. Thus the MIF analysis does not only identify an outcome-relevant feature; it can also produce a concrete rewritten paragraph.

This revision step makes explicit a constraint that was implicit in the causal query. To act on a learned feature, the feature must correspond to something one could plausibly change. Otherwise the MIF may be statistically valid but practically unhelpful: it may identify variation in the treatment that predicts or affects the outcome, yet is not available as an intervention. This leads to the question of which parts of the treatment should be treated as modifiable.

\textbf{What features of the treatment are modifiable?}
Not every feature of a treatment is something one can or should change. A learned MIF may reveal a real causal contrast, but it is only useful as a recommendation if the feature corresponds to an available intervention. For example, if the MIF says that courses in one subject enroll better than courses in another, the finding may be informative but not actionable: an instructor cannot turn Bayesian statistics into organic chemistry. What she can change is how the same course is described. We therefore distinguish between parts of the treatment that are unmodifiable and parts that may be modified.

For text, this distinction is often between content and style. Content captures what the course is about: the topic, prerequisites, learning goals, and substantive material. Style captures how that content is expressed: whether the description is concrete or abstract, direct or indirect, formal or conversational, terse or detailed. We therefore restrict the MIF search to features of the modifiable representation. The revision step then changes only this modifiable component while reconstructing the full description with the original content intact. \Cref{sec:cont-style-sep} develops this construction, which focuses the MIF on recommendations a user could actually act on.

\textbf{Contributions.} The paper makes five contributions:
\begin{enumerate}
\item We propose a causal query for causal inference with unstructured treatments: the maximally influential feature (MIF),
\item We establish identification conditions for the MIF, and develop estimation algorithms.
\item We make the MIF actionable by introducing a nudging algorithm that revises a treatment along the learned feature.
\item We discuss practical considerations in using MIF, e.g. separating features of the unstructured treatment into modifiable ones and unmodifiable ones, in order to focus the MIF search on relevant causal questions.
\item We illustrate the MIC algorithm across applications with text-based treatments, image-based  treatments, and dynamic treatment sequences.
\end{enumerate}

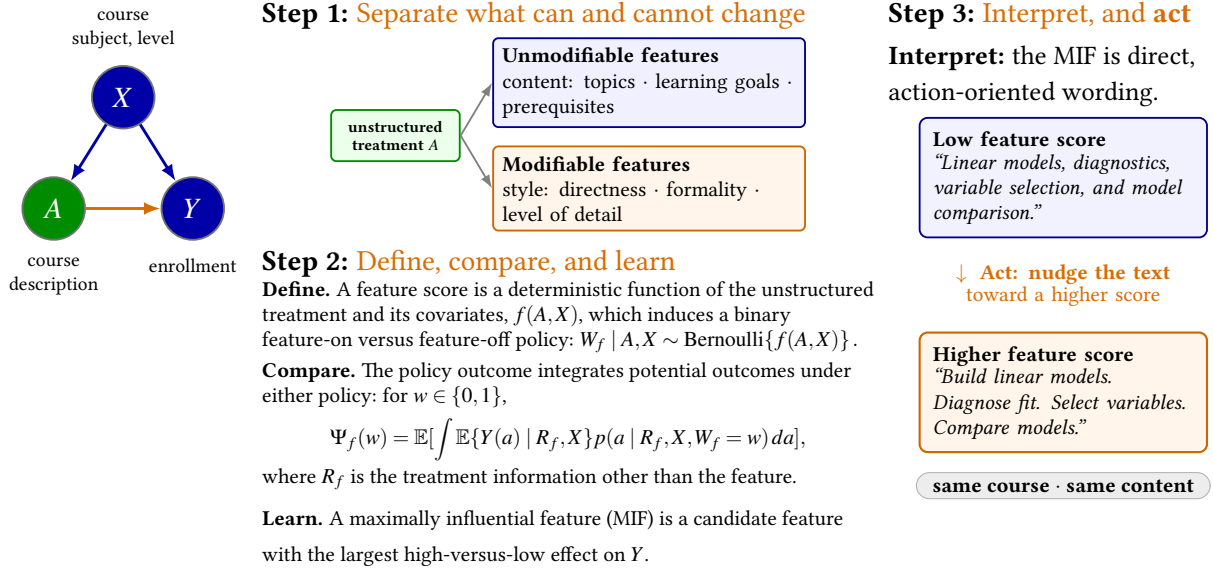
\begin{figure}[t]
\centering
\begin{minipage}[t]{0.20\textwidth}
\vspace{0pt}
\centering
\begin{tikzpicture}[scale=0.92, transform shape,
  var/.style={
    circle,
    draw=black!55,
    line width=0.7pt,
    minimum size=9mm,
    text=white,
    font=\bfseries
  },
  causal/.style={
    -{Latex[length=2.2mm,width=1.7mm]},
    line width=1.1pt,
    draw=blue!65!black
  },
  direct/.style={
    -{Latex[length=2.2mm,width=1.7mm]},
    line width=1.1pt,
    draw=orange!85!black
  },
  label/.style={font=\scriptsize, align=center}
]
  \node[var, fill=blue!65!black] (x) at (1.50,2.05) {\(X\)};
  \node[var, fill=green!55!black] (a) at (0.50,0.45) {\(A\)};
  \node[var, fill=blue!65!black] (y) at (2.50,0.45) {\(Y\)};

  \draw[causal] (x) -- (a);
  \draw[causal] (x) -- (y);
  \draw[direct] (a) -- (y);

  \node[label, above=1mm of x] {course\\subject, level};
  \node[label, below=1mm of a] {course\\description};
  \node[label, below=1mm of y] {enrollment};
\end{tikzpicture}
\end{minipage}\hfill
\begin{minipage}[t]{0.49\textwidth}
\vspace{0pt}
\raggedright
\small
\textbf{Step 1:}
\textcolor{orange!85!black}{Separate what can and cannot change}

\vspace{2pt}
\centering
\begin{tikzpicture}[
  source/.style={
    rounded corners=2pt,
    draw=green!55!black,
    fill=green!8,
    line width=0.7pt,
    align=center,
    font=\tiny\bfseries,
    text width=1.42cm,
    inner sep=4pt
  },
  fixed/.style={
    rounded corners=2pt,
    draw=blue!55!black,
    fill=blue!6,
    line width=0.7pt,
    align=left,
    font=\scriptsize,
    text width=3.88cm,
    inner sep=3.5pt
  },
  changeable/.style={
    rounded corners=2pt,
    draw=orange!80!black,
    fill=orange!8,
    line width=0.7pt,
    align=left,
    font=\scriptsize,
    text width=3.88cm,
    inner sep=3.5pt
  },
  split/.style={
    -{Latex[length=1.8mm,width=1.3mm]},
    draw=black!50,
    line width=0.7pt
  }
]
  \node[source] (treatment) at (0.72,0) {unstructured\\treatment \(A\)};
  \node[fixed, anchor=west] (fixed) at (2.00,0.70)
    {\textbf{Unmodifiable features}\\
     content: topics \(\cdot\) learning goals \(\cdot\) prerequisites};
  \node[changeable, anchor=west] (changeable) at (2.00,-0.70)
    {\textbf{Modifiable features}\\
     style: directness \(\cdot\) formality \(\cdot\) level of detail};

  \draw[split] (treatment.east) -- (fixed.west);
  \draw[split] (treatment.east) -- (changeable.west);
\end{tikzpicture}

\vspace{1pt}
\raggedright
\textbf{Step 2:}
\textcolor{orange!85!black}{Define, compare, and learn}

{\scriptsize
\setlength{\abovedisplayskip}{3pt}
\setlength{\belowdisplayskip}{2pt}
\noindent\textbf{Define.}
A feature score is a deterministic function of the unstructured treatment
and its covariates, $f(A,X)$, which induces a binary feature-on versus feature-off policy:
\(
  W_f\mid A,X
  \sim \operatorname{Bernoulli}
  \!\left\{f(A,X)\right\}.
\)\\
\vspace{2pt}
\noindent\textbf{Compare.}
The policy outcome integrates potential outcomes under either policy: for $w\in \{0,1\}$,
\[
  \Psi_f(w)
  =\mathbb E[
      \int \mathbb E\{Y(a)\mid R_f,X\}
      p(a\mid R_f,X,W_f=w)\,da],
\]
\noindent where \(R_f\) is the treatment information other than the feature.\\
\vspace{2pt}
\noindent\textbf{Learn.}
A maximally influential feature (MIF) is a candidate feature with the largest
high-versus-low effect on \(Y\).
}
\end{minipage}\hfill
\begin{minipage}[t]{0.28\textwidth}
\vspace{0pt}
\raggedright
\small
\textbf{Step 3:}
\textcolor{orange!85!black}{Interpret, and \textbf{act}}

\vspace{2pt}
{\footnotesize
\textbf{Interpret:} the MIF is direct, action-oriented wording.}

\vspace{3pt}
\centering
\begin{tikzpicture}[
  low/.style={
    rounded corners=3pt,
    draw=blue!55!black,
    fill=blue!5,
    line width=0.7pt,
    align=left,
    font=\scriptsize,
    text width=3.45cm,
    inner sep=5pt
  },
  high/.style={
    rounded corners=3pt,
    draw=orange!80!black,
    fill=orange!8,
    line width=0.8pt,
    align=left,
    font=\scriptsize,
    text width=3.45cm,
    inner sep=5pt
  },
  lock/.style={
    rounded corners=5pt,
    fill=black!7,
    draw=black!30,
    font=\scriptsize\bfseries,
    inner xsep=6pt,
    inner ysep=2.5pt
  }
]
  \node[low] (low) {\textbf{Low feature score}\\[-1pt]
    \emph{``Linear models, diagnostics, variable selection, and model comparison.''}};
  \node[font=\scriptsize, text=orange!85!black,
        text width=3.45cm, align=center,
        below=2mm of low] (nudge)
    {\(\boldsymbol{\downarrow}\)\; \textbf{Act: nudge the text}\\[-1pt]
     toward a higher score};
  \node[high, below=2mm of nudge] (high) {\textbf{Higher feature score}\\[-1pt]
    \emph{``Build linear models.\\ Diagnose fit. Select variables.\\ Compare models.''}};
  \node[lock, below=2.5mm of high] (lock)
    {same course \(\cdot\) same content};
\end{tikzpicture}
\end{minipage}
\caption{From an unstructured treatment to action. Course context \(X\) may affect both the unstructuredtreatment \(A\) and enrollment \(Y\). We separate unmodifiable content from modifiable style, define a deterministic feature score \(f(A,X)\), and learn the maximally influential feature (MIF) with the largest on-versus-off effect. We then interpret the learned feature and act on it by nudging only the modifiable text toward a higher score while holding content and context fixed.}
\label{fig:summary}
\end{figure}

\subsection{Related work}
\label{sec:related}

This work draws on five lines of related work.

\textbf{Text as covariates.}
A large literature uses text to improve causal estimates for treatments that are not themselves text; see \citet{feder2022causal} and \citet{keith2020text} for reviews. In this line of work, text is a record of context. It supplies controls, proxies, or representations for adjusting the effect of a separate, usually low-dimensional, treatment. Examples include matching and adjustment with text representations \citep{mozer2020matching,mozer2023leveraging,roberts2020adjusting,weld2022adjusting}, causally sufficient embeddings \citep{veitch2020adapting}, double machine learning with multimodal data \citep{klaassen2024doublemldeep}, outcome and propensity modeling from unstructured text \citep{zhou2026integrating}, AI-assisted regression adjustment with predictions from unstructured inputs \citep{arbour2026variance}, and methods for missing confounders or apparent overlap violations \citep{chen2024proximal,gui2023causal}. These papers ask how text can help estimate the effect of a treatment already chosen by the user. We ask the reverse question: when the text, image, or treatment history is itself the treatment object, which feature of that object should be changed?

\textbf{Text-derived and latent treatments.}
A second line moves the treatment inside the text. Topic-model and latent-treatment methods discover interpretable treatment groups from documents \citep{ahrens2021bayesian,egami2022make,fong2016discovery,fong2023causal}. Other work estimates the effects of pre-specified linguistic properties, tone, or classifier-derived text treatments \citep{pryzant2017predicting,pryzant2021causal,sridhar2019estimating,wood2018challenges}. Recent works also study the challenge of overlap in text-based treatment.  \citet{gui2023causal} reshape Bidirectional Encoder Representations from Transformers (BERT) embeddings to encode confounding while preserving overlap, and \citet{tierney2025design} argue that representation-based adjustment can still create overlap bias when the representation encodes the treatment itself. All of these approaches involve text in the treatment, but the specific treatment of interest is typically fixed before estimation: a topic, a label, a classifier output, or a named linguistic attribute. The MIF instead learns the treatment feature of interest from data.

The closest work to ours is the GenAI-powered inference of \citet{imai2025gpi,imai2024causal,nakamura2026dynamic}. \citet{imai2024causal} study texts as treatments by using a generative model to produce or reproduce treatment texts and then using the model's internal representation for causal effect estimation. \citet{imai2025gpi} broaden the framework to causal and predictive inference with unstructured data, including text and images. \citet{nakamura2026dynamic} extend it to dynamic unstructured treatments, estimating effects of treatment-feature sequences and their positions.

Our work shares their starting point: the treatment is unstructured and represented through learned features. The target is different. In the GenAI-powered inference line, the user begins with a treatment feature, or a sequence of features, whose effect is to be estimated. In our setting, the feature is not specified in advance. The MIF uses the data to discover a feature that is outcome-relevant, and actionable. Their representation is used mainly to construct a deconfounder or low-dimensional structure for estimating effects of specified features. Our representation defines a searchable space of modifiable features over which the causal contrast objective is optimized.

\textbf{High-dimensional treatments and kernel causal learning.}
Related work also studies causal inference with high-dimensional or structured treatments outside the specific setting of text or unstructured data. Kernel and neural mean-embedding methods estimate causal functions, dose-response curves, and adjustment functionals in rich treatment spaces \citep{singh2024kernel,xu2022neural}. Kernel methods have also been used for proximal causal learning with proxies \citep{mastouri2021proximal}, and neural proxy learning has been applied to confounded bandit policy evaluation \citep{xu2021deep}. More recent work studies treatment representations for instrumental-variable regression \citep{lin2025ivrepr} and inference with high-dimensional treatments more broadly \citep{kramer2026highdim}. These papers learn representations or functions for a downstream causal estimand. Our goal is different: we search for the estimand-defining feature itself.

\textbf{Natural-language interventions and modifiable features.}
Recent work formalizes interventions on language more directly. \citet{lin2023texttransport} estimate causal effects of changing linguistic attributes under distribution shift, while \citet{lin2025isolated} study isolated causal effects of focal language interventions and the difficulty of preserving non-focal language. \citet{lin2024cpo} frame language-model optimization from preference outcomes as a causal problem. Work on AI-generated treatments studies how treatment variants can be represented, compared, assigned, or annotated efficiently \citep{nwankwo2025batch,shi2025semantic}. Complementary work evaluates language-model generations under confounded model choice, combining a small randomized experiment with an offline simulator to identify causal model values and using observational logs to improve estimation \citep{jin2026partial}. Related work also uses text rationales to improve robustness in causal effect estimation from text \citep{zhang2025text}. These papers share our concern that natural-language changes are rarely one-dimensional: changing one feature of a text can also change its topic, meaning, and support. Our MIF algorithm addresses this by defining the content to be preserved and searching only within the modifiable part of the representation.

\textbf{Stochastic interventions and content-style representations.}
The causal target of the MIF is closest in spirit to stochastic interventions. Rather than assign every unit to one treatment value, a stochastic intervention shifts or reweights a treatment distribution while staying inside support. Incremental propensity-score interventions use this idea to avoid the positivity problems of point interventions \citep{kennedy2018incremental}. We use the same potential-outcomes logic: the intervention changes the distribution over which existing potential outcomes $Y(a)$ are averaged; it does not change the potential outcomes itself.

For unstructured treatments, support is not enough. The intervention must also respect which part of the treatment is modifiable and which part is not. With text-based treatments, this connects to representation learning for content and style. Work in natural language processing (NLP) studies how to build style embeddings that are less entangled with topic or content \citep{patel2023learning,patel2024styledistance,wegmann2022same}; related causal work gives formal desiderata for representation learning, including disentanglement and non-spuriousness \citep{wang2024desiderata}. Existing style representations often begin with a named style, such as authorship, formality, sentiment, or stylometric signature. We use a different division. For each application, the user defines the component of the unstructured treatment that must stay unmodifiable; the remaining variation is the space in which an intervention may act. The MIF then searches that space for the feature most strongly tied to the causal contrast objective. The representation is therefore not only descriptive. It is a constraint on the causal changes our MIF algorithm is allowed to consider.

\subsection{Organization of the paper}
The remainder of the paper is organized as follows. \Cref{sec:mif} develops causal inference with unstructured treatments: \Cref{sec:causal-query} introduces the setup and the MIF causal query; \Cref{sec:causal-identification,sec:causal-estimation} establish identification and develop estimation; \Cref{sec:interpret} explains how to interpret and act on a learned feature; and \Cref{sec:cont-style-sep} addresses modifiability through content--style separation. \Cref{sec:empirical_studies} presents empirical studies on text, images, and dynamic treatment sequences, together with experiments that nudge treatments along the learned MIF direction. \Cref{sec:disc} concludes.

\section{Causal inference with unstructured treatments}
\label{sec:mif}

This section formulates a causal query for unstructured treatments, develops its identification conditions and an estimation algorithm, and explores its interpretation and practical considerations.  The running example is the course catalog.  A course description is the treatment, enrollment is the outcome, and course information such as subject, level, instructor, prerequisites, and meeting time is the context.  The instructor is interested in what features of a course description can increase enrollment.

\subsection{Setup, potential outcomes, and the causal query}
\label{sec:causal-query}

\textbf{Observed data and potential outcomes.}
We observe i.i.d.\ data
\[
    Z_i=(X_i,A_i,Y_i),\qquad i=1,\ldots,n,
\]
where \(X_i\in\mathcal X\) contains pre-treatment covariates, \(A_i\in\mathcal A\) is a numerical representation of the full unstructured treatment, and \(Y_i\in\mathbb R\) is the outcome.  For each full treatment value \(a\in\mathcal A\), let \(Y(a)\) denote the potential outcome under \(a\).  In the course example, \(A\) represents the complete course description, \(X\) contains fixed course information, and \(Y(a)\) is the enrollment that would result if the course used the description represented by \(a\).

\textbf{Numerical representation of unstructured treatments.}
The raw treatment may be a paragraph, image, audio clip, or another object that is not naturally a vector.  Write \(O\) for this raw treatment and represent it numerically as
\[
    A=\mathcal E(O).
\]
The map \(\mathcal E\) may be a language-model embedding, a set of hand-coded variables, or a representation designed to isolate a modifiable component such as style, tone, or specificity.  Although \(A\) is numerical, it must still identify the full treatment that indexes \(Y(a)\).  We therefore assume that \(\mathcal E\) is one-to-one on the observed raw-treatment support.  This condition prevents hidden treatment versions: each supported value \(a\) corresponds to one raw treatment, so \(Y(a)\) is unambiguous.  The map need not cover the ambient representation space; an arbitrary vector may not correspond to a coherent treatment.  The policies below avoid such vectors by reweighting only supported values of \(A\).  In the running example, \(O\) is the raw course description, \(A\) is its text representation, and \(X\) records fixed course facts such as subject and prerequisites.

\textbf{Feature-scoring functions.}
For overlap margin~\(\epsilon\), we represent a candidate feature by a function
\[
    f:\mathcal A\times\mathcal X\longrightarrow[\epsilon,1-\epsilon],
    \qquad 0<\epsilon<1/2.
\]
The value \(f(a,x)\) is the feature score: a large value means that treatment \(a\) strongly displays the feature in context \(x\).  For example, \(f(a,x)\) may score how formal a course description is.  The bounds ensure that both feature-on and feature-off policies are well-defined.

\textbf{Binary feature intervention.}
Each feature score defines two stochastic interventions through an auxiliary binary feature variable
\[
    W_f\mid A,X
    \sim
    \operatorname{Bernoulli}\{f(A,X)\}.
\]
The policy indexed by \(W_f=1\) is feature-on, and the policy indexed by \(W_f=0\) is feature-off.  For a formality feature calibrated to human judgments, \(f(a,x)\) can be interpreted as the proportion of raters who would classify course description \(a\) as formal rather than informal in context \(x\).  The value \(W_f=1\) then represents a formal classification by a randomly selected rater.  The construction does not threshold the graded score; it uses the score as the probability of the feature-on label.

\textbf{Remainder-preserving feature policies and their outcomes.}
For each candidate feature-scoring function \(f\), let \(R_f=r_f(A,X)\) denote the remainder: all properties of the treatment other than the feature scored by \(f\).  These are the properties that the feature intervention does not touch and holds fixed.  If \(f\) scores formality, for example, \(R_f\) records the course description's semantic content, facts, length, directness, politeness, examples, and every other property besides formality.  Holding \(R_f\) and \(X\) fixed then lets the feature intervention change formality while preserving these other properties and the course context. \Cref{sec:causal-identification} formalizes this notation \(R_f\) and explains when the comparison contains feature variation.

For simplicity, assume that the conditional densities or probability mass functions below exist on their conditional supports and that all displayed expectations are finite.  Write \(p(a\mid r,x)\) for the observed conditional density or mass of \(A\) given \(R_f=r\) and \(X=x\).
Within a remainder--context stratum, we define the \textit{feature-on policy} as \(p(a\mid r,x,W_f=1)\), and the \textit{feature-off policy} as \(p(a\mid r,x,W_f=0)\).  For \(w\in\{0,1\}\), Bayes' rule gives
\[
\begin{aligned}
    p(a\mid r,x,W_f=w)
    &=
    \frac{
        \mathbb P(W_f=w\mid A=a,R_f=r,X=x)
        \,p(a\mid r,x)
    }{
        \mathbb P(W_f=w\mid R_f=r,X=x)
    }\\[4pt]
    &=
    \begin{cases}
    \displaystyle
    \frac{f(a,x)}{\pi_f(r,x)}p(a\mid r,x),
        & w=1,\\[8pt]
    \displaystyle
    \frac{1-f(a,x)}{1-\pi_f(r,x)}p(a\mid r,x),
        & w=0,
    \end{cases},
\end{aligned}
\]
where
\[
    \pi_f(r,x)
    =
    \mathbb P(W_f=1\mid R_f=r,X=x)
    =
    \mathbb E\{f(A,X)\mid R_f=r,X=x\},
\]
is the conditional probability that the feature is on.
Under the binary construction, the label probability in the numerator is \(f(a,x)\) when \(w=1\) and \(1-f(a,x)\) when \(w=0\).  Averaging the same label probability over \(p(a\mid r,x)\) gives the denominator \(\pi_f(r,x)\) or \(1-\pi_f(r,x)\).

The binary construction therefore produces weighted stochastic policies over supported descriptions.  In the course example, fix the context \(x\) and a remainder \(r\) that includes the description's content.  Relative to \(p(a\mid r,x)\), the feature-on policy upweights descriptions whose formality scores exceed the stratum average and downweights those whose scores fall below it.  The feature-off policy reverses this tilt.  The normalizers make both policies integrate to one, and the score bounds keep their weights positive and finite.

The feature-on/off policies also have a simple updating interpretation.  Imagine drawing a course description from \(p(a\mid r,x)\) and asking one randomly selected rater whether it is formal.  Under the calibrated-rater interpretation, \(f(a,x)\) is the probability of a formal rating.  Conditioning on a ``formal'' rating therefore updates the distribution toward descriptions more likely to receive that rating; conditioning on an ``informal'' rating updates it in the opposite direction.  These updated distributions are the feature-on and feature-off policies.  The size of the tilt is how much this one rating changes the distribution over descriptions.  It changes which full description is likely to be drawn while preserving the remainder and context.

The resulting conditional policy outcome is
\[
    \Psi_f(w\mid r,x)
    =
    \int
    \mathbb E\{Y(a)\mid R_f=r,X=x\}
    p(a\mid r,x,W_f=w)\,da.
\]
For a discrete treatment, the integral is a sum.  The population policy outcomes retain each course's remainder and context and then average over the population:
\[
    \Psi_f(w\mid x)
    =
    \mathbb E\{\Psi_f(w\mid R_f,x)\mid X=x\},
    \qquad
    \Psi_f(w)
    =
    \mathbb E\{\Psi_f(w\mid R_f,X)\}.
\]
In the running example, \(\Psi_f(1)\) is expected enrollment when courses retain their content and the other writing properties recorded in \(R_f\), while descriptions are redistributed toward higher formality.  The value \(\Psi_f(0)\) uses the corresponding redistribution toward lower formality.  The random policy draw is made independently of the potential outcomes conditional on \((R_f,X)\).

\textbf{Feature budget.}
Let \(\mathcal F\) denote the chosen class of admissible feature-scoring functions.  The average feature mass, \(\mathbb E\{f(A,X)\}\), controls the breadth of the feature.  To compare features at the same breadth, fix \(b\in(\epsilon,1-\epsilon)\) and let
\[
    \mathcal F_b
    =
    \left\{
        f\in\mathcal F:\mathbb E\{f(A,X)\}=b
    \right\}.
\]
A small \(b\) asks for a sharply defined pattern concentrated in a narrow region of the description space, such as whether a piece of course description includes an exclamation mark.  A larger \(b\) asks for a broader pattern shared by more courses.  Fixing \(b\) keeps these searches comparable.  Without this constraint, a flexible learner can lower the average feature mass to isolate a tiny region; for near-binary scores, \(b\) directly controls the fraction of descriptions assigned a high score.

\textbf{The raw causal contrast and the maximally influential feature (MIF).}
For \(f\in\mathcal F_b\), the \textit{raw causal contrast} and its fixed-budget optimizer set are
\[
    \Delta_X(f)=\Psi_f(1)-\Psi_f(0),
    \qquad
        \mathcal F_{b,\Delta}^*
    =
    \arg\max_{f\in\mathcal F_b}
    \Delta_X(f).
\]
The contrast compares expected enrollment under the feature-on and feature-off policies.  Maximizing it over \(\mathcal F_b\) answers the most literal fixed-budget query: among features with the same average feature mass, which one produces the largest policy contrast? We call each \(f_{b,\Delta}^*\in\mathcal F_{b,\Delta}^*\) a \textit{fixed-budget maximally influential feature (MIF)} at budget \(b\).

\textbf{Comparing across budgets with a variance-weighted causal contrast objective.}
The raw causal contrast objective is natural for evaluating a fixed feature at a fixed budget.  It is less convenient for searching across budgets; the optimized raw contrast is not on a common scale across budgets.  A narrow feature can have a large raw contrast by isolating a small and unusual region, while a broader feature can have a smaller raw contrast but capture more outcome signal.  For a search that ranges over budgets, we propose the \textit{variance-weighted causal contrast} objective,
\[
\begin{aligned}
    \widetilde\Delta_X(f)
    =
    \mathbb E\Big[
        \pi_f(R_f,X)\{1-\pi_f(R_f,X)\}
        \times
        \{\Psi_f(1\mid R_f,X)-\Psi_f(0\mid R_f,X)\}
    \Big].
\end{aligned}
\]
This objective uses the same feature-level policy outcomes as the raw contrast but changes their scale.  Because \(\pi_f(R_f,X)=\mathbb P(W_f=1\mid R_f,X)\), the factor \(\pi_f(R_f,X)\{1-\pi_f(R_f,X)\}\) is the conditional variance of \(W_f\); it downweights a remainder--context stratum when either feature label is rare.  The outer expectation averages over the population distribution of \((R_f,X)\), so remainder--context strata containing more courses contribute more.  In the course example, the objective is large when tilting toward formality raises enrollment, both formal and informal ratings are plausible among otherwise comparable descriptions, and such courses occur frequently.  \Cref{app:mif-original} explains in details how \(b\) controls feature breadth and how the two objective scales differ in a pooled course-description setting.

For the cross-budget search, use a chosen search class of feature functions \(\mathcal F\) itself, without a budget constraint.  The population MIF set is
\[
    \mathcal F^*
    =
    \arg\max_{f\in\mathcal F}
    \widetilde\Delta_X(f).
\]
We call each \(f^*\in\mathcal F^*\) a \textit{maximally influential feature (MIF)}.  Its breadth is \(\mathbb E\{f^*(A,X)\}\), so the MIF search chooses both the feature and its breadth.

\subsection{Causal identification of the MIF}
\label{sec:causal-identification}

The policy above defines the causal target.  The identification problem is to connect the feature-level potential outcomes to observed data. Below we present assumptions required for the identification of the feature-on and feature-off policy outcomes.

\begin{assumption}[SUTVA and consistency]
\label{ass:mif-sutva}
For every full treatment \(a\in\mathcal A\), the potential outcome \(Y(a)\) is well-defined, with no interference across units and no hidden versions of \(A\).  If the observed treatment is \(A=a\), then
\[
    Y=Y(a).
\]
\end{assumption}

In the course example, observed enrollment equals the potential enrollment under the complete description that was used.

\begin{assumption}[Full-treatment unconfoundedness]
\label{ass:mif-exchangeability}
For every full treatment \(a\in\mathcal A\),
\[
    Y(a)\perp A\mid X.
\]
\end{assumption}

This assumption asks \(X\) to contain the common causes of course descriptions and enrollment, such as subject, level, instructor, prerequisites, and meeting time.

\begin{assumption}[Feature overlap within remainder--context strata]
\label{ass:mif-overlap}
For every candidate feature \(f\) in the search class and every supported
remainder--context pair \((r,x)\),
\[
    \epsilon
    \leq
    \pi_f(r,x)
    :=
    \mathbb E\{f(A,X)\mid R_f=r,X=x\}
    \leq
    1-\epsilon ,
\]
for some \(0<\epsilon<1/2\).

Here \(R_f\) denotes the remainder associated with feature \(f\): the
properties of the treatment that are held fixed when changing the
feature scored by \(f(A,X)\). Formally, \(R_f\) is defined through a
feature coordinate system in which
\(
    A
    \longleftrightarrow
    \{f(A,X),R_f,X\}
\)
is one-to-one for each $X=x$. We discuss how such feature coordinate systems can be
constructed in practice in \Cref{sec:cont-style-sep}.
\Cref{sec:remainder-definition} explains how they ensure that the resulting policy
outcome is well defined.
\end{assumption}

\Cref{ass:mif-overlap} requires that, among treatments with the same remainder
and context, both higher and lower feature values occur with positive
probability. It is the relevant overlap condition because the
feature-on and feature-off policies reweight the conditional treatment
distribution within each remainder--context stratum. The assumption
guarantees that the normalizing constants
$\pi_f(r,x)$ and $1-\pi_f(r,x)$
are positive. It rules out a feature whose on or off comparison would require course descriptions that never occur among comparable courses.

Under these assumptions, the feature-on and feature-off policy outcomes are identified from observed data.  The following proposition establishes the result.

\begin{proposition}[Identification of feature-on/off outcomes and causal contrasts]
\label{prop:iden}
For a feature-scoring function \(f\) and its feature-specific remainder \(R_f\), under \Cref{ass:mif-sutva,ass:mif-exchangeability,ass:mif-overlap}, the feature-on and feature-off policy outcomes are identified by
\[
    \Psi_f(1)
    =
    \mathbb E\left[
        \frac{f(A,X)Y}{\pi_f(R_f,X)}
    \right],
    \qquad
    \Psi_f(0)
    =
    \mathbb E\left[
        \frac{\{1-f(A,X)\}Y}
        {1-\pi_f(R_f,X)}
    \right].
\]
Hence, the raw causal contrast $\Delta_X(f)=\Psi_f(1)-\Psi_f(0)$ is identified.
The variance-weighted causal contrast is also identified:
\[
    \widetilde\Delta_X(f)
    =
    \mathbb E\left[
        f(A,X)\{Y-m_f(R_f,X)\}
    \right],
    \qquad \text{ where } \qquad
    m_f(r,x)=\mathbb E(Y\mid R_f=r,X=x).
\]
\end{proposition}

The proof of \Cref{prop:iden} is in \Cref{app:iden-proof}. The identification formula of the variance-weighted causal contrast admits an intuitive interpretation. In the course example, \(m_f(R_f,X)\) is expected enrollment among courses with the same context and the same values of the content and writing properties recorded in \(R_f\).  The variance-weighted causal contrast thus measures the residual enrollment signal associated with the feature scored by \(f\) within those comparisons.

\textbf{Using a shared remainder for every feature.} \Cref{prop:iden} uses a feature-specific remainder \(R_f\), which
contains all treatment properties other than the candidate feature scored by
\(f\).  In principle, \(R_f\) changes with \(f\): the remainder for directness
must include formality, while the remainder for formality must include
directness.  Searching over many candidate features would therefore require a
different outcome regression
$m_f(R_f,X)=\mathbb E(Y\mid R_f,X)$
for each feature.

In many applications, however, candidate features describe separate variation
in presentation or style, while a common representation captures the
underlying content.  We can then use a shared remainder \(R\) across features.
For text, \(R\) may represent semantic content while \(f\) scores properties
such as formality, directness, or brevity.  For images, \(R\) may represent
object identity or biological content while \(f\) scores properties such as
rotation, blur, or contrast.  The required condition is that changing the
candidate feature does not systematically change treatment properties omitted
from \(R\).  Equivalently, conditional on \(R\) and \(X\), the feature label
\(W_f\) should contain no information about the remaining treatment variation.

In the course example, if \(R\) captures the semantic content of a description,
this condition means that making a description more formal should not also
change its brevity, directness, or other writing properties after conditioning
on content and course context.  In an image application, changing a feature
such as rotation should not systematically change object identity or other
visual properties represented outside the rotation feature.  Under this
separation condition, the feature-specific policy can be recovered using the
shared remainder \(R\), allowing one common outcome regression to be used for
all candidate features.

Specifically, for a shared remainder \(R\), write the feature-specific
remainder as
$R_f=(R,U_f),$
where \(U_f\) contains the additional treatment properties that would be
preserved by the feature-specific intervention but are not represented in the
shared remainder.  In an image application, for example, \(R\) can record
object or biological content, while \(U_f\) records the other visual
properties that a rotation or blur intervention should preserve.

\begin{assumption}[Shared-remainder mean balance]
\label{ass:mif-fixed-residual}
For every candidate feature-scoring function \(f\),
\[
\mathbb E\{f(A,X)\mid R_f,X\}
=
\mathbb E\{f(A,X)\mid R,X\}
\qquad\text{almost surely}.
\]
\end{assumption}

This assumption states that, after conditioning on content and context,
receiving a feature-on label provides no information about the remaining
feature-specific properties.  In the course example, fix the course content
\(R\) and context \(X\).  Learning whether a randomly selected rater called a
description formal provides no information about its length, directness,
politeness, or the other properties in \(U_f\).  Thus the descriptions
receiving formal and informal labels have the same distribution of these
other properties.  Conditioning on either label can then tilt formality
without changing the distribution of \(U_f\).  This preservation is in
distribution; it does not construct a paired rewrite that fixes \(U_f\)
description by description.

\begin{corollary}[Identification with a shared remainder]
\label{cor:mif-fixed-residual}
Under \Cref{ass:mif-sutva,ass:mif-exchangeability,ass:mif-overlap,ass:mif-fixed-residual}, the
feature-specific policy outcomes can be identified using the shared remainder:
\[
\Psi_f(1)
=
\mathbb E\left[
\frac{f(A,X)Y}
{\mathbb E\{f(A,X)\mid R,X\}}
\right],
\qquad
\Psi_f(0)
=
\mathbb E\left[
\frac{\{1-f(A,X)\}Y}
{1-\mathbb E\{f(A,X)\mid R,X\}}
\right].
\]
Moreover,
\[
\widetilde\Delta_X(f)
=
\mathbb E\left[
f(A,X)\{Y-m(R,X)\}
\right],
\qquad
\text{where}
\qquad
m(r,x)=\mathbb E(Y\mid R=r,X=x).
\]
\end{corollary}

The proof of \Cref{cor:mif-fixed-residual} is in \Cref{app:mif-fixed-residual-proof}. The corollary is an equality of population treatment laws and policy
outcomes.  It says that a coarser fixed \(R\) need not retain omitted properties for the
same course under a particular policy draw.  
\Cref{ass:mif-fixed-residual} lets a coarser \(R\) recover the
feature-specific target after population averaging.

If \Cref{ass:mif-fixed-residual} fails, reweighting within \(R\)
still defines and, under \Cref{ass:mif-sutva,ass:mif-exchangeability}, identifies a supported stochastic policy on
the full treatment.  It may, however, change other treatment properties
together with the focal feature.  It is then a bundled policy rather than the
feature-specific intervention.

From this point forward, assume that
\Cref{ass:mif-fixed-residual} holds for every candidate feature in
the search class.  We therefore use the shared-remainder representation
\[
\widetilde\Delta_X(f)
=
\mathbb E\left[
f(A,X)\{Y-m(R,X)\}
\right],
\]
where one common outcome regression can be used for all candidate features.

\begin{proposition}[Threshold form of the budgeted MIF]
\label{prop:mif-optimizer-main}
Define the fixed-budget optimizer set as
\[
\mathcal F_{b,\widetilde\Delta}^*
=
\arg\max_{f\in\mathcal F_b}
\widetilde\Delta_X(f).
\]
Choose \(f_{b,\widetilde\Delta}^*
\in\mathcal F_{b,\widetilde\Delta}^*\), and define the conditional residual-outcome score
\[
\eta(a,x)
=
\mathbb E\{Y-m(R,X)\mid A=a,X=x\}.
\]
When \(\mathcal F_b\) is otherwise unrestricted
apart from the score bounds and budget,
the variance-weighted MIF assigns \(1-\epsilon\) to treatment--context pairs with the
largest values of \(\eta(a,x)\), assigns \(\epsilon\) to pairs with the
smallest values, and may assign an intermediate value at the budget threshold.
Equivalently, for budget threshold~\(\lambda_b\),
\[
f_{b,\widetilde\Delta}^*(a,x)=1-\epsilon
\quad\text{if }\eta(a,x)>\lambda_b,
\qquad
f_{b,\widetilde\Delta}^*(a,x)=\epsilon
\quad\text{if }\eta(a,x)<\lambda_b.
\]
The value at equality is chosen to satisfy the budget.  If
\(\mathbb P\{\eta(A,X)=\lambda_b\}=0\), the optimizer is unique up to
almost-sure equality. Without a fixed budget, the threshold is zero, and uniqueness holds if \(\mathbb P\{\eta(A,X)=0\}=0\).
\end{proposition}

The proof of \Cref{prop:mif-optimizer-main} is in \Cref{sec:fixed-budget-optimizer-proof}. \Cref{prop:mif-optimizer-main} gives a direct
interpretation of the MIF.  After adjusting enrollment for the shared remainder and
course context, the MIF \(f_{b,\widetilde\Delta}^*\) assigns high feature scores to
descriptions with the largest expected residual enrollment.  The budget
determines how far down this residual-outcome ranking the feature extends.  

\subsection{Causal estimation of the MIF}
\label{sec:causal-estimation}

The identification result turns exploratory feature discovery into a residualized prediction problem.  The target causal contrast objective is
\[
    \widetilde\Delta_X(f)
    =
    \mathbb E\left[f(A,X)\{Y-m(R,X)\}\right],
    \qquad
    m(r,x)=\mathbb E(Y\mid R=r,X=x).
\]
Thus we do not need to estimate the high-dimensional regression \(\mathbb E(Y\mid A,X)\) to learn the MIF.  We estimate the outcome regression \(m(r,x)\) once from the shared content representation and course context, and the MIF learns which features of the unstructured treatment explain the remaining outcome variation.  The learned feature is selected for interpretation and further validation.

\textbf{Estimation algorithm.}
Given data \(\{(A_i,X_i,Y_i)\}_{i=1}^n\), compute the fixed remainder \(R_i\) and choose a class \(\{f_\theta:\theta\in\Theta\}\) of feature-scoring functions over the modifiable part of the treatment.  Each feature-scoring function should have a feature-specific remainder satisfying the reconstruction and variation properties in \Cref{sec:causal-identification}.  When \(R\) is coarser, the fixed-remainder preservation condition should also be credible.
\begin{enumerate}
    \item Estimate \(m(r,x)=\mathbb E(Y\mid R=r,X=x)\).  Cross-fitting may be used to obtain held-out predictions \(\widehat m(R_i,X_i)\).  The difference \(Y_i-\widehat m(R_i,X_i)\) is enrollment above or below what is expected from the course's content and context.
    \item Parameterize \(f_\theta\) to lie in \([\epsilon,1-\epsilon]\), for example with the neural-network logit~\(g_\theta(a,x)\) and logistic sigmoid~\(\sigma(u)=\{1+\exp(-u)\}^{-1}\), so that \(f_\theta(a,x)=\epsilon+(1-2\epsilon)\sigma\{g_\theta(a,x)\}\). Estimate the MIF by
    \[
        \widehat\theta
        \in
        \arg\max_{\theta\in\Theta_b}
        \frac1n
        \sum_{i=1}^n
        f_\theta(A_i,X_i)
        \{Y_i-\widehat m(R_i,X_i)\},
        \qquad
        \widehat f=f_{\widehat\theta},
    \]
    where \(\Theta_b\) imposes the budget when one is used, for example through the average-feature-mass constraint \(\frac1n\sum_{i=1}^n f_\theta(A_i,X_i)=b\).
    \item Inspect and validate the learned feature.  Compare treatments with high and low values of \(\widehat f(A_i,X_i)\), evaluate the variance-weighted causal contrast objective on held-out data, and act on the learned direction for controlled perturbations as discussed in \Cref{sec:interpret}.
\end{enumerate}

\begin{proposition}[Consistency of the learned MIF]
\label{prop:mif-consistency}
Suppose \(\Theta\) is compact, the population objective $\widetilde\Delta_X(f)
    =
    \mathbb E\left[f(A,X)\{Y-m(R,X)\}\right]$ has a unique maximizer \(\theta_0\), \(f_\theta(a,x)\) is bounded and continuous in \(\theta\), and \(\widehat m(r,x)\) converges uniformly in probability to \(m(r,x)\).  If \(\mathbb E|Y|<\infty\), then
\[
    \widehat\theta\overset{p}{\longrightarrow}\theta_0.
\]
\end{proposition}

\begin{proposition}[Held-out central limit theorem for a learned feature]
\label{prop:heldout-clt}
Let \(\widehat f\) be trained independently of an evaluation sample of size \(n_{\mathrm{eval}}\).  Conditional on the training data, suppose
\[
    V(\widehat f)
    =
    \operatorname{Var}\left[
        \widehat f(A,X)\{Y-m(R,X)\}
        \mid\widehat f
    \right],
\]
is finite and positive.  Then
\[
    \sqrt{n_{\mathrm{eval}}}
    \left[
        \frac1{n_{\mathrm{eval}}}
        \sum_{i=1}^{n_{\mathrm{eval}}}
        \widehat f(A_i,X_i)\{Y_i-m(R_i,X_i)\}
        -
        \widetilde\Delta_X(\widehat f)
    \right]
    \rightsquigarrow
    \mathcal N\{0,V(\widehat f)\}.
\]
The variance can be estimated by the empirical variance of the held-out summands. The same conclusion holds with a cross-fitted \(\widehat m\) when its contribution to the held-out average is \(o_p(n_{\mathrm{eval}}^{-1/2})\).
\end{proposition}

The proofs of \Cref{prop:mif-consistency,prop:heldout-clt} are in \Cref{app:mif-consistency-proof,app:heldout-clt-proof}. Here \(\widetilde\Delta_X(\widehat f)\) is the population objective for the realized trained feature, conditional on the training data.  The consistency result concerns the MIF within the chosen parameterization.  The held-out result gives inference for the value of a learned feature conditional on the training sample; it does not provide post-selection inference for the optimizer. We use the MIF algorithm primarily for exploratory selection and interpretation.

\subsection{Interpreting and acting on a learned feature}
\label{sec:interpret}

The MIF is a learned function: among the features we allow ourselves to search over, it is the function \(f\) that gives the largest variance-weighted causal contrast in potential outcomes.  After learning such a function, we still have to understand what it has found and decide whether it can be used.  In the course-description example, a large value of \(\widehat f(A,X)\) might correspond to direct address, concrete examples, a shorter opening sentence, a more formal tone, or some combination of these.  The first step is therefore descriptive: inspect texts with small and large values of \(\widehat f\), holding the relevant context fixed, and ask what changes.

\textbf{Why interpretation is needed.}
A learned feature can be influential without being immediately legible.  This is also the difficulty with more traditional topic-model approaches.  Suppose a topic model finds that
$\mathbb E\{Y(k)\}$
is large for some topic \(k\), where \(Y(k)\) denotes the potential outcome under topic \(k\).  If the topic is unmodifiable, the result is not an actionable recommendation.  If the topic is treated as modifiable, there are still two problems.  First, it is unclear how to systematically rewrite a given text so that it is assigned to topic \(k\).  Second, topics may miss important treatment dimensions such as tone, sentiment, formality, politeness, or brevity.  These features can cut across topics rather than coincide with them.

The MIF algorithm replaces a pre-specified topic label with the binary feature intervention \(W_{\widehat f}\) induced by the learned feature-scoring function
\[
    \mathbb P(W_{\widehat f}=1\mid A,X)
    =
    \widehat f(A,X).
\]
Because \(\widehat f\) is learned from the data, the resulting \(W_{\widehat f}\) may represent any modifiable feature allowed by the representation and parameterization.  The simplest interpretation is empirical: compare texts with high and low values of \(\widehat f(A,X)\), and infer what the feature appears to represent.

\textbf{From inspection to perturbation.}
Inspection tells us what the learned feature appears to represent.  It does not tell us whether the representation can be moved in a controlled direction.  When the treatment is represented by an embedding, a natural diagnostic is to perturb the modifiable part of the embedding in the direction that increases the feature score, and then decode the perturbed embedding back to the treatment space.

For text, this means changing the learned feature \(A\) (e.g. style) while leaving the remainder (e.g. content) unchanged.  For a course description, the goal is to change how the course is described while preserving the subject, prerequisites, and learning goals.  For an image, the same idea changes presentation while preserving content.  For a dynamic treatment sequence, it searches nearby or feasible future continuations while treating the past as unmodifiable.  These algorithms implicitly rely on the feature-remainder support condition: the new (style) feature should be compatible with the (content) remainder.

There are two closely related ways to act on a learned feature. One can replace \(A\) by another treatment value \(\widetilde A\) (e.g., text) such that \(\mathbb P(W_{\widehat f}=1\mid \widetilde A,X)\) is larger. Alternatively, one can \textit{nudge} the unstructured treatment \(A\) directly in a local direction that increases the feature score.

\textbf{Nudging without covariates.}
Suppose the learned feature-scoring function depends only on the treatment, without any covariates, so that
\[
    \mathbb P(W_{\widehat f}=1\mid A)
    =
    \widehat f(A),
    \qquad
    \widehat f(A)\in[\epsilon,1-\epsilon].
\]
For a small perturbation \(\delta\), Taylor expansion gives
\[
    \widehat f(A+\delta)
    \approx
    \widehat f(A)
    +
    \delta^\top
    \nabla_{a}
    \widehat f(A).
\]
Thus the steepest local increase is in the direction of
\(\nabla_{a}\widehat f(A)\).
In the empirical studies, we use
    \(A^{\mathrm{new}}=A+\lambda\nabla_a\widehat f(A)/\nu(A)\),
where \(\lambda>0\) is a step size and \(\nu(A)\) is a normalization factor.  This normalization is not part of the estimand; it is a numerical device that makes steps comparable across examples.

\textbf{Nudging with covariates.}
When the learned feature-scoring function depends on covariates \(X\):
\[
    \mathbb P(W_{\widehat f}=1\mid A,X)
    =
    \widehat f(A,X),
\]
the covariate is held fixed.  For a small perturbation \(\delta\) to the style embedding,
\[
\begin{aligned}
    \widehat f(A+\delta,X)
    &=
    \widehat f\{(A,X)+(\delta,0)\}
    \approx
    \widehat f(A,X)
    +
    \delta^\top
    \nabla_{a}
    \widehat f(A,X).
\end{aligned}
\]
The gradient is taken only with respect to the modifiable part of \(A\).  This is intentional: if \(X\) contains the course department, the student population, the article topic, or the publication time, nudging should not change those quantities.  It should only change the treatment feature being probed.

\textbf{Why smoothing is useful.}
The population optimizer often has a hard-cutoff behavior: it puts mass near the boundary values \(\epsilon\) and \(1-\epsilon\), as shown in \Cref{prop:mif-optimizer-main}. This is useful for feature selection, but it can make gradient-based perturbation difficult because the final sigmoid layer may be nearly flat near the boundaries. In the nudging experiments, we therefore replace the final sigmoid \(\sigma(z)\) by the temperature-smoothed version \(\sigma(z/\tau)\) when computing gradients, for positive temperature~\(\tau\). This smoothing is used only for the diagnostic perturbation step. It does not change the target estimand or the learned MIF.

\textbf{Budgets, weighting, and selected regions.}
When several treatment dimensions may affect the outcome, such as formality, brevity, politeness, or sentiment, the learned feature depends on how broad a region it is allowed to select.  \Cref{sec:causal-query} controls this breadth through the budget
\[
    \mathbb E\{f(A,X)\}=b.
\]
Small budgets focus attention on the most outcome-relevant treatment regions.  Larger budgets ask for a broader feature and therefore include weaker regions as well.

The optimizer result in \Cref{sec:causal-query} gives a simple way to read the learned feature.  Recall that the variance-weighted causal contrast objective is
$\widetilde\Delta_X(f)
    =
    \mathbb E\left[f(A,X)\{Y-m(R,X)\}\right].$
The MIF ranks treatment-context pairs by their expected outcome residual.  In the course-description example, it ranks descriptions by enrollment above or below what would be expected from the course content and context.  A budgeted MIF then selects the upper tail of this residual-outcome score.  The selected region expands as the budget increases.

This also clarifies why we use the variance-weighted causal contrast objective for discovery.  The raw causal contrast
   $ \Delta_X(f)
    =
    \Psi_f(1)-\Psi_f(0)$
is the natural estimand for evaluating a fixed feature, but it is normalized by the amount of feature-on and feature-off mass within context.  During search, this normalization can make rare textual patterns look important.  The variance-weighted causal contrast objective \(\widetilde\Delta_X(f)\) keeps the same feature-level potential outcomes but scores a feature by the amount of residual outcome variation it captures.  This makes the learned feature easier to interpret: high \(\widehat f\) should mark descriptions that sit in high residual-outcome regions, not merely descriptions that form a tiny high-contrast group.

\subsection{Practical considerations of the MIF algorithm: What features of the treatment are modifiable?}
\label{sec:cont-style-sep}

Not every feature of a treatment is something one can or should change.  If the learned MIF says that courses in one subject enroll better than courses in another, it may be causally meaningful, but it is not an actionable writing recommendation.  An instructor cannot turn Bayesian statistics into organic chemistry.  What she can change is how the same course is described.  We therefore distinguish between parts of the treatment that are unmodifiable and parts that may be modified.

For text, this separation is often a separation between content and style.  Content captures what the course is about: the topic, prerequisites, learning goals, and substantive material.  Style captures how that content is expressed: whether the description is concrete or abstract, direct or indirect, formal or conversational, terse or detailed.  In the MIF search, we restrict attention to the modifiable part of the treatment.  In the revision step, we change only that part while preserving the original content.  This makes the learned feature closer to a recommendation an instructor could actually use.

The same issue appears in a news-headline example.  A headline about a tennis tournament should not be modified into a headline about a presidential election.  The topic and factual information are unmodifiable.  At the same time, other aspects of the headline, such as formality, brevity, or directness, may be freely modified.  In this work, we refer to the unmodifiable aspects of the text as the \textit{content}, and the aspects of the text that can be modified as the \textit{style}.

The content-style distinction is particularly important in the MIF algorithm because we learn the feature-scoring function $f$ from the data and use it to define the binary feature intervention $W_f$.\footnote{Most existing methods pre-define the treatments. These treatments are assumed to be modifiable, so content-style separation is not applicable.}  We need to ensure that $W_f$ represents a modifiable part of the treatment.  There are many possible definitions of content and style, and the right definition depends on the application.  For example, in the content-style separation literature, sentiment is typically treated as style \citep{patel2023learning,patel2024styledistance,wegmann2022same}.  But it is also reasonable to treat sentiment as content, since changing a text from positive to negative can fundamentally alter the information it conveys.  We therefore allow different definitions of content and style in different settings.

For now, we assume that the text embedding $A$ can be separated\footnote{Unlike existing approaches that typically focus only on learning style embeddings, our content-style separation algorithm learns both content and style embeddings and allows the original embedding to be reconstructed from these components.} into the content embedding $A_{\mathrm{unmodifiable}}$ and the style embedding $A_{\mathrm{modifiable}}$, and that one can reconstruct $A$ from $A_{\mathrm{unmodifiable}}$ and $A_{\mathrm{modifiable}}$.  Roughly speaking, the content embedding captures all information in the text that is considered content, while the style embedding captures all information in the text that is considered style.  Consequently, our feature-scoring function becomes
\[
    \mathbb{P}(W_f = 1 \mid A_{\mathrm{modifiable}}, X)
    =
    f(A_{\mathrm{modifiable}},X),
\]
instead of
$    \mathbb{P}(W_f = 1 \mid A, X)
    =
    f(A,X).$
This is also why topic-model-based approaches are not ideal for influential feature discovery: the information encoded by topic models is primarily content.

\textbf{Known covariates and support.}
We implicitly assume that the known confounder $X$ does not contain stylistic information.  It may contain content information, as well as other non-text features not included in $A$.  For instance, in the news-headline setting, $X$ may include the article's topic and publication time, but not the headline's formality if formality is treated as style.  We also assume that content and style have independent supports: any content can in principle occur with any style, and any style can in principle occur with any content.  If all positive reviews were exclusively about food, while all negative reviews were exclusively about drinks, then disentangling content from style would be difficult, if not impossible.

The same idea extends beyond text.  In images, the unmodifiable part might be object identity or biological content, while the modifiable part might be blur, contrast, rotation, or other acquisition choices.  In dynamic treatment sequences, the past is unmodifiable and the future is modifiable.  In such cases, content can naturally be interpreted as the portion of the treatment sequence that has already been administered, while style corresponds to the choices that remain modifiable in the future.  This setting is often more straightforward than text or image applications, since one can enforce content-style separation by treating the past trajectory as unmodifiable and only intervening on future treatments.

\textbf{Existing style embeddings can still contain content.}
We begin with a simple failure mode. Existing content-style separation methods often define style in advance. They then use contrastive learning to extract style embeddings \citep{patel2023learning,patel2024styledistance,wegmann2022same}. In these works, sentiment is typically treated as style rather than content. But even under this convention, the learned style embeddings may still contain substantial content information.

Consider sentences of the form
\[
    \texttt{The [food/drink item] is [adj\_1] and [adj\_2]},
\]
where each sentence is equally likely to be about food or drink, and independently equally likely to be positive or negative.  Here, the content is the food or drink item, while the style is represented by the adjectives.  This matches the convention in which sentiment is treated as part of style.

We perform a probing experiment.  We ask whether two binary variables can be predicted from the style embeddings: the food/drink label and the positive/negative adjective label.  For the style embeddings of \citet{patel2024styledistance}, the test accuracies for predicting food/drink and positive/negative are $0.97$ and $1.00$, respectively.\footnote{In both setups, all quantities are averaged across 10 experiments.}  For the style embeddings of \citet{wegmann2022same}, the corresponding accuracies are $0.96$ and $1.00$.  In other words, the style embeddings still contain a large amount of content information.

\textbf{Defining content and style.}
We now describe the representation we want to learn.  The text embedding $A$ is split into the content embedding $A_{\mathrm{unmodifiable}}$ and the style embedding $A_{\mathrm{modifiable}}$.  Ideally, $A_{\mathrm{unmodifiable}}$ encodes the parts of the text that cannot be modified, while $A_{\mathrm{modifiable}}$ encodes the parts of the text that can be modified.  The pair $(A_{\mathrm{unmodifiable}},A_{\mathrm{modifiable}})$ should still contain enough information to reconstruct the original text embedding $A$.

How should content and style be defined?  One possibility is to define content as what the text is about, and style as everything other than content.  Under this definition, sentiment is part of content, because changing sentiment changes what the text says.  This differs from the convention in \citet{wegmann2022same} and \citet{patel2023learning,patel2024styledistance}, where sentiment is treated as style.

Another possibility is to define content for the application.  In the food/drink review example, for instance, we may define the food or drink item as content and everything else as style.  We do not pre-define the style feature itself, because our goal is to discover the MIF from observational data.  We only specify what is unmodifiable; the remaining variation is the space in which the MIF searches.

\textbf{Extracting content and style embeddings.}
To separate the text embedding $A$ into content and style embeddings, we develop a content-style separation algorithm that uses four losses:
\begin{enumerate}
    \item \textbf{Reconstruction loss}, ensuring that we can reconstruct the original embedding from the content and style embeddings.
    \item \textbf{Content loss}, ensuring that texts with the same content have similar content embeddings.
    \item \textbf{Style loss}, ensuring that texts with the same style have similar style embeddings.
    \item \textbf{Independence-of-support loss}, ensuring that content and style embeddings have independent supports.
\end{enumerate}

Theoretically, losses \#1, \#2, and \#3 are sufficient for content-style separation, even though loss \#4 can also be added to further encourage separation.  In practice, we minimize a linear combination of these losses. \Cref{sec:cont-style-sep-theory} details this algorithm, and proves in \Cref{prop:content-style-sep} that, under suitable assumptions, perfect content-style separation is achieved when the reconstruction, content, and style losses are all zero.  In practice, we cannot expect to achieve zero loss, but we can still achieve good separation with small losses.  \Cref{sec:cont-style-sep-examples} reports controlled separation, the failure of a headline representation under narrow style supervision, and approximate sentiment preservation in restaurant-review transfers.

\section{Empirical studies}
\label{sec:empirical_studies}

The empirical studies ask whether the MIF (i) recovers a known outcome-relevant feature after conditioning on context, (ii) provides a direction for modifying the treatment while approximately preserving its unmodifiable component, and (iii) changes coherently with the search representation and feature budget.

Across text, image, and dynamic treatment-sequence studies, the learned MIF recovers the feature used to generate the outcome, including settings in which the preferred direction reverses with context. Nudging the treatment along the MIF direction changes formality, toxicity, rotation, and blur in the corresponding direction. In particular, in the formality study on text, the embedding-based MIF attains a held-out contrast of \(1.4\), compared with \(0.24\) for the best topic-based split and \(2.0\) for a perfect formality split. The content--style separation algorithm also succeeds in controlled examples.

\subsection{Experiment I: Empirical studies on text-based treatments}
\label{sec:emp-text}

The text studies consider formality in user comments from Grammarly's Yahoo Answers Formality Corpus (GYAFC), toxicity in ParaDetox comments, and sentiment in food and drink reviews.

\subsubsection{GYAFC: User comments, response helpfulness, and formality}
\label{sec:emp-main-gyafc}

\textbf{Causal question of interest.}
What textual features of a user comment are associated with more helpful responses, after accounting for the context in which the comment appears? Operationally, can the MIF recover formality both when its outcome-enhancing direction is shared across contexts and when that direction reverses between casual and advice-seeking threads?

\textbf{Experimental setup.}
We use GYAFC comments~\citep{rao2018dear} from its Family and Relationships domain. The context is \(X=0\) for a casual discussion thread and \(X=1\) for an advice-seeking thread, and \(S_{\mathrm{formal}}\) records whether a comment is formal. In the first semi-synthetic scenario, formal comments have larger response-helpfulness outcomes in both contexts. In the second, the outcome-relevant direction reverses: informal comments have larger outcomes in casual threads, whereas formal comments have larger outcomes in advice-seeking threads.

\textbf{Results.}
On the held-out test set, the average learned scores for informal and formal comments in the shared-direction scenario are \(0.19\) and \(0.82\) when \(X=0\), and \(0.26\) and \(0.84\) when \(X=1\). In the reversing scenario, the corresponding test-set averages are \(0.81\) and \(0.22\) when \(X=0\), but \(0.28\) and \(0.85\) when \(X=1\). Thus, the learned score recovers both a common formality direction and a formality direction that changes with context (\Cref{fig:exp1-plot}).

\begin{figure}[!htbp]
    \centering
    \includegraphics[width=0.92\textwidth]{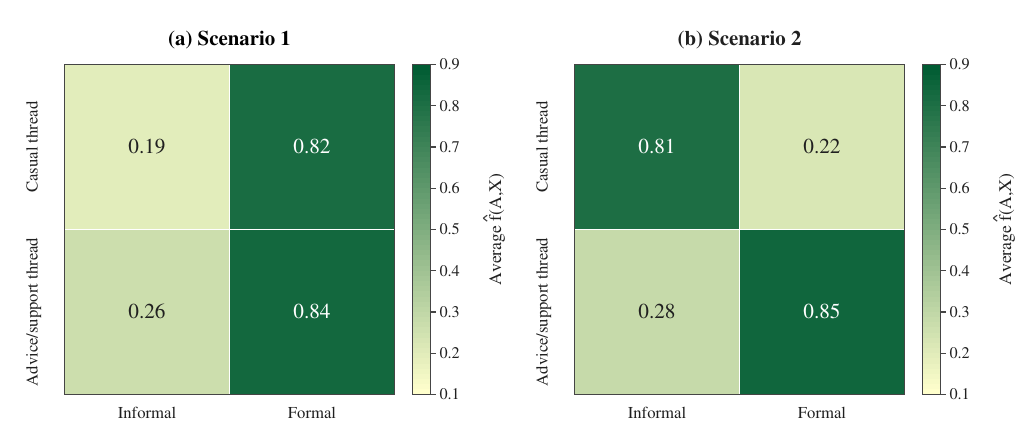}
    \caption{The learned MIF recovers the outcome-relevant formality direction on the held-out test set. In the shared-direction scenario, formal comments receive larger scores in both contexts. In the context-dependent scenario, informal comments receive larger scores in casual threads and formal comments receive larger scores in advice-seeking threads. Values inside the cells are test-set average learned scores; larger scores indicate the outcome-favored formality direction.}
    \label{fig:exp1-plot}
\end{figure}

\textbf{Nudging results.}
In the shared-direction scenario, ten nudging iterations increase average formality from approximately \(0.2\) to nearly \(1.0\). In the context-dependent scenario, nudging initially informal comments in advice-seeking threads increases average formality from approximately \(0.3\) to \(0.8\), while nudging initially formal comments in casual threads decreases average formality from approximately \(0.9\) to \(0.5\). The perturbation therefore follows the context-specific direction rather than pushing every comment toward formality (\Cref{fig:exp1-nudging}).

\begin{figure}[!htbp]
    \centering
    \includegraphics[width=0.80\textwidth]{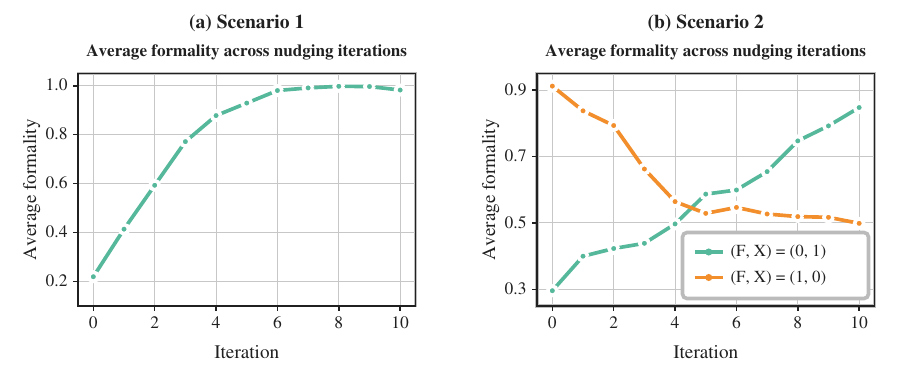}
    \caption{Nudging changes formality in the direction associated with larger outcomes. In Scenario 1, nudging increases formality because formal comments are associated with higher outcomes in both contexts. In Scenario 2, the direction depends on context: nudging raises formality in advice-seeking threads and lowers it in casual threads. Higher is more formal.}
    \label{fig:exp1-nudging}
\end{figure}

\textbf{Causal question of interest: no-covariate diagnostic.}
When formality is the only outcome-relevant feature, does the learned MIF recover formality, and does nudging move comments toward the higher-outcome direction? Can a semantic-topic feature class recover this stylistic feature, which cuts across topics, as well as an embedding-based search?

\textbf{Experimental setup.}
Using the same GYAFC domain, we remove covariates, learn the MIF from SONAR embeddings under an outcome determined by formality, and nudge the learned embedding. Decoded text is evaluated with \texttt{s-nlp/xlmr\_formality\_classifier}, a Hugging Face evaluator that returns a formality score in \([0,1]\). We then compare the held-out variance-weighted causal contrast with the best split available from latent Dirichlet allocation (LDA) topic assignments.

\textbf{Results.}
Formal sentences concentrate at large learned scores and informal sentences at small scores, while the external formality score increases over ten nudging iterations. The five displayed nudges generally preserve the main semantic content while increasing formality. The embedding-based MIF attains a held-out contrast of \(1.4\), compared with \(0.24\) for the best topic-based split and \(2.0\) for a perfect formality split. The selected topics describe relationship themes but do not isolate formality, which cuts across those themes.

\textbf{Remaining results.}
\Cref{sec:text-gyafc} gives the shared- and context-dependent data-generating processes, estimation and nudging settings, and qualitative low- and high-score comments. \Cref{sec:formality-diagnostic} reports the no-covariate score distribution, formality trajectory, and before-and-after texts; \Cref{sec:topic-model-failure} gives the complete topic search, objective calculation, and representative words.

\FloatBarrier

\subsubsection{ParaDetox: Toxicity and platform response priorities}
\label{sec:emp-main-toxicity}

\textbf{Causal question of interest.}
What textual properties of a user comment should drive platform response priorities, after accounting for the context in which the comment appears? In particular, can the MIF distinguish a neutral-wording direction for response priority interpreted as constructive engagement from a toxic-wording direction for moderation review when that direction reverses by context?

\textbf{Experimental setup.}
We use neutral and toxic comment pairs from ParaDetox~\citep{logacheva2022paradetox}. The context distinguishes community discussion from moderation-sensitive threads. In the first scenario, neutral comments have larger platform response-priority scores in both contexts, interpreted as promoting constructive engagement. In the second, neutral comments have larger response-priority scores in community threads, whereas toxic comments receive larger review-priority scores in moderation-sensitive threads.

\textbf{Results.}
On the held-out test set, the average learned scores for neutral and toxic comments in the shared-direction scenario are \(0.88\) and \(0.13\) in community threads, and \(0.88\) and \(0.11\) in moderation-sensitive threads. In the reversing scenario, the corresponding test-set averages are \(0.87\) and \(0.16\) in community threads, but \(0.15\) and \(0.89\) in moderation-sensitive threads (\Cref{fig:tox-plot}). Nudging also follows these directions. It reduces average toxicity from approximately \(0.9\) to \(0.1\) in the shared-direction scenario and changes direction with the thread context in the second scenario.

\begin{figure}[!htbp]
    \centering
    \includegraphics[width=0.92\textwidth]{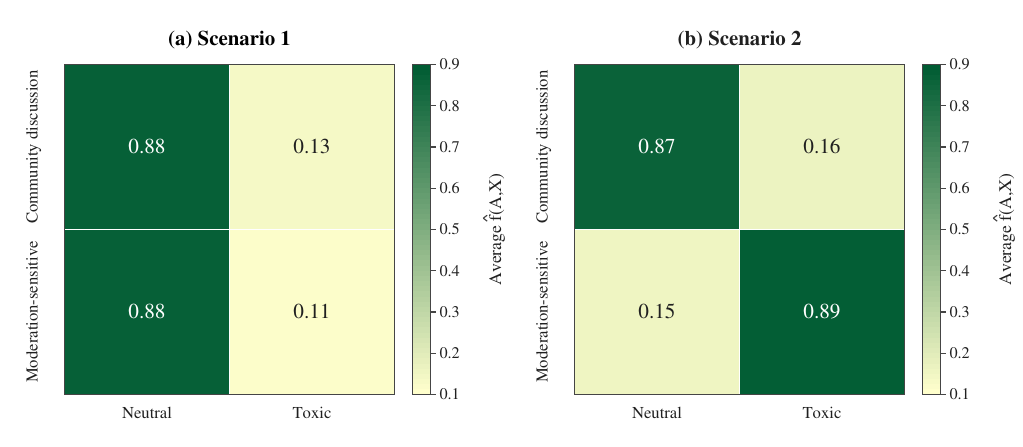}
    \caption{The learned MIF recovers the outcome-relevant toxicity direction on the held-out test set. It assigns larger scores to neutral comments in both contexts in the shared-direction scenario and reverses the ordering in moderation-sensitive threads when toxicity determines review priority. Values inside the cells are test-set average learned scores; larger scores indicate higher response priority.}
    \label{fig:tox-plot}
\end{figure}

\textbf{Remaining results.}
\Cref{sec:text-toxic} reports the complete outcome models, Detoxify-based nudging trajectories, implementation details, and low- and high-score comment examples.

\FloatBarrier

\subsubsection{Food and drink reviews: Helpfulness and representation}
\label{sec:emp-main-food}

\textbf{Causal question of interest.}
What textual features in a review are associated with higher helpfulness, after accounting for the type of item being reviewed? More specifically, does \(f(A,X)\) recover category-specific sentiment directions, and what marginal sentiment direction does \(f(A)\) recover when it cannot condition on category?

\textbf{Experimental setup.}
We use author-generated controlled food and drink reviews, with \(X=0\) for food and \(X=1\) for drink, and review sentiment as the outcome-relevant text feature. We compare two queries. The context-dependent score \(f(A,X)\) can assign different directions to the same sentiment across categories. The treatment-only score \(f(A)\) must instead select one marginal direction after content and sentiment style are separated.

\textbf{Results.}
On the held-out test set, when positive reviews have larger helpfulness outcomes in both categories, \(f(A,X)\) assigns larger values to positive reviews throughout. When the preferred sentiment reverses across categories, \(f(A,X)\) assigns larger values to negative food reviews and positive drink reviews. By contrast, when \(\mathbb P(X=1)=0.8\) and the score cannot condition on category, \(f(A)\) recovers the marginally preferred direction: \(0.9\) for positive reviews and \(0.1\) for negative reviews (\Cref{fig:dgp1-exp2}). This comparison makes the target of the two parameterizations visible.

\begin{figure}[!htbp]
    \centering
    \includegraphics[width=0.88\textwidth]{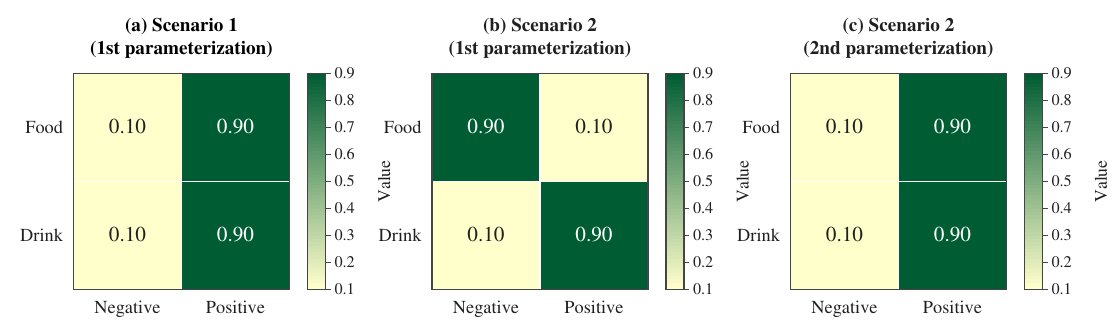}
    \caption{The learned sentiment direction changes with the MIF parameterization on the held-out test set. Panel (a) assigns large \(\widehat f(A,X)\) values to positive reviews when positive sentiment is favored in both categories. Panel (b) reverses the direction across food and drink when the outcome-relevant direction depends on category. Panel (c) uses the style-only score \(f(A)\), which recovers the marginally optimal direction when drink reviews are more common. Values inside the cells are test-set average learned scores; larger scores indicate the outcome-favored sentiment direction.}
    \label{fig:dgp1-exp2}
\end{figure}

\textbf{Representation question.}
Can the search be restricted to modifiable sentiment and style while excluding food-or-drink item content, and can the original text still be reconstructed?

\textbf{Experimental setup.}
Using the same controlled food-and-drink review family, we first probe two existing style embeddings for content leakage. We then train the proposed content--style separation procedure and test what can be predicted from each learned component.

\textbf{Results.}
The two existing style representations predict the food--drink content label with accuracies \(0.97\) and \(0.96\), and both predict the sentiment label with accuracy \(1.00\), showing substantial content leakage. Under the proposed procedure, the learned content embedding predicts content and style with accuracies \(1.00\) and \(0.53\), while the style embedding predicts them with accuracies \(0.52\) and \(1.00\). The probe results and the reconstructed sentences therefore support successful separation in this controlled example.

\textbf{Remaining results.}
\Cref{sec:text-food} gives the full outcome models, the population calculation for the treatment-only rule, and the remaining causal discussion. \Cref{sec:cont-style-sep} and \Cref{sec:cont-style-sep-theory} give the separation objectives and guarantees; \Cref{sec:cont-style-food} reports the loss trajectories, probe details, and reconstructed reviews.

\FloatBarrier

\subsection{Experiment II: Empirical studies on image-based treatments}
\label{sec:emp-image}

The image studies ask whether the MIF can identify a modifiable visual feature while separating it from image content. The first study makes the preferred rotation depend on digit parity; the second isolates a focus artifact from cell count.

\subsubsection{Rotated handwritten digits}
\label{sec:emp-main-mnist}

\textbf{Causal question of interest.}
After accounting for digit identity, which visual patterns are associated with larger downstream recognition or review scores? Operationally, can the MIF recover a parity-dependent rotation direction while treating digit identity as unmodifiable?

\textbf{Experimental setup.}
We use MNIST digits~\citep{lecun1998gradient}. We draw the rotation angle~\(\alpha\) from \(\mathrm{Uniform}(-70,70)\), treat digit identity as unmodifiable content, and encode visual style in four dimensions. The context records digit parity, with \(X=0\) for even digits and \(X=1\) for odd digits. With independent error~\(\xi\sim\mathcal N(2,1)\), the semi-synthetic outcome is
\[
    Y=0.1\alpha(2X-1)+10X+\xi.
\]
The outcome-relevant rotation direction is therefore opposite for even and odd digits.

\textbf{Results.}
The representation check first shows that changing digit identity while holding the style embedding fixed preserves rotation and stroke pattern; the complete reconstruction grid is in \Cref{sec:mnist}. The MIF then recovers the parity-dependent direction. Among even digits, the mean rotation angle changes from \(26.8\) degrees when \(\widehat f<0.2\) to \(-32.5\) degrees when \(\widehat f>0.8\). Among odd digits, it changes from \(-21.1\) to \(28.7\) degrees (\Cref{fig:scatter-image}).

\begin{figure}[!htbp]
    \centering
    \includegraphics[width=0.80\textwidth]{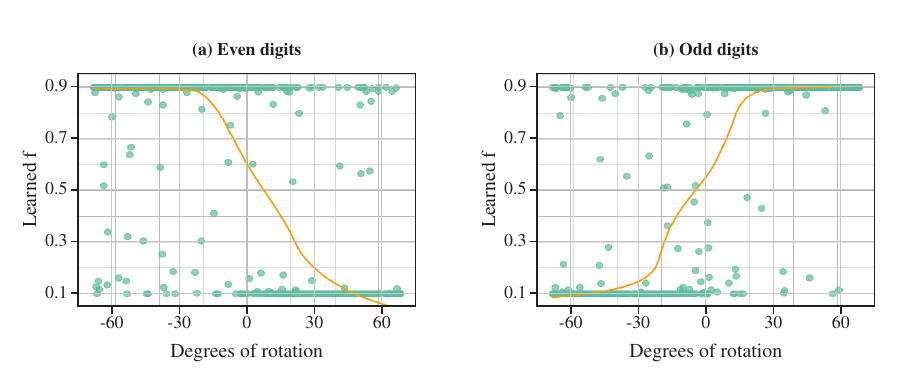}
    \caption{The learned MIF recovers opposite rotation directions for even and odd digits. Large \(\widehat f\) corresponds to clockwise rotations for even digits and counter-clockwise rotations for odd digits. Larger scores indicate the outcome-favored rotation direction.}
    \label{fig:scatter-image}
\end{figure}

\textbf{Nudging results.}
Nudging rotates odd digits counter-clockwise and even digits clockwise while preserving digit identity (\Cref{fig:nudging-mnist}). This is a direct visual check that changes in the modifiable embedding follow the learned direction while the designated content remains stable.

\begin{figure}[!htbp]
    \centering
    \includegraphics[width=0.74\textwidth]{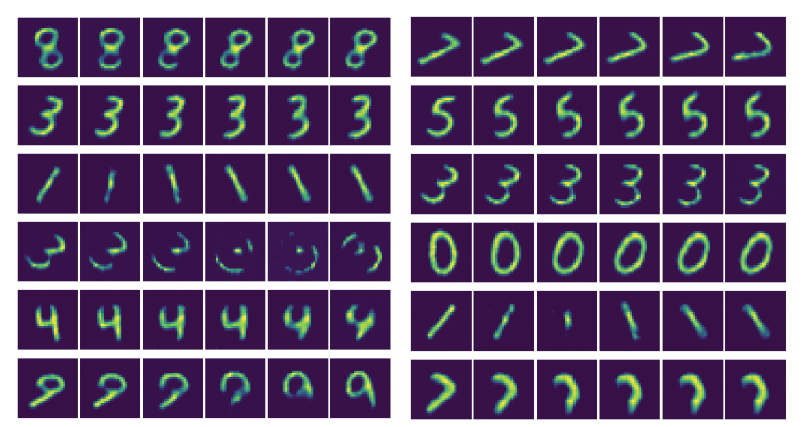}
    \caption{MNIST nudging changes rotation while preserving digit identity. Odd digits rotate counter-clockwise and even digits rotate clockwise, following the parity-dependent direction learned from the outcome model. Columns give nudging iterations \(0\) through \(5\).}
    \label{fig:nudging-mnist}
\end{figure}

\textbf{Remaining results.}
Full image-model and nudging details appear in \Cref{sec:mnist}.

\FloatBarrier

\subsubsection{Cell-body stains}
\label{sec:emp-main-cell}

\textbf{Causal question of interest.}
After accounting for the number of cells in an image, which acquisition artifact is associated with larger quality-control or review-priority scores? Operationally, can the MIF isolate focus blur from cell count?

\textbf{Experimental setup.}
We use BBBC005v1 cell-body-stain images~\citep{ljosa2012annotated}, treating cell count as unmodifiable content and focus blur~\(B\in[1,48]\) as modifiable style. With \(X\) indicating whether the image contains fewer than \(20\) cells and independent error~\(\xi\sim\mathcal N(2,1)\), we generate
\[
    Y=0.1B+10X+\xi.
\]
The MIF should therefore identify blur after adjusting for cell-count group.

\textbf{Results.}
The representation check first shows that reconstructed images approximately preserve blur while cell count changes; the complete grid is in \Cref{sec:cell}. Among images with more than \(20\) cells, mean blur increases from \(15.9\) when \(\widehat f<0.2\) to \(40.3\) when \(\widehat f>0.8\). Among images with fewer than \(20\) cells, it increases from \(15.7\) to \(37.1\) (\Cref{fig:scatter-image-cell}). Nudging makes images blurrier while approximately preserving cell count. Here, increased blur is a diagnostic of the direction associated with larger review-priority outcomes, not a recommendation for improving image quality.

\begin{figure}[!htbp]
    \centering
    \includegraphics[width=0.80\textwidth]{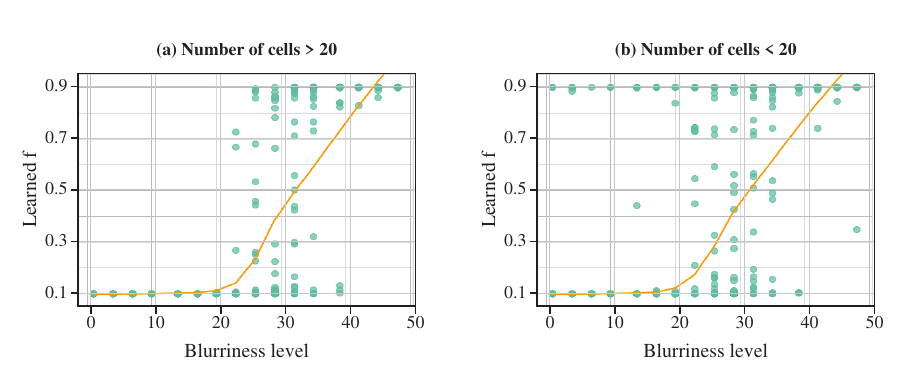}
    \caption{The learned MIF recovers blur in both cell-count groups. Images with larger scores have substantially more blur regardless of cell count, so the MIF tracks the acquisition artifact rather than only the number of cells. Larger scores indicate the outcome-favored blur direction.}
    \label{fig:scatter-image-cell}
\end{figure}

\textbf{Remaining results.}
\Cref{sec:cell} gives the image architecture and style dimension, shows the complete content--style reconstructions, and reports all nudging panels and their interpretation.

\FloatBarrier

\subsection{Experiment III: Empirical studies on dynamic treatment sequences}
\label{sec:emp-sequence}

Dynamic treatments arise in longitudinal care, where a patient may receive different interventions across recurring visits and those visits may occur at varying intervals. The causal question is which features of the treatment sequence---such as treatment type, order, or spacing---most influence the patient's final outcome. Our sequence studies isolate this question by varying how a treatment history is represented and which parts remain modifiable. We first analyze the same short histories under categorical and sentence representations, and then add administration times and a fixed treatment prefix.

\subsubsection{Short treatment sequences under categorical and sentence representations}
\label{sec:emp-main-sequence-short}

\textbf{Causal question of interest.}
Which treatment-type counts in a short longitudinal regime yield the largest raw causal contrast for the final outcome, and does nudging toward the learned score improve that outcome under two opposing models? When context reverses the preferred direction, does the learned direction reverse within context?

\textbf{Experimental setup.}
We generate sequences of length one to six in which \(a\), \(b\), and \(c\) denote three treatment or visit types. Two no-covariate outcome models reward the type counts in opposite directions. A third design uses a binary context \(X\) to switch between those directions. We encode each sequence with SONAR, take five nudging steps, and decode back to valid categorical sequences.

\textbf{Results.}
In the two no-covariate scenarios, the average outcome rises from \(45.5\) to \(64.7\) and from \(64.5\) to \(74.5\). In the context-dependent design, it rises from \(45.5\) to \(77.5\) for \(X=0\) and from \(64.5\) to \(68.1\) for \(X=1\). The \(X=1\) trajectory is not monotone: it first falls to \(62.2\) before recovering. Because these short sequences have no unmodifiable component, direct comparison of learned scores can be more useful than local nudging when the entire treatment may be changed.

\textbf{Representation question.}
Does the same causal feature and outcome-improving nudge persist when the identical treatment histories and scenarios are represented as natural-language sentences?

\textbf{Experimental setup.}
We encode the same categorical histories in natural language; for example, \textit{a b c} becomes \textit{The first treatment is a. The second treatment is b. The third treatment is c.} We then repeat the five-step nudging procedure.

\textbf{Results.}
The average outcome rises from \(36.2\) to \(48.2\) in the first scenario and from \(73.8\) to \(89.8\) in the second. The second trajectory peaks at \(89.9\) after four steps, so the endpoint improves substantially even though the final step is not monotone. \Cref{tab:main-sequence-results} collects the endpoint results for both representations.

\begin{table}[!htbp]
    \centering
    \small
    \caption{Five nudging iterations increase the endpoint average outcome in all six sequence settings. The \(X=1\) categorical trajectory initially declines, and the second sentence trajectory peaks one iteration before the endpoint. Higher is better.}
    \begin{tabularx}{\textwidth}{@{}llrrX@{}}
        \toprule
        Representation & Condition & Initial & Iteration 5 & Intermediate behavior \\
        \midrule
        Categorical & No covariate, Scenario 1 & \(45.5\) & \(64.7\) & Monotone increase \\
        Categorical & No covariate, Scenario 2 & \(64.5\) & \(74.5\) & Monotone increase \\
        Categorical & \(X=0\) & \(45.5\) & \(77.5\) & Monotone increase \\
        Categorical & \(X=1\) & \(64.5\) & \(68.1\) & Falls to \(62.2\), then recovers \\
        Sentence & Scenario 1 & \(36.2\) & \(48.2\) & Flat through iteration 2 \\
        Sentence & Scenario 2 & \(73.8\) & \(89.8\) & Peaks at \(89.9\) in iteration 4 \\
        \bottomrule
    \end{tabularx}
    \label{tab:main-sequence-results}
\end{table}

\textbf{Remaining results.}
\Cref{sec:sequence-short} gives the two no-covariate models, the context-dependent model, the decoding filter, and all iteration-level values. \Cref{sec:sequence-sentence} reports every iteration-level result for the sentence representation.

\FloatBarrier

\subsubsection{Treatment sequences with administration times}
\label{sec:emp-main-sequence-time}

\textbf{Causal question of interest.}
In a longitudinal treatment regime---such as care delivered at recurring prenatal visits with varying intervals---which treatment types and spacings most influence the patient's final outcome under a fixed feature budget? Once a treatment prefix has been administered, which feasible suffix raises the learned score while leaving that prefix unchanged?

\textbf{Experimental setup.}
We use a synthetic abstraction of recurring clinical visits: we allow at most three treatments, record each treatment type and administration time, and measure the final outcome at time~\(100\). The outcome rewards the counts of \(a\), \(b\), and \(c\), as well as the shortest interval~\(t_p\) between treatments:
\[
    Y=5N_a(A_{1:L})+10N_b(A_{1:L})+15N_c(A_{1:L})+0.2\,t_p+\xi,
    \qquad \xi\sim\mathcal N(0,4).
\]
We fit treatment-only scores under four feature budgets.

\textbf{Results.}
At every budget, the learned score is positively associated with the expected outcome. Smaller budgets reserve large scores for a narrower group of larger-outcome sequences, whereas larger budgets expand the selected region (\Cref{fig:scatter}). When part of a sequence has already been administered, the score can also rank feasible suffixes while leaving the observed prefix unchanged.

\begin{figure}[!htbp]
    \centering
    \includegraphics[width=0.98\textwidth]{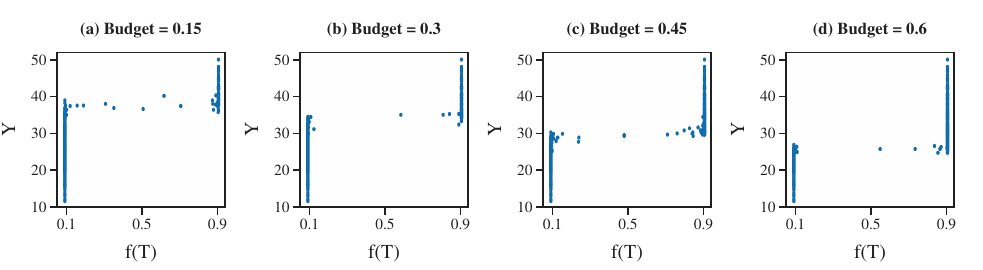}
    \caption{Smaller budgets concentrate large MIF scores on a narrower set of larger-outcome sequences. Across budgets, \(\widehat f(A_{1:L})\) remains positively associated with \(\mathbb E(Y\mid A_{1:L})\). Higher is better for both quantities.}
    \label{fig:scatter}
\end{figure}

\textbf{Remaining results.}
\Cref{sec:sequence-time} gives the complete timing design, budget interpretation, and prefix-preserving suffix-search discussion.

\FloatBarrier

\subsection{Content-style separation on text representations}
\label{sec:emp-representation-aux}

We also report two studies of the content--style separation algorithm in \Cref{sec:cont-style-sep}. Both use a large language model (LLM) to construct supervision and assess whether the learned representation is suitable for a later causal analysis.

\subsubsection{News headlines with LLM-based embeddings}
\label{sec:emp-main-headlines}

\textbf{Representation question.}
Can headline content be held fixed while modifiable style is isolated well enough to support outcome-relevant feature discovery?

\textbf{Experimental setup.}
We assemble \(20{,}000\) English news headlines from the general-news domains listed in \Cref{sec:cont-style-headlines}. We use the Llama 3.1 8B Instruct model to create pairs with shared content but different style and pairs with shared style but different content. We train content and style mappings together with a reconstruction map. The style-preserving prompt chiefly enforces a shared grammatical skeleton.

\textbf{Results.}
Reconstruction works well, but content--style separation does not work very well. The learned style largely captures grammatical skeleton rather than the broader writing choices that the study intends to make modifiable. The experiment identifies a concrete failure mode: narrow style supervision can make the representation learn surface form instead of the intended notion of style.

\textbf{Remaining results.}
\Cref{sec:cont-style-headlines} retains the complete data construction, prompts, training objective, representative reconstructions, and failure analysis.

\FloatBarrier

\subsubsection{Restaurant reviews with LLM-based embeddings}
\label{sec:emp-main-restaurants}

\textbf{Representation question.}
Can style be transferred while preserving restaurant, dish, taste, and sentiment content, so that a subsequent outcome contrast can be attributed to style rather than to changed content?

\textbf{Experimental setup.}
We generate short one-sentence Asian restaurant reviews that vary in sentiment and formality, construct same-content and same-style triplets, and transfer each source review into the styles of target reviews. We assess the transferred text qualitatively and use an external model to check sentiment preservation.

\textbf{Results.}
The transferred reviews approximately preserve restaurant, dish, taste, and sentiment on average, but individual rewrites contain grammatical and semantic artifacts. Across 100 target styles for each of ten source reviews, positive and negative polarity is generally preserved, although both magnitude and isolated cases can change. The representation therefore works reasonably well on average but is not exact for every transfer.

\textbf{Remaining results.}
\Cref{sec:cont-style-restaurants} retains the triplet construction, all qualitative transfer tables, every reported average sentiment score, and the discussion of grammatical and semantic artifacts.

\FloatBarrier

\section{Discussion}
\label{sec:disc}
This paper introduces the MIF algorithm for causal inference in observational studies involving unstructured treatments, such as texts, images, or medical treatment sequences. In these settings, traditional causal queries like the average treatment effect are often ill-defined because the overlap condition rarely holds. We address this problem by learning the maximally influential feature (MIF), which defines a binary feature intervention on unstructured objects. Using embedding representations, the MIF algorithm can capture abstract treatment features (e.g., formality, politeness, and brevity) that topic model-based approaches often fail to represent. It also identifies the most influential treatment feature from the data, revealing the dimension that has the strongest causal impact on the outcome. To make interventions realistic, we separate each treatment into unmodifiable content and modifiable style components. Moreover, the MIF is both interpretable and practically actionable: we can examine the learned feature score to understand what it represents and directly adjust the unstructured treatment to generate changes that are expected to improve outcomes. Overall, the MIF provides a simple and data-driven way to discover, interpret, and act on the key causal factors in unstructured treatments.

\section*{Acknowledgements}

This work was supported in part by funding from the Office of Naval Research under grant N00014-23-1-2590, the National Science Foundation under grant No. 2310831, No. 2428059, No. 2435696, No. 2440954, a Michigan Institute for Data Science Propelling Original Data Science (PODS) grant, Two Sigma Investments LP, and  LG Management Development Institute AI Research. Any opinions, findings, and conclusions or recommendations expressed in this material are those of the authors and do not necessarily reflect the views of the sponsors.

\clearpage
\bibliographystyle{abbrvnat}
\bibliography{bibfile}
\clearpage

\appendix
\crefalias{section}{appendix}
\crefalias{subsection}{appendix}
\crefalias{subsubsection}{appendix}

\begin{center}
    \Large{\textbf{Supplementary Materials: Causal Inference with Unstructured Treatments}}
\end{center}

\section{Raw causal contrast and no-covariate calculations}
\label{app:mif-original}

This appendix records the raw causal contrast and no-covariate calculations used to motivate the budget and the variance-weighted causal contrast objective in the main text.  The main conclusion is that, at a fixed budget, the raw causal contrast and the variance-weighted causal contrast rank features in the same way in the no-covariate case; across budgets, the variance-weighted causal contrast has a cleaner interpretation: it measures the total residual outcome signal captured by the selected feature region.

\subsection{No-covariate raw causal contrast}

If there are no covariates and the feature score approaches a binary-valued function, the raw causal contrast is
\[
    \Delta_X(f)
    =
    \Psi_f(1)-\Psi_f(0).
\]
Under the same identification logic as \Cref{prop:iden},
\[
    \Psi_f(1)
    =
    \frac{\mathbb E\{Yf(A)\}}{\mathbb E\{f(A)\}},
\]
and
\[
    \Psi_f(0)
    =
    \frac{\mathbb E\{Y[1-f(A)]\}}{\mathbb E\{1-f(A)\}}.
\]
Thus
\[
    \Delta_X(f)
    =
    \frac{\mathbb E\{Yf(A)\}}{\mathbb E\{f(A)\}}
    -
    \frac{\mathbb E\{Y[1-f(A)]\}}{\mathbb E\{1-f(A)\}}.
\]
This normalized form is natural for evaluating a fixed feature, but it can behave awkwardly as the budget changes because the selected region and the denominator change at the same time.

\subsection{Fixed-budget optimizer for the raw causal contrast}

Let \(b=\mathbb E\{f(A)\}\).  If the budget is fixed, then maximizing the raw causal contrast over features with \(\mathbb E\{f(A)\}=b\) is equivalent to selecting treatment values with large values of the relevant conditional outcome score.  In the no-covariate treatment-only case, this score is
\[
    m(a)=\mathbb E(Y\mid A=a).
\]
The optimizer assigns the high feature value to the upper tail of \(m(A)\), with the threshold chosen to satisfy the budget.  This is the same ranking logic as the variance-weighted causal contrast objective, but the objective value is normalized by the feature mass.  For this reason the optimized raw causal contrast need not be monotone in the budget.

\subsection{Two-feature budget intuition}

Consider a no-covariate treatment \(A=(A_1,A_2)\), where the binary feature \(A_1\in\{0,1\}\) is politeness and the binary feature \(A_2\in\{0,1\}\) is formality. The four population-cell probabilities are arranged as
\[
\begin{array}{c|cc}
 & A_2=0 & A_2=1 \\ \hline
A_1=0 & p_1 & p_2 \\
A_1=1 & p_3 & p_4
\end{array}
\qquad\text{with}\qquad
p_1+p_2+p_3+p_4=1.
\]
The conditional mean outcome is
\[
    \mathbb E(Y\mid A_1,A_2)=A_1+2A_2.
\]
Thus, both features increase the outcome, but formality contributes twice as much as politeness. For this illustration, allow \(f(A)\in[0,1]\), and name its four cell-specific feature-on probabilities
\[
    a=f(0,0),\qquad b=f(0,1),\qquad c=f(1,0),\qquad d=f(1,1).
\]
Define the feature budget \(r\), the outcome-weighted feature-on mass \(q\), and the population mean outcome \(s\) by
\[
    r=ap_1+bp_2+cp_3+dp_4,\qquad q=2bp_2+cp_3+3dp_4,\qquad s=2p_2+p_3+3p_4.
\]
The exact raw causal contrast is
\[
    \Delta(f)=\Psi_f(1)-\Psi_f(0)=\frac{q-rs}{r(1-r)}.
\]
For a fixed budget \(r\in(0,1)\), the quantities \(r\) and \(s\) are fixed, so maximizing the raw contrast is equivalent to maximizing \(q\). The optimizer therefore fills the four cells in decreasing conditional-mean order: first \((A_1,A_2)=(1,1)\), then \((0,1)\), then \((1,0)\), and finally \((0,0)\). \Cref{tab:r_mapping} gives the resulting four budget regions. The displayed fractions apply when the corresponding cell probabilities are positive; a zero-probability cell is skipped.

\begin{table}[!htbp]
\centering
\small
\caption{The fixed-budget optimizer fills cells in decreasing outcome order. The table maps each budget range to the exact cell scores \(a,b,c,d\).}
\label{tab:r_mapping}
\begin{tabular}{@{}ccccc@{}}
\toprule
Budget range & \(a\) & \(b\) & \(c\) & \(d\) \\
\midrule
\((0,p_4]\) & \(0\) & \(0\) & \(0\) & \(r/p_4\) \\
\((p_4,p_2+p_4]\) & \(0\) & \((r-p_4)/p_2\) & \(0\) & \(1\) \\
\((p_2+p_4,p_2+p_3+p_4]\) & \(0\) & \(1\) & \((r-p_2-p_4)/p_3\) & \(1\) \\
\((p_2+p_3+p_4,1)\) & \((r-p_2-p_3-p_4)/p_1\) & \(1\) & \(1\) & \(1\) \\
\bottomrule
\end{tabular}
\end{table}

Small budgets therefore first select observations that are both polite and formal. In the special case \(p_4=0\), so that no observation is both, formal observations are selected before polite observations because their conditional mean outcome is larger. This is the concrete two-feature version of the upper-tail ranking in \Cref{prop:mif-optimizer-main}.

\begin{figure}[!htbp]
    \centering
    \begin{subfigure}[t]{0.82\textwidth}
        \centering
        \includegraphics[width=\textwidth]{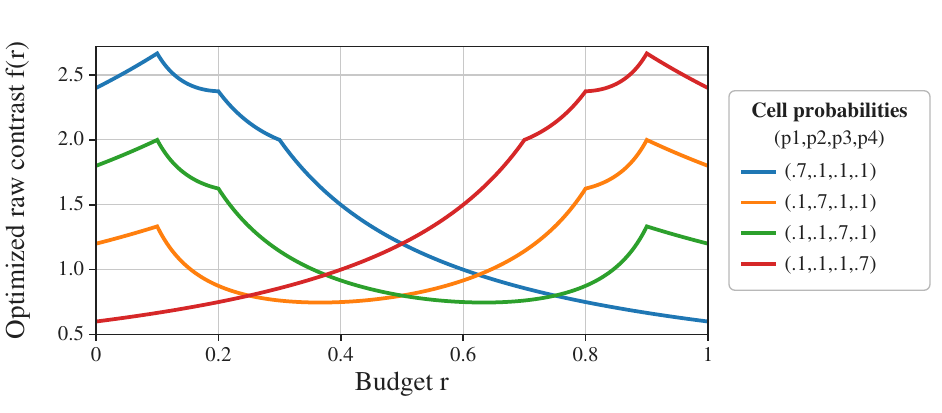}
        \caption{Each of the four curves places mass \(0.7\) on one cell and mass \(0.1\) on each other cell.}
    \end{subfigure}

    \medskip
    \begin{subfigure}[t]{0.82\textwidth}
        \centering
        \includegraphics[width=\textwidth]{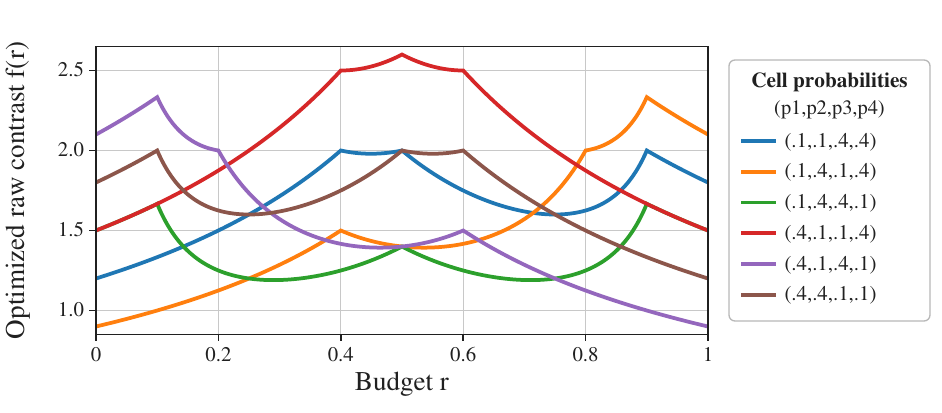}
        \caption{Each of the six curves places mass \(0.4\) on two cells and mass \(0.1\) on each other cell.}
    \end{subfigure}
    \caption{The budget maximizing the raw causal contrast depends on the population composition. The horizontal axis \(r\) is the feature budget, and the legacy vertical-axis label \(f(r)\) denotes the optimized raw contrast \(\Delta^*(r)\). Curves retain every probability configuration \((p_1,p_2,p_3,p_4)\) from the earlier manuscript. The optimized raw contrast can be nonmonotone in the budget. Higher is a larger raw contrast.}
    \label{fig:budget_illust}
\end{figure}

\Cref{fig:budget_illust} shows that changing \((p_1,p_2,p_3,p_4)\) changes the budget that maximizes the raw contrast. It also shows why the optimized raw contrast need not increase with \(r\): expanding the selected region admits lower-outcome cells and changes the normalization by \(r(1-r)\). The variance-weighted causal contrast instead records the total residual outcome captured by a selected region, so it supports comparisons across budgets without this normalization effect.

\long\def\mifobjectiveappendix{\section{Algebra and optimizer results for the MIF objective}
\label{app:mif-objective-proofs}

This appendix collects the algebra behind the variance-weighted causal contrast objective and the proofs of the optimizer statements used in \Cref{sec:causal-query}.  The main text keeps the identification proof because it is part of the causal argument.  The calculations here explain why the variance-weighted causal contrast objective becomes a residualized learning problem and why the optimizer has a threshold form.

\subsection{From the identified contrast to the residualized objective}

Suppose the shared-remainder conditions of
\Cref{cor:mif-fixed-residual} hold for every feature-scoring function
in the search class.  For any such function \(f\) satisfying the score bounds
in \Cref{sec:causal-query}, write
\[
    \pi_f(r,x)
    =
    \mathbb E\{f(A,X)\mid R=r,X=x\},
    \qquad
    m(r,x)
    =
    \mathbb E(Y\mid R=r,X=x).
\]
Substituting the shared-remainder policy formulas from
\Cref{cor:mif-fixed-residual} into the variance-weighted contrast
gives
\[
\begin{aligned}
\widetilde\Delta_X(f)
&=
\mathbb E\Big[
    \{1-\pi_f(R,X)\}
    \mathbb E\{f(A,X)Y\mid R,X\}\\
&\hspace{24mm}
    -\pi_f(R,X)
    \mathbb E[\{1-f(A,X)\}Y\mid R,X]
\Big]\\
&=
\mathbb E\Big[
    \mathbb E\{f(A,X)Y\mid R,X\}
    -\pi_f(R,X)m(R,X)
\Big]\\
&=
\mathbb E\left[
    f(A,X)\{Y-m(R,X)\}
\right].
\end{aligned}
\]
The second equality expands
\(\mathbb E[\{1-f(A,X)\}Y\mid R,X]\), and the last equality follows from
iterated expectation.
This is the shared-remainder working objective used in the optimizer and
estimation results below.

\subsection{Proof of \texorpdfstring{\Cref{prop:mif-optimizer-main}}{Proposition~\getrefnumber{prop:mif-optimizer-main}}}

\label{sec:fixed-budget-optimizer-proof}
\begin{proof}

Choose a fixed-budget variance-weighted optimizer
\[
    f_{b,\widetilde\Delta}^*
    \in
    \arg\max_{f\in\mathcal F_b}
    \widetilde\Delta_X(f).
\]
Define the conditional residual-outcome score
\[
    \eta(a,x)=\mathbb E\{Y-m(R,X)\mid A=a,X=x\}.
\]
Then
\[
\widetilde\Delta_X(f)
=
\mathbb E\{f(A,X)\eta(A,X)\}.
\]
Consider the unrestricted fixed-budget class satisfying \(\epsilon\leq f\leq1-\epsilon\) and \(\mathbb E\{f(A,X)\}=b\). Write the rescaled feature score as
\[
    f(A,X)=\epsilon+(1-2\epsilon)q(A,X),
    \qquad 0\leq q(A,X)\leq1.
\]
The corresponding rescaled budget is
\[
    \mathbb E\{q(A,X)\}
    =
    \frac{b-\epsilon}{1-2\epsilon}
    =q_b.
\]
The objective becomes
\[
    \epsilon\mathbb E\{\eta(A,X)\}
    +(1-2\epsilon)\mathbb E\{q(A,X)\eta(A,X)\}.
\]
The first term is constant in \(q\). The second term is maximized by assigning \(q=1\) to the largest values of \(\eta(A,X)\), assigning \(q=0\) to the smallest values, and assigning a fractional value on a threshold set if needed to make \(\mathbb E(q)=q_b\). Thus there exists a budget threshold~\(\lambda_b\) such that
\[
    f_{b,\widetilde\Delta}^*(a,x)=1-\epsilon \quad\text{when } \eta(a,x)>\lambda_b,
    \qquad
    f_{b,\widetilde\Delta}^*(a,x)=\epsilon \quad\text{when } \eta(a,x)<\lambda_b,
\]
with the threshold value chosen to satisfy the budget.  If \(\mathbb P\{\eta(A,X)=\lambda_b\}=0\), the threshold set has no mass, so the optimizer is unique up to equality almost surely.  Without a fixed budget, the pointwise maximizer sets \(f=1-\epsilon\) where \(\eta>0\) and \(f=\epsilon\) where \(\eta<0\), with uniqueness when \(\mathbb P\{\eta(A,X)=0\}=0\).
\end{proof}

\subsection{Treatment-only feature-scoring functions}
\label{app:mif-treatment-only}

The main text allows \(f\) to depend on both \(A\) and \(X\).  Sometimes the feature should be a property of the treatment alone.  In the course example, this asks whether the same textual property of a description is useful across course contexts, rather than allowing the feature to mean different things for different subjects or levels.

For treatment-only features \(f=f(A)\), the same variance-weighted causal contrast objective becomes
\[
    \widetilde\Delta_X(f)
    =
    \mathbb E\left[f(A)\{Y-m(R,X)\}\right].
\]
Define the treatment-only residual-outcome score
\[
    r(a)=\mathbb E\{Y-m(R,X)\mid A=a\}.
\]
Then
\[
    \widetilde\Delta_X(f)=\mathbb E\{f(A)r(A)\}.
\]
The unbudgeted optimizer is
\[
    f^*(a)
    =
    \begin{cases}
        1-\epsilon, & r(a)>0,\\
        \epsilon, & r(a)<0,\\
        c(a), & r(a)=0,
    \end{cases}
    \qquad c(a)\in[\epsilon,1-\epsilon],
\]
where \(c(a)\) is arbitrary.  At a fixed budget \(\mathbb E\{f(A)\}=b\), the optimizer selects treatment values in decreasing order of \(r(a)\).  The proof is identical to the rearrangement argument above, with \(\eta(A,X)\) replaced by \(r(A)\).

\subsection{Fixed-budget variance-weighted optimizer set}
\label{app:mif-uniqueness}

Let \(\mathcal F_b\) be the search class of feature-scoring functions, including the box constraint and budget.  The fixed-budget variance-weighted optimizer set is
\[
    \mathcal F_{b,\widetilde\Delta}^*
    =
    \arg\max_{f\in\mathcal F_b}
    \widetilde\Delta_X(f).
\]
Under \Cref{ass:mif-sutva,ass:mif-exchangeability}, \Cref{prop:iden} identifies \(\Psi_f(1)\), \(\Psi_f(0)\), \(\Delta_X(f)\), and \(\widetilde\Delta_X(f)\) for every fixed \(f\in\mathcal F_b\) with its feature-specific remainder. Hence the variance-weighted maximizer set over a fixed search class is identified. If the set contains a single equivalence class under equality almost surely, the fixed-budget variance-weighted optimizer is unique. The threshold conditions in \Cref{prop:mif-optimizer-main} give explicit uniqueness conditions in the unrestricted class.
}

\section{Technical discussions around the feature remainder}

\label{sec:remainder-definition}

For each candidate feature \(f(A,X)\), a remainder
\(R_f=r_f(A,X)\) is defined relative to a feature coordinate system that
separates the feature variation from all remaining treatment variation; it is the collection of treatment properties that the
feature intervention preserves.

\textbf{Definition of the remainder.}
For each candidate feature \(f\), define the feature-specific remainder \(R_f=r_f(A,X)\). The remainder is defined through a
feature coordinate system satisfying
\[
    A\longleftrightarrow \{f(A,X),R_f,X\},
\]
meaning that there exists a deterministic decoder \(D_f\) such that
\[
    A=D_f\{f(A,X),R_f,X\}.
\]
Thus, the feature score, remainder, and context together retain the full
treatment information. Holding \(R_f\) and \(X\) fixed changes only the
feature represented by \(f\), while preserving all other treatment
properties encoded by \(R_f\).

The construction of such feature coordinate systems depends on the
application. In the course example, if \(f\) scores formality, then
\(R_f\) contains the description properties that are intended to remain
unchanged, such as semantic content, course facts, length, examples, and
other writing characteristics. \Cref{sec:cont-style-sep}
discusses practical constructions of these feature coordinate systems.

The numerical encoding of a fixed remainder partition is not unique. The
argument below shows that the resulting policy outcome is invariant to a
one-to-one recoding of that partition. 

\textbf{Constructing feature coordinate systems.}
The definition of the remainder can be implemented by constructing an
invertible feature coordinate system. For each candidate feature \(f\),
learn a transformation
\[
    \Phi_f(a,x)
    =
    \{f(a,x),r_f(a,x)\},
\]
with decoder
\[
    G_f\{f(a,x),r_f(a,x),x\}=a.
\]
The first coordinate is the feature score, and \(R_f=r_f(A,X)\) contains the remaining treatment information held fixed by the feature
intervention.

Because the transformation is invertible,
\[
    A=G_f\{f(A,X),R_f,X\},
\]
so the feature score, remainder, and context together reconstruct the
full treatment. Thus, holding \(R_f\) and \(X\) fixed changes only the
feature represented by \(f\).

A valid feature coordinate system should also ensure that changing the
feature score does not change the remainder. For differentiable
coordinates, let \(d\) denote the dimension of treatment representation~\(a\). Local orthogonality requires
\[
    J_a r_f(a,x)\nabla_a f(a,x)=0,
\]
where \(J_a r_f\) is the Jacobian of the remainder and
\(\nabla_a f\) is the gradient of the feature score. This condition makes
the feature direction locally orthogonal to the remainder coordinates.

Together with
\[
    \operatorname{rank}\{J_a r_f(a,x)\}=d-1,
\]
invertibility, and connected feature paths, this makes \(R_f\) an exact
remainder: it contains all treatment information except the variation
represented by \(f\).

In practice, the coordinate system can be learned using an invertible
neural network or another invertible representation. A reconstruction
loss can be written as
\[
    \mathcal L_{\mathrm{recon}}(f)
    =
    \mathbb E
    \left[
        \left\|
        G_f\{f(A,X),r_f(A,X),X\}-A
        \right\|^2
    \right],
\]
and the orthogonality condition can be encouraged through
\[
    \mathcal L_{\mathrm{orth}}(f)
    =
    \mathbb E
    \left[
        \left\|
        J_a r_f(A,X)
        \nabla_a f(A,X)
        \right\|^2
    \right].
\]
At zero reconstruction and orthogonality loss, together with the rank
and connected-path conditions, \(R_f\) is an exact remainder associated
with the feature score \(f\).

\textbf{Linear feature coordinates as a special case.}
The linear construction is a special case of the above coordinate
system. Suppose
\[
    f(a,x)=\ell(v_f^\top a,x),
\]
where \(\ell\) is a strictly increasing link function and \(v_f\) is the unit feature direction:
\[
    \|v_f\|=1.
\]
Let \(I\) denote the \(d\times d\) identity matrix, and define
\[
    R_f=(I-v_fv_f^\top)A.
\]
Because
\[
    A
    =
    v_f(v_f^\top A)
    +
    (I-v_fv_f^\top)A,
\]
the feature score and remainder reconstruct the original treatment.
Changing \(f\) while holding \(R_f\) fixed changes only the component of
the treatment along the feature direction.

The nonlinear construction generalizes this idea by allowing the feature
coordinate system to define curved feature paths instead of straight
lines. Both constructions have the same interpretation: \(f\) identifies
the treatment variation being changed, while \(R_f\) records all other
treatment information.

\textbf{Invariance of policy outcomes to remainder encoding.}
For a fixed remainder partition, the policy outcome does not change under a one-to-one re-encoding. Suppose
\[
    \widetilde R_f=h_f(R_f,X),
\]
where \(h_f(\cdot,x)\) is one-to-one for every context \(x\). Then
\(R_f\) and \(\widetilde R_f\) define the same remainder strata because
\[
    \{R_f=r,X=x\}
    =
    \{\widetilde R_f=h_f(r,x),X=x\}.
\]
Therefore,
\[
    p(a\mid R_f=r,X=x)
    =
    p(a\mid
    \widetilde R_f=h_f(r,x),X=x),
\]
and
\[
    \pi_f(r,x)
    =
    \mathbb E\{f(A,X)\mid R_f=r,X=x\},
\]
is unchanged after relabeling.

Consequently, the feature-on and feature-off policies defined by
conditioning on the remainder are unchanged:
\[
    \widetilde p(a\mid \widetilde r,x,W_f=w)
    =
    p(a\mid r,x,W_f=w),
\]
after the corresponding relabeling of remainder values. Hence
\[
    \widetilde\Psi_f(w)=\Psi_f(w).
\]
Therefore, the policy outcome depends only on the remainder partition and
not on the particular numerical encoding used for \(R_f\).

\section{Proof of \texorpdfstring{\Cref{prop:iden}}{Proposition~\getrefnumber{prop:iden}}}
\label{app:iden-proof}

\begin{proof}
Fix \(R_f=r\) and \(X=x\).  The definition of the feature-on policy gives
\[
\begin{aligned}
    \Psi_f(1\mid r,x)
    &=
    \int
    \mathbb E\{Y(a)\mid R_f=r,X=x\}
    \frac{f(a,x)}{\pi_f(r,x)}
    p(a\mid r,x)\,da\\
    &=
    \int
    \mathbb E(Y\mid A=a,R_f=r,X=x)
    \frac{f(a,x)}{\pi_f(r,x)}
    p(a\mid r,x)\,da\\
    &=
    \frac{
        \mathbb E\{f(A,X)Y\mid R_f=r,X=x\}
    }{
        \pi_f(r,x)
    }.
\end{aligned}
\]
The second equality is where the causal assumptions enter.  Because \(R_f=r_f(A,X)\), \Cref{ass:mif-exchangeability} implies \(Y(a)\perp(A,R_f)\mid X\).  Thus adding \(A=a\) to the conditioning set does not change the mean of \(Y(a)\).  \Cref{ass:mif-sutva} makes \(Y(a)\) well-defined, and its consistency part replaces \(Y(a)\) by the observed \(Y\) when \(A=a\).  The last equality is the law of iterated expectation.

Averaging over the observed distribution of \((R_f,X)\) gives the feature-on formula.  Replacing \(f\) by \(1-f\) gives the feature-off formula.  The overlap bound above keeps both denominators positive; because each policy is a reweighting of \(p(a\mid r,x)\), every treatment used in the calculation is on the observed conditional support.

It remains to identify the variance-weighted contrast.  Substituting the two
identified conditional policy outcomes into its definition gives
\[
\begin{aligned}
\widetilde\Delta_X(f)
&=
\mathbb E\Big[
    \pi_f(R_f,X)\{1-\pi_f(R_f,X)\}
    \{\Psi_f(1\mid R_f,X)-\Psi_f(0\mid R_f,X)\}
\Big]\\
&=
\mathbb E\Big[
    \{1-\pi_f(R_f,X)\}
    \mathbb E\{f(A,X)Y\mid R_f,X\}\\
&\hspace{24mm}
    -\pi_f(R_f,X)
    \mathbb E[\{1-f(A,X)\}Y\mid R_f,X]
\Big]\\
&=
\mathbb E\Big[
    \mathbb E\{f(A,X)Y\mid R_f,X\}
    -\pi_f(R_f,X)m_f(R_f,X)
\Big]\\
&=
\mathbb E\left[
    f(A,X)\{Y-m_f(R_f,X)\}
\right],
\end{aligned}
\]
where
\[
m_f(r,x)=\mathbb E(Y\mid R_f=r,X=x).
\]
The third equality uses
\[
\mathbb E[\{1-f(A,X)\}Y\mid R_f,X]
=
m_f(R_f,X)-\mathbb E\{f(A,X)Y\mid R_f,X\},
\]
and the last equality follows from iterated expectation.
Thus the variance-weighted objective is identified.

In the course example, \(m_f(R_f,X)\) is expected enrollment among courses with the same context and the same values of the content and writing properties recorded in \(R_f\).  The objective measures the residual enrollment signal associated with the feature scored by \(f\) within those comparisons.
\end{proof}

\section{Proof of \texorpdfstring{\Cref{cor:mif-fixed-residual}}{Corollary~\getrefnumber{cor:mif-fixed-residual}}}
\label{app:mif-fixed-residual-proof}

\begin{proof}
Because \(R\) is a function of \((A,X)\), full-treatment
unconfoundedness implies
\[
Y(a)\perp A\mid R,X.
\]
Standard stochastic-intervention arguments therefore identify the feature-on
and feature-off policies based on \(R\):
\[
\Psi_f(1)
=
\mathbb E\left[
\frac{f(A,X)Y}{\mathbb E\{f(A,X)\mid R,X\}}
\right],
\]
and similarly for \(w=0\).

It remains to show that the shared-remainder policy agrees with the
feature-specific policy based on \(R_f\).  Write
\[
R_f=(R,U_f).
\]
The feature-specific policy satisfies
\[
p(a\mid r,u,x,W_f=w)
\propto
p(W_f=w\mid a,x)p(a\mid r,u,x).
\]
Because
\[
p(W_f=1\mid A=a,X=x)=f(a,x),
\]
the only difference between the two policies is the distribution of
\(U_f\) after conditioning on \(W_f\).

\Cref{ass:mif-fixed-residual} is equivalent to
\(W_f\perp U_f\mid R,X\).  Hence,
\[
p(u_f\mid r,x,W_f=w)
=
p(u_f\mid r,x).
\]
It follows that
\[
p(a\mid r,x,W_f=w)
=
\int
p(a\mid r,u_f,x,W_f=w)
p(u_f\mid r,x)\,du_f.
\]
Thus the shared-remainder policy averages the feature-specific policy over
the unchanged distribution of \(U_f\), and the two policies induce the same
distribution over full treatments after marginalizing over \(U_f\).

Finally, let
\[
\pi_R(R,X)=\mathbb E\{f(A,X)\mid R,X\}.
\]
The shared-remainder condition gives
\[
\mathbb E\{f(A,X)\mid R_f,X\}=\pi_R(R,X),
\]
and, because \(R\) is a component of \(R_f\),
\[
\mathbb E\{m_f(R_f,X)\mid R,X\}=m(R,X).
\]
Therefore, \Cref{prop:iden} and iterated expectation give
\[
\begin{aligned}
\widetilde\Delta_X(f)
&=
\mathbb E\{f(A,X)Y\}
-
\mathbb E\{f(A,X)m_f(R_f,X)\}\\
&=
\mathbb E\{f(A,X)Y\}
-
\mathbb E\{\pi_R(R,X)m_f(R_f,X)\}\\
&=
\mathbb E\{f(A,X)Y\}
-
\mathbb E\{\pi_R(R,X)m(R,X)\}\\
&=
\mathbb E\left[
f(A,X)\{Y-m(R,X)\}
\right].
\end{aligned}
\]
\end{proof}

\mifobjectiveappendix

\section{Proofs of \texorpdfstring{\Cref{prop:mif-consistency,prop:heldout-clt}}{Propositions~\getrefnumber{prop:mif-consistency} and~\getrefnumber{prop:heldout-clt}}}
\label{app:mif-estimation-proofs}

This appendix proves the estimation statements in \Cref{sec:causal-estimation}.  The results concern the shared-remainder variance-weighted causal contrast objective
\[
    \widetilde\Delta_X(f_\theta)=
    \mathbb E\left[
        f_\theta(A,X)\{Y-m(R,X)\}
    \right],
    \qquad
    m(r,x)=\mathbb E(Y\mid R=r,X=x).
\]

\subsection{Proof of \texorpdfstring{\Cref{prop:mif-consistency}}{Proposition~\getrefnumber{prop:mif-consistency}}}
\label{app:mif-consistency-proof}
\begin{proof}

Define the oracle and feasible empirical objectives by
\[
\widetilde\Delta_{X,n}(f_\theta)
=
\frac{1}{n}
\sum_{i=1}^n
f_\theta(A_i,X_i)
\{Y_i-m(R_i,X_i)\},
\]
and
\[
\widehat{\widetilde\Delta}_{X,n}(f_\theta)
=
\frac{1}{n}
\sum_{i=1}^n
f_\theta(A_i,X_i)
\{Y_i-\widehat m(R_i,X_i)\}.
\]
Since \(\Theta\) is compact and \(\widetilde\Delta_X(f_\theta)\) has a unique maximizer, it suffices by the argmax theorem to show
\[
\sup_{\theta\in\Theta}
\left|
\widehat{\widetilde\Delta}_{X,n}(f_\theta)
-
\widetilde\Delta_X(f_\theta)
\right|
\overset{p}{\longrightarrow}0.
\]
By the triangle inequality,
\[
\sup_{\theta\in\Theta}
\left|
\widehat{\widetilde\Delta}_{X,n}(f_\theta)
-
\widetilde\Delta_X(f_\theta)
\right|
\leq
\sup_{\theta\in\Theta}
\left|
\widetilde\Delta_{X,n}(f_\theta)
-
\widetilde\Delta_X(f_\theta)
\right|
+
\sup_{\theta\in\Theta}
\left|
\widehat{\widetilde\Delta}_{X,n}(f_\theta)
-
\widetilde\Delta_{X,n}(f_\theta)
\right|.
\]
The first term is \(o_p(1)\) by a uniform law of large numbers.  The summand class is indexed by a compact parameter space, \(f_\theta(A,X)\) is continuous in \(\theta\), and
\[
|f_\theta(A,X)\{Y-m(R,X)\}|
\leq
|Y-m(R,X)|,
\qquad
\mathbb E|Y-m(R,X)|\leq2\mathbb E|Y|<\infty.
\]
For the second term,
\[
\widehat{\widetilde\Delta}_{X,n}(f_\theta)
-
\widetilde\Delta_{X,n}(f_\theta)
=
-
\frac{1}{n}
\sum_{i=1}^n
f_\theta(A_i,X_i)
\{\widehat m(R_i,X_i)-m(R_i,X_i)\}.
\]
Thus
\[
\begin{aligned}
\sup_{\theta\in\Theta}
\left|
\widehat{\widetilde\Delta}_{X,n}(f_\theta)
-
\widetilde\Delta_{X,n}(f_\theta)
\right|
&\leq
\frac{1}{n}
\sum_{i=1}^n
\sup_{\theta\in\Theta}|f_\theta(A_i,X_i)|
|\widehat m(R_i,X_i)-m(R_i,X_i)| \\
&\leq
\sup_{r,x}|\widehat m(r,x)-m(r,x)|
=o_p(1).
\end{aligned}
\]
Combining the two bounds gives uniform convergence of \(\widehat{\widetilde\Delta}_{X,n}\) to \(\widetilde\Delta_X\), and the argmax theorem yields \(\widehat\theta\overset{p}{\longrightarrow}\theta_0\).
\end{proof}

\subsection{Proof of \texorpdfstring{\Cref{prop:heldout-clt}}{Proposition~\getrefnumber{prop:heldout-clt}}}
\label{app:heldout-clt-proof}
\begin{proof}

Conditional on the training data, \(\widehat f\) is fixed and the held-out summands \(\widehat f(A_i,X_i)\{Y_i-m(R_i,X_i)\}\) are independent and identically distributed with mean \(\widetilde\Delta_X(\widehat f)\) and variance \(V(\widehat f)\). The central limit theorem gives
\[
    \sqrt{n_{\mathrm{eval}}}
    \left\{
        \widehat{\widetilde\Delta}_{X,n_{\mathrm{eval}}}(\widehat f)
        -
        \widetilde\Delta_X(\widehat f)
    \right\}
    \rightsquigarrow
    \mathcal N\{0,V(\widehat f)\}.
\]
The variance can be estimated by the empirical variance of the held-out summands. If \(m\) is replaced by \(\widehat m\) and the induced difference in the held-out average is \(o_p(n_{\mathrm{eval}}^{-1/2})\), Slutsky's theorem gives the same limit.
\end{proof}

\subsection{Compactness relaxation for \texorpdfstring{\Cref{prop:mif-consistency}}{Proposition~\getrefnumber{prop:mif-consistency}}}

The compactness assumption on \(\Theta\) can be replaced by a tightness condition: for every \(\delta>0\), there exists a compact subset \(K_\delta\subset\Theta\) containing \(\theta_0\) such that, for all sufficiently large \(n\),
\[
    \mathbb P(\widehat\theta\in K_\delta)\geq1-\delta.
\]
The same proof applies on \(K_\delta\), with the remaining probability controlled by tightness.

\section{Theoretical guarantees for content--style separation}
\label{sec:cont-style-sep-theory}

We now introduce notation for stating this result. Let $\mathcal{C}$ denote the set of contents, $\mathcal{S}$ the set of styles, and $\mathcal{A}$ the set of text embeddings. Assume there exists a deterministic, injective map
\[
    h:\mathcal{C}\times\mathcal{S}\rightarrow\mathcal{A},
\]
such that each text embedding is uniquely determined by its content and style: $h(c,s)=h(c',s')$ implies $(c,s)=(c',s')$. Let $\mu_C$ and $\mu_S$ denote the distributions of content and style, and assume their joint distribution is $\mu=\mu_C\otimes\mu_S$. Thus, content and style are distributionally independent, and $c,c'\stackrel{\mathrm{iid}}{\sim}\mu_C$ and $s,s'\stackrel{\mathrm{iid}}{\sim}\mu_S$ below are mutually independent draws.

Introduce a content--style encoder
\[
    \Gamma=(\gamma_C,\gamma_S):
    \mathcal{A}\rightarrow\mathcal{Z}_C\times\mathcal{Z}_S,
\]
where \(\mathcal Z_C\) and \(\mathcal Z_S\) are the content and style representation spaces. The encoder decomposes a text embedding $a\in\mathcal{A}$ into $\Gamma(a)=(\gamma_C(a),\gamma_S(a))$, and a decoder $G:\mathcal{Z}_C\times\mathcal{Z}_S\rightarrow\mathcal{A}$. In the notation used throughout the rest of the paper,
\[
    A_{\mathrm{unmodifiable}}=\gamma_C(A),
    \qquad
    A_{\mathrm{modifiable}}=\gamma_S(A).
\]
For the technical steps below, we treat $\mathcal C$ and $\mathcal S$ as measurable spaces, take $\mathcal A$, $\mathcal Z_C$, and $\mathcal Z_S$ to be normed spaces with their usual Borel $\sigma$-algebras, and assume that $h$, $\Gamma$, and $G$ are measurable. These mild conditions ensure that the expectations below and our applications of Tonelli's and Fubini's theorems are well defined. We define the following losses:
\begin{enumerate}
    \item $\mathcal{L}_{\textrm{recon}}(\Gamma,G) = \mathbb{E}_{(c,s) \sim \mu}\left[ \lVert G\{\gamma_C(h(c,s)),\gamma_S(h(c,s))\}-h(c,s) \rVert^2 \right]$, i.e., reconstruction loss.
    \item $\mathcal{L}_{\textrm{cont}}(\Gamma) = \mathbb{E}_{c \sim \mu_C,\,s,s' \stackrel{\mathrm{iid}}{\sim} \mu_S}\left[ \lVert \gamma_C(h(c,s))-\gamma_C(h(c,s')) \rVert^2 \right]$, i.e., content loss.
    \item $\mathcal{L}_{\textrm{style}}(\Gamma) = \mathbb{E}_{s \sim \mu_S,\,c,c' \stackrel{\mathrm{iid}}{\sim} \mu_C}\left[ \lVert \gamma_S(h(c,s))-\gamma_S(h(c',s)) \rVert^2 \right]$, i.e., style loss.
\end{enumerate}

The following result states that zero reconstruction, content, and style losses achieve content-style separation.

\begin{proposition}
    \label{prop:content-style-sep}
    Suppose $\mathcal{L}_{\textrm{recon}}(\Gamma,G)=\mathcal{L}_{\textrm{cont}}(\Gamma)=\mathcal{L}_{\textrm{style}}(\Gamma)=0$. Then there exist measurable functions $\phi_C:\mathcal{C}\rightarrow\mathcal{Z}_C$ and $\phi_S:\mathcal{S}\rightarrow\mathcal{Z}_S$ such that, for $\mu$-almost every $(c,s)$,
    \[
        \gamma_C(h(c,s))=\phi_C(c),\qquad
        \gamma_S(h(c,s))=\phi_S(s).
    \]
    Moreover, for $\mu$-almost every $(c,s)$,
    \[
        G\{\phi_C(c),\phi_S(s)\}=h(c,s).
    \]
    The maps $\phi_C$ and $\phi_S$ are injective on full-measure subsets of $\mathcal C$ and $\mathcal S$. If another measurable encoder--decoder pair $(\Gamma',G')$ of the same type satisfies
    \[
        \mathcal{L}_{\textrm{recon}}(\Gamma',G')
        =\mathcal{L}_{\textrm{cont}}(\Gamma')
        =\mathcal{L}_{\textrm{style}}(\Gamma')
        =0,
    \]
    and induces maps $\phi'_C$ and $\phi'_S$ as above, then there are full-measure sets $\mathcal C_*\subseteq\mathcal C$ and $\mathcal S_*\subseteq\mathcal S$ and bijections
    \[
        b_C:\phi_C(\mathcal C_*)\rightarrow\phi'_C(\mathcal C_*),
        \qquad
        b_S:\phi_S(\mathcal S_*)\rightarrow\phi'_S(\mathcal S_*),
    \]
    such that $\phi'_C=b_C\circ\phi_C$ on $\mathcal C_*$ and $\phi'_S=b_S\circ\phi_S$ on $\mathcal S_*$. Thus, the two representations differ only by coordinatewise one-to-one reparameterizations on sets of probability one.
\end{proposition}

\begin{proof}[Proof of \Cref{prop:content-style-sep}]
    Each loss averages a nonnegative quantity, so a zero loss means that the corresponding equality holds almost everywhere. In particular, zero content loss gives
    \[
        \gamma_C(h(c,s))=\gamma_C(h(c,s')),
    \]
    for $\mu_C\otimes\mu_S\otimes\mu_S$-almost every $(c,s,s')$. Tonelli's theorem therefore lets us fix a style $s_0$ such that
    \[
        \gamma_C(h(c,s))=\gamma_C(h(c,s_0)),
    \]
    for $\mu_C\otimes\mu_S$-almost every $(c,s)$. Define $\phi_C(c)=\gamma_C(h(c,s_0))$. The same argument applied to zero style loss lets us fix a content $c_0$ and define $\phi_S(s)=\gamma_S(h(c_0,s))$. These functions are measurable because they are fixed sections of measurable maps. Thus, the content and style losses give the two coordinate factorizations. Combining them with zero reconstruction loss gives
    \[
        G\{\phi_C(c),\phi_S(s)\}=h(c,s),
    \]
    for $\mu$-almost every $(c,s)$.

    To establish injectivity, let $B$ be the full-measure set on which this reconstruction identity holds. Fubini's theorem gives a full-measure set of contents for which
    \[
        B_c=\{s\in\mathcal S:(c,s)\in B\},
    \]
    has full $\mu_S$-measure. For any two contents $c,c'$ in that set, $B_c\cap B_{c'}$ contains a common style $s$. If $\phi_C(c)=\phi_C(c')$, then
    \[
        h(c,s)
        =G\{\phi_C(c),\phi_S(s)\}
        =G\{\phi_C(c'),\phi_S(s)\}
        =h(c',s),
    \]
    so injectivity of $h$ gives $c=c'$. The same argument, with content and style reversed, proves that $\phi_S$ is injective on a full-measure set.

    Finally, apply the same argument to a second zero-loss pair $(\Gamma',G')$ and restrict both representations to common full-measure sets $\mathcal C_*$ and $\mathcal S_*$ on which their coordinate maps are injective. Define
    \[
        b_C\{\phi_C(c)\}=\phi'_C(c),
        \qquad
        b_S\{\phi_S(s)\}=\phi'_S(s).
    \]
    Injectivity of both pairs of coordinate maps makes $b_C$ and $b_S$ well defined and bijective between the images stated in the proposition. Both representations reconstruct the same text embedding, so for $\mu$-almost every $(c,s)\in\mathcal C_*\times\mathcal S_*$,
    \[
        G'\!\left[b_C\{\phi_C(c)\},b_S\{\phi_S(s)\}\right]
        =h(c,s)
        =G\{\phi_C(c),\phi_S(s)\}.
    \]
    This proves uniqueness up to the stated coordinatewise bijections.
\end{proof}

The product-measure assumption above is distributional independence and is used only in the zero-loss proposition and its proof. It permits the independent content and style draws and the common full-measure sections used above. The independence-of-support loss below instead concerns product support and does not require distributional independence.

It is easy to see that separation is not successful if at least one of the reconstruction, content, and style losses is removed.  Without reconstruction, information can be discarded.  Without the content loss, content can leak into the style representation.  Without the style loss, style can leak into the content representation.

The independence-of-support loss is inspired by \citet{wang2024desiderata}, namely
\[
    \textrm{supp}([A_{\mathrm{unmodifiable}}, A_{\mathrm{modifiable}}]) = \textrm{supp}(A_{\mathrm{unmodifiable}}) \times \textrm{supp}(A_{\mathrm{modifiable}}).
\]
The intuition is that if content and style embeddings truly have independent supports, then any combination of a content vector $A_{\mathrm{unmodifiable}}$ and a style vector $A_{\mathrm{modifiable}}$ should correspond to some point within the joint support.

To evaluate the independence-of-support score (IOSS), we normalize each embedding dimension to $[0,1]$ and generate $K$ block samples by pairing a style vector from one data point with a content vector from another.  These synthetic pairs should lie within the joint support if the product-support condition holds.  We then measure how far these synthetic pairs are from the nearest actual data point in the joint embedding space.  The IOSS loss is defined as the maximum of these nearest-neighbor distances.  The smaller the value, the closer the synthetic pairs are to real observations, and the more the independence-of-support condition is satisfied.

\section{Empirical studies for the content-style separation algorithm}

\label{sec:cont-style-sep-examples}

The three examples are controlled food/drink reviews, news headlines embedded by a large language model (LLM), and controlled restaurant reviews.

\subsection{Controlled food/drink reviews}
\label{sec:cont-style-food}
\textbf{Loss behavior.}
We return to the simple food/drink review example.  \Cref{fig:traj} shows the loss trajectories for each loss component, excluding reconstruction loss.  The paired content and style losses are low during the early training phase because the content and style embeddings are initially roughly the same regardless of the text.

\begin{figure}[!htbp]
    \centering
    \includegraphics[width=0.98\textwidth]{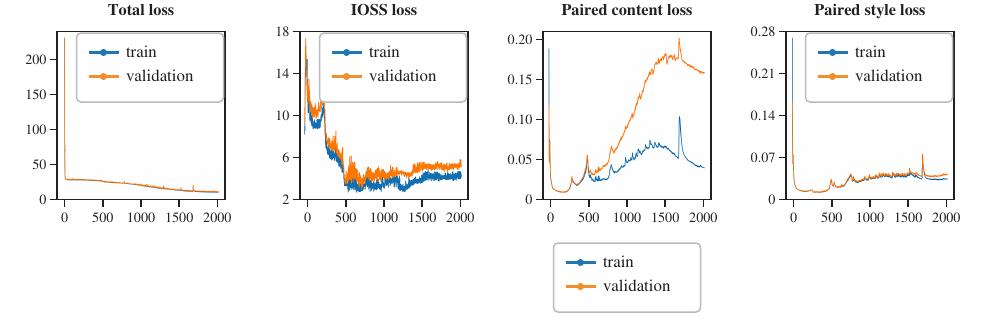}
    \caption{The paired losses begin low because content and style embeddings are initially similar across texts. The panels show the content--style separation training and validation losses across epochs. Lower is better.}
\label{fig:traj}
\end{figure} 

\textbf{Probe results.}
To check whether content-style separation is successful, we train multilayer perceptron (MLP) classifiers to predict content and style from $A_{\mathrm{unmodifiable}}$ and $A_{\mathrm{modifiable}}$.  From $A_{\mathrm{unmodifiable}}$ (content embedding), the test set accuracies of predicting content and style are $1.0$ and $0.53$.  From $A_{\mathrm{modifiable}}$ (style embedding), the test set accuracies of predicting content and style are $0.52$ and $1.0$.  These results show that the content-style separation is successful.

\textbf{Reconstructions.}
Manual inspection also suggests that reconstruction,
\[
    A \rightarrow A_{\mathrm{unmodifiable}}, A_{\mathrm{modifiable}} \rightarrow A_{\textrm{recon}},
\]
is successful on the test set.  Some examples are:
\begin{itemize}
    \item $A$: \textit{This tempeh is forgettable and lifeless.}\\
    $A_{\textrm{recon}}$: \textit{This tempeh is unpleasant and unattractive.}
    \item $A$: \textit{This pumpkin is fabulous and glorious.}\\
    $A_{\textrm{recon}}$: \textit{This pumpkin is amazing and wonderful.}
    \item $A$: \textit{This broccoli is magnificent and flavorful.}\\
    $A_{\textrm{recon}}$: \textit{This broccoli is delicious and delicious.}
\end{itemize}

\subsection{News headlines with LLM-based embeddings}
\label{sec:cont-style-headlines}
\textbf{Data and supervision.}
We next attempt to learn two functions, one that extracts the \textit{content} and one that extracts the \textit{style} of a text as represented by its embedding.  On top of these, we learn another function that reconstructs the original text embedding from the content and style embeddings.

We assemble a data set of headlines from the following general-news domains:
\begin{quote}
\small\itshape\raggedright
apnews.com, dailymail.co.uk, sfgate.com, chron.com,
timesofindia.indiatimes.com, thehindu.com, theguardian.com, allafrica.com,
mirror.co.uk, chronicleonline.com, washingtonpost.com, forbes.com,
newsnow.co.uk, nytimes.com, stuff.co.nz, theepochtimes.com, thestar.com.my,
and abcnews.go.com.
\end{quote}
We use the Hugging Face
\href{https://huggingface.co/facebook/fasttext-language-identification}{fastText language-identification model}
to identify English headlines. We keep only English headlines with no commas and with length between 50 and 80 characters, then take a random subset of $20{,}000$ headlines.

For each headline, we use the Llama 3.1 8B Instruct model to output two headlines: one with the \textit{same content} as the original headline, and another with the \textit{same style} as the original headline.  We use the following prompts:
\begin{enumerate}[(A)]
    \item \textbf{Same content, different style.}
    
    \texttt{SYSTEM:} You are a highly skilled news headline paraphraser. Your task is to paraphrase news headlines that completely preserve content but change style.

    \texttt{USER:} Original headline: \texttt{[headline]}

    Rewrite this headline, ensuring that ALL the following requirements are satisfied:
    \begin{enumerate}
        \item The content is strictly preserved.
        \item The style is completely changed.
        \item Most importantly, do not add, remove, or subjectively interpret any information.
    \end{enumerate}
    Output ONLY the new headline.

    \item \textbf{Same style, different content.}
    
    \texttt{SYSTEM:} You are a highly skilled linguistic analyst and creative writer. Your task is to extract grammatical skeletons from sentences and generate new headlines that preserve style but change content to be about university.

    \texttt{USER:} Original headline: \texttt{[headline]}

    Rewrite this headline, ensuring that ALL the following requirements are satisfied:
    \begin{enumerate}
        \item The grammatical skeleton is strictly preserved.
        \item The content is about university.
        \item Most importantly, the new headline should not contain the same nouns, verbs, or adjectives as the original headline.
    \end{enumerate}
    Output ONLY the new headline.
\end{enumerate}

Below are some examples:
\begin{itemize}
    \item \textit{Original:} Whoopi Goldberg tells racist fantasy fans to `get a job'. \\
    \textit{Same content:} Whoopi Goldberg Urges Fans of Racist Fantasy to Seek Employment Opportunities. \\
    \textit{Same style:} Academic critics advise entitled student protesters to `secure a scholarship'.

    \item \textit{Original:} Analysis: Ukraine war forces United Arab Emirates to hedge. \\
    \textit{Same content:} UAE Adapts to Ukraine Conflict by Diversifying Its Stance. \\
    \textit{Same style:} Analysis: Rising tuition costs forces Canadian students to hedge.

    \item \textit{Original:} Queen Camilla is a mom of 2: What to know about her children. \\
    \textit{Same content:} Meet Queen Camilla's Two Offspring: A Closer Look at Her Family Life. \\
    \textit{Same style:} Professor Thompson is a mentor of 10: What to know about her students.
\end{itemize}

\textbf{Training objective.}
The training loss consists of three components:
\begin{enumerate}
    \item \textbf{Reconstruction loss}, ensuring that we can reconstruct the original embedding from the content and style embeddings.
    \item \textbf{Content loss}, ensuring that texts with the same content have similar embeddings.
    \item \textbf{Style loss}, ensuring that texts with the same style have similar embeddings.
\end{enumerate}

\textbf{Failure mode.}
We observe that reconstruction works well.  However, content-style separation does not work very well, because ``style'' mostly captures grammatical skeleton.  Two sentences with the same grammatical skeleton are likely to have the same style, but the converse need not hold.  Style may be more abstract than grammatical form.  Some examples are:
\begin{enumerate}\raggedright
    \item 
    \textit{Original sentence:} Helen Flanagan flaunts her sensational figure in sky blue bikini. \\
    \textit{Target style:} Jane Campion wins directing Oscar for 'Power of the Dog'. \\
    \textit{Rewritten sentence:} Jane Flanagan flaunts ecstasy on Halloween for 'Heaven of bikini.'

    \item 
    \textit{Original sentence:} Helen Flanagan flaunts her sensational figure in sky blue bikini. \\
    \textit{Target style:} NPC: Wellington beat Tasman 25-19 after controversial Connor Garden-Bachop try. \\
    \textit{Rewritten sentence:} Britishline Flannigan kicks Helen after glamorous butts.

    \item 
    \textit{Original sentence:} Court rules NY Democrats gerrymandered district boundaries. \\
    \textit{Target style:} 'Teflon' Mark Rutte is longest-serving Dutch prime minister. \\
    \textit{Rewritten sentence:} 'Nodes' NYC courting's become last of state district judges.

    \item 
    \textit{Original sentence:} Court rules NY Democrats gerrymandered district boundaries. \\
    \textit{Target style:} Togo: How AI Helped Deliver Cash Aid to Many of the Poorest People in Togo. \\
    \textit{Rewritten sentence:} NYC: Democratic Policy Delivers to Neighborhoods in New York City.

    \item 
    \textit{Original sentence:} West Australians boast about their abundance of lettuce. \\
    \textit{Target style:} Vietnam to lift Covid-19 restrictions on international flights. \\
    \textit{Rewritten sentence:} Australia unleashes savagery on outdoor vegetarians.

    \item 
    \textit{Original sentence:} West Australians boast about their abundance of lettuce. \\
    \textit{Target style:} How to invest in sustainable food: Six best fund idea. \\
    \textit{Rewritten sentence:} How to farm in Australian wildlife: Sexy New Zealand.
\end{enumerate}

This experiment is useful because it shows what can go wrong when style supervision is too narrow.  If the target style is defined mainly through grammatical skeleton, the learned representation may capture surface form rather than the broader writing choices that are actually modifiable.

\subsection{Restaurant reviews with LLM-based embeddings}
\label{sec:cont-style-restaurants}
\textbf{Data and supervision.}
We also consider a more controlled LLM-generated setting.  We prompt \texttt{Llama-3.1-8B-Instruct} to generate short one-sentence Asian restaurant reviews, each of which contains the restaurant name, a dish, and a short description of how it tastes.  Each review is either positive or negative, and either formal or informal.

We form a data set of triplets as follows:
\begin{enumerate}
    \item To generate a review with the same content, we prompt \texttt{Llama-3.1-8B-Instruct} to paraphrase the review.
    \item For a review with the same style, we again prompt \texttt{Llama-3.1-8B-Instruct} to write about a different dish and restaurant while maintaining the grammatical skeleton and changing the review's sentiment.
\end{enumerate}

\textbf{Qualitative transfers.}
As an illustration, we consider two sentences:
\[
    \text{``Ming's Oriental Bistro serves an exquisite Kung Pao Chicken delightfully''},
\]
and
\[
    \text{``Panda's Palace has an awful Pad Thai that tastes terrible.''}
\]
For each sentence, we rewrite it following the styles of nine target sentences, as summarized in \Cref{tab:rewrite-1,tab:rewrite-2}. In general, each rewritten review maintains the original review's content, namely the restaurant name, dish name, and how the dish tastes, while adopting the target review's grammatical skeleton.

\FloatBarrier

\begin{table}[!htbp]
\centering
\caption{Style transfer often retains restaurant and dish content but can introduce semantic and grammatical artifacts. The table rewrites ``Ming's Oriental Bistro serves an exquisite Kung Pao Chicken delightfully'' in nine target styles.}
\label{tab:rewrite-1}
\begin{tabular}{>{\raggedright\arraybackslash}p{7.5cm} >{\raggedright\arraybackslash}p{7.5cm}}
\hline
Target style                                                             & Rewritten sentence                                                                   \\ \hline
Bangkok Bites' Pad Thai is ridiculously flavorful and addictive.         & Ming's Kungo Bistro's Oriental Chicken is ridiculously flavorful and enticing.       \\ \hline
Donghai Asia's beef rendang tastes like stale cardboard somehow.         & Ming's Kung Pao Bistro Chicken of East Bistro tastes like perfectly stale butter.    \\ \hline
Tasty Wok had amazing Kung Pao chicken that's ridiculously flavorful.    & Ming's Oriental Bistro had Amazing Kung Pao Chicken which is ridiculously flavorful. \\ \hline
Ming Dynasty's Vegetable Lo Mein has an unpleasant, bitter taste.        & Ming's Oriental Bistro Kungo Pao Chicken has an exotic, rich flavor.                 \\ \hline
I had terrible beef at Wong's which was very salty.                      & I had Mingko Kung Pao at Mingko Bistro's Oriental, was extremely gluten.             \\ \hline
I ordered Won Ton soup at Lucky Wok Noodle House, it's gross.            & I have eaten Ming Kung Pao at Oriental's Jiang Bistro, it is delicious.              \\ \hline
Dumpling Heaven serves super tender beef potstickers that taste amazing. & Ming's Oriental Bistro serves exquisite Kung Pao Chicken which tastes wonderful.     \\ \hline
Tokyo Wok's General Tso's Chicken tastes extremely over salted.          & Ming's Oriental Bistro Kungo Pao Chicken has excessively salted.                     \\ \hline
Panda's Palace has awful Pad Thai that tastes terrible.                  & Ming's Oriental Bistro has awesome Kung Pao Chicken that tastes fabulous.            \\ \hline
\end{tabular}
\end{table}

\begin{table}[!htbp]
\centering
\caption{Style transfer often retains restaurant and dish content but can introduce semantic and grammatical artifacts. The table rewrites ``Panda's Palace has an awful Pad Thai that tastes terrible'' in nine target styles.}
\label{tab:rewrite-2}
\begin{tabular}{>{\raggedright\arraybackslash}p{7.5cm} >{\raggedright\arraybackslash}p{7.5cm}}
\hline
Target style                                                                                  & Rewritten sentence                                                      \\ \hline
Ming's Oriental Bistro serves an exquisite Kung Pao Chicken delightfully. & Panda's Palace serves an exquisite Pad Thai terribly.                   \\ \hline
Bangkok Bites' Pad Thai is ridiculously flavorful and addictive.                              & Panda's Pad Thai's Palace is ridiculously flavored and annoying.        \\ \hline
Donghai Asia's beef rendang tastes like stale cardboard somehow.                              & Pad Thai's Panda Pad Thai's tastes like old cardboard by the way.       \\ \hline
Tasty Wok had amazing Kung Pao chicken that's ridiculously flavorful.                         & Panda's Palace had incredible Pad Thai which is ridiculously flavorful. \\ \hline
Ming Dynasty's Vegetable Lo Mein has an unpleasant, bitter taste.                             & Pad Thai's Panda Palace's Pad Thai has an unpleasant, thick taste.      \\ \hline
I had terrible beef at Wong's which was very salty.                                           & I had the Pad Thai's Panda Palace was extremely salty.                  \\ \hline
I ordered Won Ton soup at Lucky Wok Noodle House, it's gross.                                 & I've tried the Pad Thai's Pad Thai's Palace, it is disgusting.          \\ \hline
Dumpling Heaven serves super tender beef potstickers that taste amazing.                      & Panda's Palace serves incredible Pad Thai that tastes amazing.          \\ \hline
Tokyo Wok's General Tso's Chicken tastes extremely over salted.                               & Panda Pad Thai's Pad Palace has an extremely salty taste.               \\ \hline
\end{tabular}
\end{table}

\textbf{Sentiment preservation.}
To further demonstrate the content-style separation algorithm's success, we
take ten representative reviews.  For each review, we change its style
following each of 100 target reviews.  We measure each rewritten review's
sentiment using the following Hugging Face model:
\begin{center}
\texttt{tabularisai/multilingual-sentiment-analysis}
\end{center}
Scores of $-2$ and $+2$ indicate ``very negative'' and ``very positive,''
respectively.  The average sentiment of the rewritten reviews for each
original review is shown in \Cref{tab:avg-sent}. In general, sentiment
polarity is preserved, although its magnitude and isolated cases can change.
Thus, the content--style separation works reasonably well on average but is
not exact for every rewrite.

\begin{table}[!htbp]
\centering
\caption{Positive and negative sentiment are generally preserved across target styles, with exceptions. The table gives the average sentiment of each selected review after transfer to 100 target styles.}
\label{tab:avg-sent}
\begin{tabular}{lc}
\hline
Original sentence                                                                             & Avg. sentiment \\ \hline
\multicolumn{1}{l}{Tasty Wok had amazing Kung Pao chicken that's ridiculously flavorful.}       & $+0.94$          \\ \hline
{Ming's Oriental Bistro serves an exquisite Kung Pao Chicken delightfully.}                     & $+0.75$          \\ \hline
Dumpling Heaven serves super tender beef potstickers that taste amazing.                        & $+0.65$          \\ \hline
Tokyo Wok's General Tso's Chicken tastes extremely over salted.                                 & $+0.53$          \\ \hline
Bangkok Bites' Pad Thai is ridiculously flavorful and addictive.                                & $+0.45$          \\ \hline
I ordered Won Ton soup at Lucky Wok Noodle House, it's gross.                                   & $-0.07$          \\ \hline
I had terrible beef at Wong's which was very salty.                                             & $-0.26$          \\ \hline
Donghai Asia's beef rendang tastes like stale cardboard somehow.                                & $-0.37$          \\ \hline
Panda's Palace has awful Pad Thai that tastes terrible.                                         & $-0.43$          \\ \hline
Ming Dynasty's Vegetable Lo Mein has an unpleasant, bitter taste.                               & $-0.62$          \\ \hline
\end{tabular}
\end{table}

Across these examples, the goal is the same.  We want the MIF to find the most outcome-relevant feature among the parts of the treatment that one could plausibly change.

\clearpage
\section[Additional results for Experiment I: Empirical studies on text-based treatments]{Additional results for Experiment I:\texorpdfstring{\\}{ }Empirical studies on text-based treatments}
\label{app:text-details}
\label{sec:app-text}

This appendix contains the details and remaining results for every text study in \Cref{sec:emp-text}. It gives the complete data-generating processes and implementation choices for the GYAFC, ParaDetox, and food-and-drink studies; all remaining nudging trajectories and qualitative examples; the no-covariate formality diagnostic; and the complete topic-model comparison.

\subsection{GYAFC: User comments and response helpfulness}
\label{sec:text-gyafc}

\textbf{Data and question.}
We begin with a semi-synthetic experiment using the Family and Relationships domain of the GYAFC dataset~\citep{rao2018dear}. We reinterpret each sentence as a user comment in an online discussion or advice forum. The treatment $A$ is the comment text, and the outcome $Y$ is a response helpfulness score, measuring the quality or usefulness of the response that the comment receives. Although the outcome is simulated, this setup corresponds to a common question of interest: \textit{what textual features of a user comment are associated with more helpful responses, after accounting for the context in which the comment appears?}

Let $S_{\mathrm{formal}}\in\{0,1\}$ denote the style label of the comment, where $S_{\mathrm{formal}}=0$ denotes an informal comment and $S_{\mathrm{formal}}=1$ denotes a formal comment. Let $X\in\{0,1\}$ denote the discussion context: $X=0$ corresponds to a casual discussion thread, while $X=1$ corresponds to an advice-seeking or support thread. The context $X$ is observed by the analyst but is not itself part of the text embedding. In the data-generating process, $X$ affects both the distribution of comment style and the response helpfulness score, so it plays the role of a confounder. Except where a diagnostic states otherwise, the semi-synthetic text outcome models use independent error~\(\xi\sim\mathcal N(2,1)\).

\textbf{A shared formality direction.}
We first consider a setting in which formal comments receive more helpful responses in both contexts. The data generating process, denoted Scenario 1, is:
\begin{enumerate}
    \item $X\in\{0,1\}$ with equal probability.
    \item If $X=0$, $A$ is formal with probability $0.3$ and informal with probability $0.7$. If $X=1$, $A$ is formal with probability $0.7$ and informal with probability $0.3$.
    \item $Y=10S_{\mathrm{formal}}+10X+\xi$.
\end{enumerate}
In this scenario, formal comments have higher response helpfulness than informal comments within both contexts. Therefore, the learned $\widehat f(A,X)$ should be close to $0.9$ for formal comments and close to $0.1$ for informal comments, regardless of $X$.

On the held-out test set, panel (a) of \Cref{fig:exp1-plot} confirms this behavior. The average learned $\widehat f(A,X)$ values for informal and formal comments are approximately $0.19$ and $0.82$ when $X=0$, and $0.26$ and $0.84$ when $X=1$. Hence, the learned representation recovers the outcome-relevant textual feature: formality is associated with higher response helpfulness in both contexts.

\textbf{A context-dependent formality direction.}
We next consider a setting in which the useful style depends on the discussion context. The data generating process, denoted Scenario 2, is:
\begin{enumerate}
    \item $X\in\{0,1\}$ with equal probability.
    \item If $X=0$, $A$ is formal with probability $0.3$ and informal with probability $0.7$. If $X=1$, $A$ is formal with probability $0.7$ and informal with probability $0.3$.
    \item $Y=10S_{\mathrm{formal}}(2X-1)+10X+\xi$.
\end{enumerate}
Here, the direction of the style effect changes across contexts. In casual discussion threads, informal comments receive more helpful responses; in advice-seeking or support threads, formal comments receive more helpful responses. Thus, a valid algorithm should not simply discover ``formality'' globally. Instead, it should discover a context-dependent textual feature.

On the same held-out test set, panel (b) of \Cref{fig:exp1-plot} shows that the learned $\widehat f(A,X)$ adapts to this context dependence. When $X=0$, the average learned value is high for informal comments and low for formal comments, approximately $0.81$ versus $0.22$. When $X=1$, the pattern reverses: the average learned value is approximately $0.28$ for informal comments and $0.85$ for formal comments. This demonstrates that the variance-weighted causal contrast learns the text direction associated with larger outcomes relative to the context-specific baseline, rather than merely learning a fixed style label.

\textbf{Nudging.}
We further probe the learned direction through embedding-based nudging. Starting from an initial comment, we transform its SONAR embedding in the direction that increases $\widehat f(A,X)$ while holding the context $X$ fixed. Since the learned $\widehat f(A,X)$ values are often close to the bounds $0.1$ and $0.9$, we apply a temperature-smoothed sigmoid during nudging by replacing $\sigma(\cdot)$ with $\sigma(\cdot/\tau)$. We also normalize the gradients to have mean $0$ and standard deviation $1$. In all experiments, we apply $10$ nudging steps with step size $\lambda=0.2$.

\Cref{fig:exp1-nudging} summarizes the nudging trajectories. In Scenario 1, nudging consistently increases the average formality score of the comments, from approximately $0.2$ at the initial texts to nearly $1.0$ after $10$ iterations. This agrees with the data-generating process, where formal comments are associated with higher response helpfulness in both contexts. In Scenario 2, the nudging direction depends on the context. For initially informal comments in the advice-seeking context, corresponding to $(S_{\mathrm{formal}},X)=(0,1)$, nudging increases average formality from approximately $0.3$ to $0.8$. In contrast, for initially formal comments in the casual discussion context, corresponding to $(S_{\mathrm{formal}},X)=(1,0)$, nudging decreases average formality from approximately $0.9$ to $0.5$. These trajectories show that the learned intervention is not simply pushing all texts toward formality; instead, it moves texts toward the style associated with higher response helpfulness within the given context.

\begin{table}[t]
    \centering
    \small
    \begin{tabularx}{\textwidth}{>{\raggedright\arraybackslash}p{0.18\textwidth}>{\raggedright\arraybackslash}X>{\raggedright\arraybackslash}X}
        \toprule
        Context & Low $\widehat f(A,X)$ examples & High $\widehat f(A,X)$ examples \\
        \midrule
        Advice/support thread & \textit{THE LONGEST WAS 7 HOURS} \newline \textit{it would have gone really easy.}  \newline \textit{But i thought I had your heart?} & \textit{If that is the case, then you truly do love her.} \newline \textit{I believe that deep down no one cares about others, and they prefer to wallow in self-pity.} \newline \textit{Personally, I would not do it since you have known him only six months.}  \\
        \midrule
        Casual thread & \textit{that's all I'm sayin'  ppl back me up on this 1 !!}  \newline \textit{men seem to have a weird way of showing their hurt.} \newline \textit{Um...okay, I HATE online relationships, just to let you know.}  & \textit{No, you know when you are ready for whatever you are talking about. You are your own person.}  \newline \textit{If you were in his shoes, you may very well end up like that.}  \newline \textit{Men need to control their desires and use their brains.} \\
        \midrule
        \bottomrule
    \end{tabularx}
    \caption{High-scoring comments are formal and low-scoring comments are informal in both contexts in Scenario 1. The table gives representative comments.}
    \label{tab:scenario1-low-high-examples}
\end{table}

\begin{table}[t]
    \centering
    \small
    \begin{tabularx}{\textwidth}{>{\raggedright\arraybackslash}p{0.18\textwidth}>{\raggedright\arraybackslash}X>{\raggedright\arraybackslash}X}
        \toprule
        Context & Low $\widehat f(A,X)$ examples & High $\widehat f(A,X)$ examples \\
        \midrule
        Advice/support thread & \textit{now i dont know if that is true or not.} \newline \textit{Lubrication= KY jelly or something similar...}  \newline \textit{it would have gone really easy.}  & \textit{I am sorry for both of you.} \newline \textit{I recommend speaking up and saying there's something wrong.} \newline \textit{Raising your children in that kind of situation will influence their future decisions.} \\
        \midrule
        Casual thread & \textit{Please enjoy being young and only 17 years old while pursuing a relationship.} \newline \textit{Obviously a man will desire sexual intercourse.} \newline \textit{Give your credit card details, too.}  & \textit{to protect them from getting hurt}  \newline \textit{(i'm always sure to return the favor)} \newline \textit{they just want for us to get confused!} \\
        \midrule
        \bottomrule
    \end{tabularx}
    \caption{The high-score formality direction reverses with context in Scenario 2: informal in casual threads and formal in advice-seeking threads. The table gives representative comments.}
    \label{tab:scenario2-low-high-examples}
\end{table}

\textbf{Qualitative interpretation.}
Finally, \Cref{tab:scenario1-low-high-examples,tab:scenario2-low-high-examples} provide qualitative examples of the textual directions discovered by the learned $\widehat f(A,X)$. In Scenario 1, high $\widehat f(A,X)$ comments are generally more polished and formal in both contexts, while low $\widehat f(A,X)$ comments tend to be more informal, abbreviated, or conversational. This agrees with the data generating process, where formality is associated with larger outcomes regardless of context. In Scenario 2, however, the interpretation of high $\widehat f(A,X)$ depends on the context. In advice/support threads, high $\widehat f(A,X)$ comments are again more formal and carefully phrased, such as comments that apologize, recommend an action, or explain consequences. In casual threads, the pattern reverses: high $\widehat f(A,X)$ comments are more informal and conversational, while low $\widehat f(A,X)$ comments are comparatively more formal. These examples show that the learned intervention does not simply recover a global formality direction. Instead, it identifies the text style associated with larger outcomes relative to the context-specific baseline, as intended by the variance-weighted causal contrast.

\FloatBarrier

\subsubsection{No-covariate formality diagnostic}
\label{sec:formality-diagnostic}

\textbf{Purpose.}
We evaluate how the nudging algorithm in \Cref{sec:interpret} performs in practice.

\textbf{Design.}
We illustrate the diagnostic on the Family and Relationships domain of the GYAFC data set~\citep{rao2018dear}.  We use a no-covariate semi-synthetic outcome
\[
    Y = 10S_{\mathrm{formal}} + 20 + 10\xi,
    \qquad
    \xi\sim\mathcal N(0,1),
\]
where \(S_{\mathrm{formal}}=1\) denotes a formal sentence and \(S_{\mathrm{formal}}=0\) denotes an informal sentence.  The only outcome-relevant text feature is formality.  We learn \(\widehat f(A)\in[0.1,0.9]\) by maximizing the sample analogue of the variance-weighted causal contrast
\[
    \widetilde\Delta_X(f)=\mathbb E[f(A)\{Y-\mathbb E(Y)\}].
\]
We then perform ten nudging iterations in the direction of increasing \(\widehat f(A)\), using gradient normalization, step size \(0.1\), and temperature \(50\).  The SONAR representation~\citep{duquenne2023sonar} is useful here because it provides both text-to-embedding and embedding-to-text maps. We score each decoded text with \texttt{s-nlp/xlmr\_formality\_classifier}, a Hugging Face evaluator that returns a formality score in \([0,1]\), with larger values indicating more formal text.

\textbf{Results.}
\Cref{fig:nudging} gives two checks.  First, the learned \(\widehat f\) values are concentrated near the boundaries: formal sentences mostly receive high scores, and informal sentences mostly receive low scores.  Second, nudging increases the average evaluator score.  \Cref{tab:nudging-examples} shows five representative before-and-after texts. The nudged texts generally preserve their main semantic content while moving toward more formal phrasing, although some introduce decoding artifacts.

\begin{figure}[!htbp]
    \centering
    \begin{subfigure}[t]{0.47\textwidth}
        \vspace{0pt}
        \centering
        \footnotesize
        \renewcommand{\arraystretch}{2.2}
        \begin{tabular}{lcc}
            \toprule
            & Formal text & Informal text \\
            \midrule
            $\widehat f \leq 0.12$ 
                & $0.041$ & \textbf{0.780} \\
            $0.12 < \widehat f < 0.88$ 
                & $0.169$ & $0.155$ \\
            $\widehat f \geq 0.88$ 
                & \textbf{0.789} & $0.065$ \\
            \bottomrule
        \end{tabular}
        \caption{Formal and informal texts concentrate in opposite score categories.}
        \label{fig:nudging-f-categories}
    \end{subfigure}
    \hfill
    \begin{subfigure}[t]{0.47\textwidth}
        \vspace{0pt}
        \centering
        \includegraphics[width=0.95\textwidth]{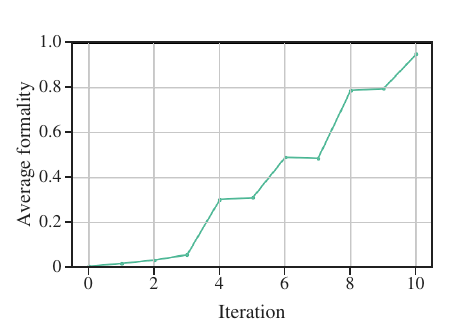}
        \caption{Nudging steadily increases average formality.}
        \label{fig:nudging-formality-iter}
    \end{subfigure}
    \caption{Learned scores separate formal from informal text, and nudging increases external formality. Panel (a) shows that formal texts concentrate at large \(\widehat f\) and informal texts at small \(\widehat f\); panel (b) shows the average formality trajectory. Higher is more formal.}
    \label{fig:nudging}
\end{figure}

\begin{table}[t]
    \centering
    \small
    \begin{tabularx}{\textwidth}{@{}>{\raggedright\arraybackslash}X >{\raggedright\arraybackslash}X@{}}
        \toprule
        \textbf{Original} & \textbf{Nudged} \\
        \midrule
        \textit{but if you think they are youre prob right...}
        &
        \textit{If you think that they are, however, you are likely to be right.}
        \\
        \midrule
        \textit{if u dnt u'll rgret it}
        &
        \textit{If you do not, you will be sorry for it.}
        \\
        \midrule
        \textit{yes what were you thinking  o my god}
        &
        \textit{True, what were you thinking of, as to my God?}
        \\
        \midrule
        \textit{eeek... marriage your thinking about this too much.}
        &
        \textit{Looking at marriage, you are thinking about this too much.}
        \\
        \midrule
        \textit{umm no.. stop asking dumb questions!}
        &
        \textit{Humpty, no. Stop asking stupid questions.}
        \\
        \bottomrule
    \end{tabularx}
    \caption{Nudging generally preserves the main semantic content while increasing formality, although individual decoded sentences can still contain artifacts. The table gives five representative texts before and after nudging.}
    \label{tab:nudging-examples}
\end{table}

\FloatBarrier

\subsubsection{Why topics can miss the feature of interest}
\label{sec:topic-model-failure}

\textbf{Topic-based feature class.}
The same formality example also illustrates why the search space matters. A topic model can split texts by relationship content, but formality cuts across topics. To make this comparison concrete, we fit LDA models~\citep{blei2003latent} with \(K\in\{2,3,\ldots,8\}\) topics on the GYAFC training sentences. For each fitted model, let \(J_K(A)\in\{1,\ldots,K\}\) be the assigned topic. We then search over nonempty subsets \(\mathcal J\subset\{1,\ldots,K\}\) and define
\[
    f_{\mathcal J}(A)
    =
    0.1+0.8\,\mathbf 1\{J_K(A)\in\mathcal J\}.
\]
This maps a deterministic topic split into the same score range \([0.1,0.9]\) used by the MIF.  The selected split maximizes the training analogue of
\[
    \widetilde\Delta_X(f_{\mathcal J})
    =
    \mathbb E\left[
        f_{\mathcal J}(A)\{Y-\mathbb E(Y)\}
    \right].
\]
Equivalently,
\[
    \widetilde\Delta_X(f_{\mathcal J})
    =
    0.8\,c_{\mathcal J}(1-c_{\mathcal J})
    \left\{
        \mathbb E(Y\mid J_K(A)\in\mathcal J)
        -
        \mathbb E(Y\mid J_K(A)\notin\mathcal J)
    \right\},
\]
where
\[
    c_{\mathcal J}
    =
    \mathbb P\{J_K(A)\in\mathcal J\}.
\]
This is the best contrast available inside the topic-based feature class.

\textbf{Held-out comparison.}
The selected topic model has \(K=5\) and \(\mathcal J=\{3,5\}\).  Its test-set variance-weighted causal contrast is \(0.24\), far below the contrast \(1.4\) achieved by the embedding-based MIF and the theoretical benchmark \(2.0\) under a perfect formality split with
\[
    \mathbb P(S_{\mathrm{formal}}=0)=\mathbb P(S_{\mathrm{formal}}=1)=1/2.
\]
The representative words of the five topics are:
\begin{itemize}
    \item Topic 1: \textit{like know time guy girls girl bad guys better really}
    \item Topic 2: \textit{love good ask way does like feel friends talk really}
    \item Topic 3: \textit{just men women say man sex maybe ask yes friend}
    \item Topic 4: \textit{don want just think people dont know make love sure}
    \item Topic 5: \textit{tell love relationship married believe old depends years boyfriend man}
\end{itemize}
The selected topics are interpretable as relationship themes, but they do not isolate formality.  This is the point of the example: when the outcome-relevant feature is stylistic, syntactic, or otherwise spread across semantic topics, a topic model may be too coarse a feature class.

\subsection{ParaDetox: Toxicity and platform response priorities}

\label{sec:text-toxic}

\textbf{Data and question.}
We next consider a semi-synthetic experiment in which the outcome-relevant textual feature is toxicity rather than formality. We use the ParaDetox dataset \citep{logacheva2022paradetox}, which contains paired toxic and neutralized versions of online comments. We reinterpret each text as a user comment submitted to an online platform. The treatment $A$ is the comment text, and the outcome $Y$ is a platform response priority score. Depending on the context, a larger priority score may correspond either to promoting constructive engagement or to prioritizing moderation review. This gives a meaningful question: \textit{what textual properties of a user comment should drive platform response priorities, after accounting for the context in which the comment appears?}

Let $Q\in\{0,1\}$ denote the toxicity label of the comment, where $Q=0$ denotes a neutral comment and $Q=1$ denotes a toxic comment. Let $X\in\{0,1\}$ denote the discussion context: $X=0$ corresponds to a regular community discussion thread, while $X=1$ corresponds to a moderation-sensitive thread. As in \Cref{sec:text-gyafc}, $X$ affects both the distribution of the text feature and the outcome, so it plays the role of a confounder. We use the same variance-weighted causal contrast and estimation algorithm as in the previous experiment, learning $\widehat f(A,X)\in[0.1,0.9]$ to identify text-context pairs with larger outcomes relative to the context-specific baseline.

\textbf{A shared toxicity direction.}
We first consider a setting in which neutral comments receive larger platform response priority scores in both contexts. The data generating process, denoted Scenario 1, is:

\begin{enumerate}

    \item $X\in\{0,1\}$ with equal probability.

    \item If $X=0$, $A$ is toxic with probability $0.3$ and neutral with probability $0.7$. If $X=1$, $A$ is toxic with probability $0.7$ and neutral with probability $0.3$.

    \item $Y=20+10(1-Q)+10X+\xi$.

\end{enumerate}

In this scenario, neutral comments have larger outcomes than toxic comments within both contexts. Therefore, the learned $\widehat f(A,X)$ should be close to $0.9$ for neutral comments and close to $0.1$ for toxic comments, regardless of $X$.

On the held-out test set, panel (a) of \Cref{fig:tox-plot} confirms this behavior. The average learned $\widehat f(A,X)$ values for neutral and toxic comments are approximately $0.88$ and $0.13$ in community discussion threads, and $0.88$ and $0.11$ in moderation-sensitive threads. Hence, the learned representation recovers the intended outcome-relevant feature: lower toxicity is associated with larger outcomes in both contexts.

\textbf{A context-dependent toxicity direction.}
We next consider a setting in which the direction of the toxicity effect depends on the discussion context. The data generating process, denoted Scenario 2, is:

\begin{enumerate}

    \item $X\in\{0,1\}$ with equal probability.

    \item If $X=0$, $A$ is toxic with probability $0.3$ and neutral with probability $0.7$. If $X=1$, $A$ is toxic with probability $0.7$ and neutral with probability $0.3$.

    \item $Y=20+10Q(2X-1)+10X+\xi$.

\end{enumerate}

Here, the direction of the toxicity effect changes across contexts. In regular community discussion threads, neutral comments receive larger platform response priority scores. In moderation-sensitive threads, toxic comments receive larger scores because they are more urgent for moderation review. Thus, a valid algorithm should not simply discover ``toxicity'' globally. Instead, it should discover a context-dependent textual direction.

On the same held-out test set, panel (b) of \Cref{fig:tox-plot} shows that the learned $\widehat f(A,X)$ adapts to this context dependence. When $X=0$, the average learned value is high for neutral comments and low for toxic comments, approximately $0.87$ versus $0.16$. When $X=1$, the pattern reverses: the average learned value is approximately $0.15$ for neutral comments and $0.89$ for toxic comments. This demonstrates that the variance-weighted causal contrast objective learns the text direction associated with larger outcomes relative to the context-specific baseline, rather than merely learning a fixed toxicity label.

\textbf{Nudging.}
We again probe the learned direction through embedding-based nudging. Starting from an initial comment, we transform its SONAR embedding in the direction that increases $\widehat f(A,X)$ while holding the context $X$ fixed. We use the same nudging algorithm as in \Cref{sec:text-gyafc}, including the same temperature smoothing and gradient normalization, and apply $10$ nudging steps with step size $\lambda=0.1$. We evaluate toxicity using Detoxify \citep{Detoxify}.

\Cref{fig:tox-nudging} summarizes the nudging trajectories. In Scenario 1, nudging toxic comments in the direction of increasing $\widehat f(A,X)$ substantially decreases their average toxicity score, from approximately $0.9$ at the initial texts to approximately $0.1$ after $10$ iterations. This agrees with the data generating process, where neutral comments are associated with larger outcomes in both contexts. In Scenario 2, the nudging direction depends on the context. For neutral comments in moderation-sensitive threads, corresponding to $(Q,X)=(0,1)$, nudging increases average toxicity because toxic comments receive larger moderation-priority scores in this context. In contrast, for toxic comments in community discussion threads, corresponding to $(Q,X)=(1,0)$, nudging decreases average toxicity because neutral comments receive larger outcomes in this context. These trajectories show that the learned intervention does not simply detoxify all text. Instead, it moves text toward the toxicity direction associated with larger outcomes within the given context.

\begin{figure}[t]
    \centering
    \includegraphics[width=0.80\textwidth]{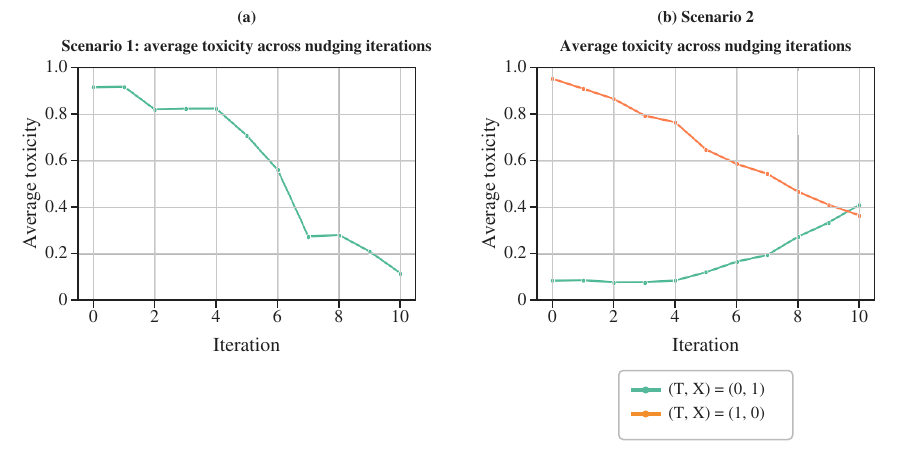}
    \caption{Nudging follows both the shared and context-dependent toxicity directions. In Scenario 1, it decreases toxicity because neutral comments have larger outcomes in both contexts. In Scenario 2, it decreases toxicity in community threads but increases toxicity in moderation-sensitive threads, matching the data-generating process. Higher is more toxic.}
    \label{fig:tox-nudging}
\end{figure}

\begin{table}[t]
    \centering
    \small
    \begin{tabularx}{\textwidth}{>{\raggedright\arraybackslash}p{0.18\textwidth}>{\raggedright\arraybackslash}X>{\raggedright\arraybackslash}X}
        \toprule
        Context & Low $\widehat f(A,X)$ examples & High $\widehat f(A,X)$ examples \\
        \midrule
        Community discussion & \textit{- \_ - lmao damn boy its all his fault !} \newline \textit{what the fuck does that have to do with the left not caring about the truth ?} \newline \textit{hey you know what ? fuck off} & \textit{i gotta stop caring so much nobody ever cares about me} \newline \textit{The kind of explanation is that for something you want people to be aware of.} \newline \textit{This makes religious people look unsuitable.} \\
        \midrule
        Moderation-sensitive & \textit{faxx but sometimes you still get pregnant why you on it that shit can fuck you up period .} \newline \textit{how does all this crap add up ?} \newline \textit{you cant be this stupid .} & \textit{i just told you tomar} \newline \textit{Donald is true to his word:.} \newline \textit{funny thing is that saakashvili is right now 11 tv} \\
        \midrule
        \bottomrule
    \end{tabularx}
    \caption{High-scoring comments are neutral and low-scoring comments are toxic in both contexts in Scenario 1. The table gives representative comments.}
    \label{tab:tox-scenario1-low-high-examples}
\end{table}

\begin{table}[t]
    \centering
    \small
    \begin{tabularx}{\textwidth}{>{\raggedright\arraybackslash}p{0.18\textwidth}>{\raggedright\arraybackslash}X>{\raggedright\arraybackslash}X}
        \toprule
        Context & Low $\widehat f(A,X)$ examples & High $\widehat f(A,X)$ examples \\
        \midrule
        Community discussion & \textit{the nra sure as fuck doesn 't .} \newline \textit{surgeons are bastards but they are shit - hot at their jobs} \newline \textit{that 's how we roll , alt as fuck} & \textit{i gotta stop caring so much nobody ever cares about me} \newline \textit{The kind of explanation is that for something you want people to be aware of.} \newline \textit{This makes religious people look unsuitable.} \\
        \midrule
        Moderation-sensitive & \textit{Which was my point?} \newline \textit{Not being offensive, just looking to understand} \newline \textit{Come on.} & \textit{that my friend , was fucking witty as shit} \newline \textit{fuck both of these anti - social assholes .} \newline \textit{the fools running the city of eugene give fools a bad name .} \\
        \midrule
        \bottomrule
    \end{tabularx}
    \caption{The high-score toxicity direction reverses with context in Scenario 2: neutral in community threads and toxic in moderation-sensitive threads. The table gives representative comments.}
    \label{tab:tox-scenario2-low-high-examples}
\end{table}

\textbf{Qualitative interpretation.}
Finally, \Cref{tab:tox-scenario1-low-high-examples,tab:tox-scenario2-low-high-examples} provide qualitative examples of the textual directions discovered by the learned $\widehat f(A,X)$. In Scenario 1, low $\widehat f(A,X)$ comments are generally more toxic or hostile, while high $\widehat f(A,X)$ comments are comparatively more neutral. This agrees with the data generating process, where neutral comments are associated with larger outcomes regardless of context. In Scenario 2, however, the interpretation of high $\widehat f(A,X)$ depends on the context. In community discussion threads, high $\widehat f(A,X)$ comments are more neutral, whereas in moderation-sensitive threads, high $\widehat f(A,X)$ comments are more toxic and therefore correspond to higher moderation priority. These examples illustrate that the learned intervention identifies the text direction associated with larger outcomes relative to the context-specific baseline, rather than simply learning a global toxicity score.

\FloatBarrier

\subsection{Food and drink review helpfulness}
\label{sec:text-food}
\textbf{Data and question.}
Finally, we consider a semi-synthetic experiment using the food and drink review data from \Cref{sec:cont-style-sep}. Each text is a user review of a food or drink item. The treatment $A$ is the review text, and the outcome $Y$ is a review helpfulness score, measuring how useful the review is to future customers. This setup corresponds to following question: \textit{what textual features in a review are associated with higher helpfulness, after accounting for the type of item being reviewed?}

Let $X\in\{0,1\}$ denote the item category, where $X=0$ corresponds to food and $X=1$ corresponds to drink. Let $S_{\mathrm{sentiment}}\in\{0,1\}$ denote the sentiment of the review, where $S_{\mathrm{sentiment}}=0$ denotes a negative review and $S_{\mathrm{sentiment}}=1$ denotes a positive review. The item category $X$ affects both the distribution of review sentiment and the outcome, and therefore plays the role of a confounder. As in \Cref{sec:text-gyafc,sec:text-toxic}, we use the variance-weighted causal contrast and learn $\widehat f(A,X)\in[0.1,0.9]$ to identify text-context pairs whose outcomes are large relative to the category-specific baseline.

\textbf{First parameterization: feature scoring with context.} We first consider the parameterization $f(A,X)$, which allows the learned intervention to depend on both the review text and the item category. We study two scenarios. In Scenario 1, positive reviews are more helpful for both food and drink items:
\begin{enumerate}
    \item $X\in\{0,1\}$ with equal probability.
    \item The review sentiment $S_{\mathrm{sentiment}}$ is associated with the review text $A$.
    \item $Y=10S_{\mathrm{sentiment}}+10X+\xi$.
\end{enumerate}
In this scenario, positive reviews have larger helpfulness scores than negative reviews within both item categories. Therefore, the learned $\widehat f(A,X)$ should be close to $0.9$ for positive reviews and close to $0.1$ for negative reviews, regardless of whether the item is food or drink.

On the held-out test set, panel (a) of \Cref{fig:dgp1-exp2} confirms this behavior. The learned $\widehat f(A,X)$ values are high for positive reviews and low for negative reviews in both categories, showing that the MIF algorithm recovers the sentiment direction associated with larger review helpfulness.

Scenario 2 introduces a category-dependent sentiment effect:
\begin{enumerate}
    \item $X\in\{0,1\}$ with equal probability.
    \item The review sentiment $S_{\mathrm{sentiment}}$ is associated with the review text $A$.
    \item $Y=10S_{\mathrm{sentiment}}(2X-1)+10X+\xi$.
\end{enumerate}
Here, the outcome-relevant sentiment direction changes across categories. For food reviews ($X=0$), negative reviews have larger outcomes, whereas for drink reviews ($X=1$), positive reviews have larger outcomes. Thus, a valid algorithm should not simply discover a global positive-versus-negative sentiment direction; it should instead learn a category-dependent sentiment direction.

On the same held-out test set, panel (b) of \Cref{fig:dgp1-exp2} shows that $\widehat f(A,X)$ adapts to this category dependence. For food reviews, the learned values are higher for negative reviews than for positive reviews. For drink reviews, the pattern reverses, with higher values for positive reviews. This demonstrates that the variance-weighted causal contrast learns the review sentiment associated with larger outcomes relative to the item-category baseline.

\textbf{Second parameterization: feature scoring without context.} We next consider the second parameterization, where the learned intervention depends only on $A$ and not on $X$. That is, we learn $f(A)$ rather than $f(A,X)$. This parameterization is useful when the goal is to discover a single review-style direction that is shared across categories, after separating item content from sentiment style as in \Cref{sec:cont-style-sep}.

To make this setting nontrivial, we consider Scenario 2 and set $\mathbb P(X=1)=0.8$, so drink reviews are more frequent than food reviews. Since positive sentiment is outcome-enhancing for drink reviews and drink reviews dominate the population, the population-optimal style-only rule assigns high values to positive reviews and low values to negative reviews. Under the constraint $f(A)\in[0.1,0.9]$, this corresponds to $f(A)=0.9$ for positive reviews and $f(A)=0.1$ for negative reviews.

On the held-out test set, panel (c) of \Cref{fig:dgp1-exp2} is consistent with this prediction. Because $f(A)$ cannot condition directly on $X$, it recovers the marginally outcome-enhancing sentiment direction, which is positive sentiment in this imbalanced setting. This contrasts with the first parameterization in panel (b), where the learned rule can reverse direction across food and drink categories.

\clearpage
\section[Additional results for Experiment II: Empirical studies on image-based treatments]{Additional results for Experiment II:\texorpdfstring{\\}{ }Empirical studies on image-based treatments}
\label{app:image-details}
\label{sec:app-imag}

This appendix contains the details and remaining results for both image studies in \Cref{sec:emp-image}. It gives the content--style representations, image architectures, reconstruction checks, and remaining perturbation evidence for rotated handwritten digits and cell-body stains.

We now illustrate that the MIF algorithm can be applied beyond text. We consider two image experiments: rotated handwritten digits in \Cref{sec:mnist} and simulated microscopy images of cell body stains in \Cref{sec:cell}. In both cases, the treatment is an image, and the goal is to discover which modifiable visual style feature is associated with larger outcomes after accounting for unmodifiable image content. The outcomes are semi-synthetic, but the setups correspond to meaningful questions: \textit{which acquisition or presentation artifacts affect downstream recognition, review, or quality-control scores?}

\subsection{Rotated MNIST}
\label{sec:mnist}

\textbf{Data and representation.}
We first consider a rotated MNIST experiment~\citep{lecun1998gradient}. Each image is a handwritten digit. The unmodifiable content is the digit identity, while the modifiable style is how the digit is written or presented, including its rotation angle. This setup can be interpreted as a simplified optical character recognition problem: \textit{after accounting for the digit identity, we ask which patterns are associated with larger downstream recognition or review scores.}

We randomly rotate each image by angle~$\alpha$, where $\alpha\sim\mathrm{Uniform}(-70,70)$. To separate digit identity from visual style, we use adversarial autoencoder\footnote{\texttt{https://github.com/johncf/mnist-style}} with the default parameters. Each image is represented by a content embedding $A_{\mathrm{unmodifiable}}$, given by a one-hot vector indicating the digit, and a style embedding $A_{\mathrm{modifiable}}\in\mathbb R^4$. The decoder reconstructs an image from any content-style pair. \Cref{fig:img_example} shows this decomposition: each row fixes the style embedding of the leftmost image and varies the digit identity from $0$ to $9$. The resulting images preserve the visual style, such as rotation and stroke pattern, while changing the digit content. This indicates that the style embedding captures presentation features rather than digit identity.

\begin{figure}[!htbp]
    \centering
    \includegraphics[width=0.72\textwidth]{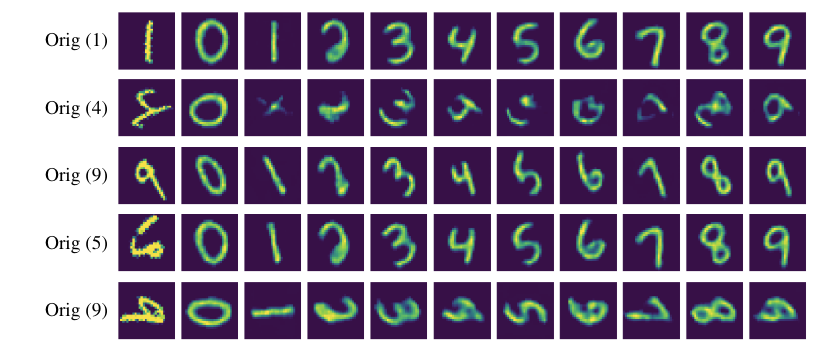}
    \caption{The style embedding preserves rotation and stroke pattern while digit identity changes. Each row fixes the style of the leftmost image and varies the digit from \(0\) to \(9\).}
    \label{fig:img_example}
\end{figure}

\textbf{Outcome model and recovery.}
We then learn the MIF from these image embeddings. Let $X\in\{0,1\}$ denote digit parity, where $X=0$ corresponds to even digits and $X=1$ corresponds to odd digits. With independent error~\(\xi\sim\mathcal N(2,1)\), we generate
\[
    Y = 0.1\alpha(2X-1) + 10X + \xi,
\]
where $\alpha$ is the rotation angle. This data-generating process represents a setting in which the preferred image orientation depends on the digit group: for even digits, clockwise rotations are associated with larger outcomes, while for odd digits, counter-clockwise rotations are associated with larger outcomes. Thus, a valid algorithm should not simply learn a global rotation direction. Instead, it should learn a parity-dependent visual direction.

We learn $\mathbb P(W_{\widehat f}=1\mid A_{\mathrm{modifiable}},X)=\widehat f(A_{\mathrm{modifiable}},X)$ using the style embedding $A_{\mathrm{modifiable}}$ and the parity covariate $X$. \Cref{fig:scatter-image} shows that the learned $\widehat f(A_{\mathrm{modifiable}},X)$ behaves as expected. Among even digits, images with high $\widehat f$ values tend to have more negative rotation angles, while images with low $\widehat f$ values tend to have more positive rotation angles. Specifically, the mean rotation angle is $26.8$ degrees for even-digit samples with $\widehat f<0.2$, compared to $-32.5$ degrees for those with $\widehat f>0.8$. Among odd digits, the pattern reverses: the mean rotation angle is $-21.1$ degrees when $\widehat f<0.2$ and $28.7$ degrees when $\widehat f>0.8$. Thus, the learned intervention captures the visual orientation associated with larger outcomes relative to the digit-parity baseline.

\textbf{Nudging.}
Finally, we probe the learned direction through image nudging. Starting from a given image, we transform its style embedding in the direction of increasing $\widehat f(A_{\mathrm{modifiable}},X)$ while treating the digit identity as unmodifiable. \Cref{fig:nudging-mnist} shows representative nudging trajectories. In general, odd digits are rotated counter-clockwise while even digits are rotated clockwise, consistent with the learned $\widehat f$ and with the data-generating process. This demonstrates that the learned intervention is not only interpretable but also actionable in the image space: it identifies how to modify the visual style while preserving image content.

\FloatBarrier

\subsection{Cell-body stains}
\label{sec:cell}

\textbf{Data.}
We next consider a microscopy image experiment using the simulated BBBC005v1 dataset from the Broad Bioimage Benchmark Collection~\citep{ljosa2012annotated}. We restrict attention to cell body stain images with at most $40$ cells. This experiment can be interpreted as a simplified image quality control problem: \textit{after accounting for the number of cells in an image, we ask which acquisition artifact is associated with larger quality control or review priority scores.}

Here, the unmodifiable content is the number of cells, and the modifiable style is the level of focus blur. We use the same adversarial autoencoder algorithm as in \Cref{sec:mnist}, with architectural adjustments for cell images and a style embedding dimension of $32$. \Cref{fig:img_example_cell} shows the resulting content-style separation. Each row fixes the style of the leftmost image and varies the cell count across selected values. The reconstructed images preserve the blurriness level while changing the number of cells, suggesting that the learned style embedding captures the focus artifact rather than the cell-count content.

\begin{figure}[!htbp]
    \centering
    \includegraphics[width=0.82\textwidth]{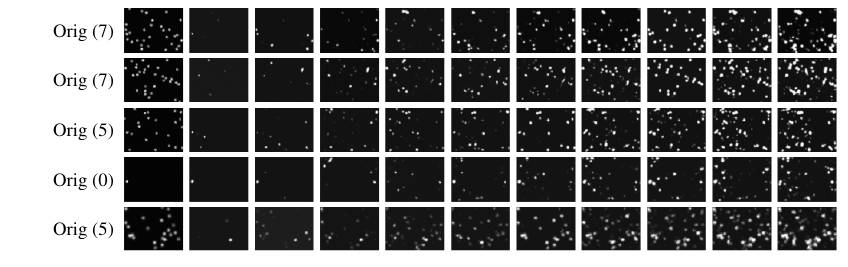}
    \caption{The style embedding approximately preserves blur while the content embedding controls cell count. Each row fixes the style of the leftmost image and varies the number of cells.}
    \label{fig:img_example_cell}
\end{figure}

\textbf{Outcome model and recovery.}
We then learn the MIF from the cell image embeddings. Let $B\in[1,48]$ denote the level of blurriness, and let $X\in\{0,1\}$ indicate whether the image contains fewer than $20$ cells, with $X=1$ if the cell count is below $20$ and $X=0$ otherwise. With independent error~\(\xi\sim\mathcal N(2,1)\), we generate
\[
    Y = 0.1B + 10X + \xi.
\]
In this setup, larger blurriness corresponds to larger quality-control or review-priority scores, regardless of the cell-count group. Thus, after adjusting for $X$, a valid algorithm should learn that high-blur style embeddings are associated with larger outcomes.

We learn $\mathbb P(W_{\widehat f}=1\mid A_{\mathrm{modifiable}},X)=\widehat f(A_{\mathrm{modifiable}},X)$, where $A_{\mathrm{modifiable}}$ is the style embedding. \Cref{fig:scatter-image-cell} shows that larger blurriness levels correspond to larger learned $\widehat f(A_{\mathrm{modifiable}},X)$ values in both cell-count groups. Among images with more than $20$ cells, the mean blurriness level is $15.9$ for samples with $\widehat f<0.2$, compared to $40.3$ for those with $\widehat f>0.8$. Similarly, among images with fewer than $20$ cells, the mean blurriness level is $15.7$ when $\widehat f<0.2$ and $37.1$ when $\widehat f>0.8$. Therefore, the learned intervention recovers the acquisition artifact driving the outcome, rather than merely learning the number of cells.

\begin{figure}[!t]
    \centering
    \includegraphics[width=0.80\textwidth]{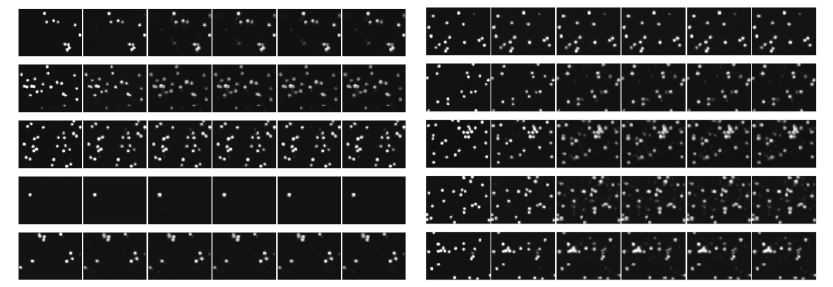}
    \caption{Nudging increases blur while approximately preserving cell count, consistent with the learned $\widehat f(A_{\mathrm{modifiable}},X)$ direction. Columns give iterations \(0\) through \(5\), following the learned \(\widehat f(A_{\mathrm{modifiable}},X)\) direction.}
    \label{fig:nudging-cell}
\end{figure}

\textbf{Nudging and interpretation.}
Finally, \Cref{fig:nudging-cell} shows nudging results for representative cell images. Nudging the style embeddings in the direction of increasing $\widehat f(A_{\mathrm{modifiable}},X)$ makes the images more blurry while preserving the approximate cell counts. This is consistent with the learned relationship between blur and the quality control priority outcome. In this example, nudging should be interpreted as a diagnostic probe of the learned image direction rather than as a recommendation to improve image quality: \textit{the MIF algorithm identifies the visual artifact associated with larger review priority scores.}

\FloatBarrier
\section[Additional results for Experiment III: Empirical studies on dynamic treatment sequences]{Additional results for Experiment III:\texorpdfstring{\\}{ }Empirical studies on dynamic treatment sequences}
\label{app:sequence-details}
\label{sec:app-dyna}

This appendix contains the details and remaining results for all sequence studies in \Cref{sec:emp-sequence}. It gives the complete outcome models, decoding rule, iteration-level tables, timing design, budget interpretation, and prefix-preserving suffix discussion.

\subsection{Short treatment sequences under categorical and sentence representations}
\label{sec:sequence-short}
\subsubsection{Categorical representation}
\textbf{Treatment representation.}
Dynamic treatment sequences arise in longitudinal care, where a patient receives a series of interventions and the treatment history may shape a final health outcome. We begin with a simplified categorical regime of length \(L\in\{1,\ldots,6\}\). We denote the treatment sequence by \(A_{1:L}=(A_1,\ldots,A_L)\), where each event~\(A_\ell\) is one of three treatment or visit types, denoted \(a\), \(b\), and \(c\).

\textbf{No-covariate scenarios.}
We first consider two no-covariate scenarios, each with independent error~\(\xi\sim\mathcal N(0,4)\):
\begin{enumerate}
    \item $Y = 10 + 5N_a(A_{1:L}) + 10N_b(A_{1:L}) + 15N_c(A_{1:L}) + \xi$.
    \item $Y = 100 - 5N_a(A_{1:L}) - 10N_b(A_{1:L}) - 15N_c(A_{1:L}) + \xi$.
\end{enumerate}
Here, \(N_j(A_{1:L})\) denotes the number of \(j\)'s in \(A_{1:L}\), for \(j\in\{a,b,c\}\). Our goal is the same as the text-based setting, which is to learn a feature-scoring function \(\widehat f(A_{1:L}) := \mathbb{P}(W_{\widehat f}=1 \mid A_{1:L})\) that maximizes the raw causal contrast \(\Delta_X(\widehat f)\).

\textbf{Embedding and nudging.}
As before, we represent \(A_{1:L}\) with its SONAR embedding. After learning such a feature-scoring function \(\widehat f\), we nudge \(A_{1:L}\) in the direction of increased \(\widehat f(A_{1:L})\) for five iterations and decode the nudged embeddings via top-$p$ sampling with $p = 0.99$, removing treatment sequences with more than three letters or containing a letter other than \(a\), \(b\), or \(c\). For each scenario, we calculate the average $Y$ value for the nudged treatment at each iteration, as summarized in \Cref{tab:avg-y}. We can see that nudging generally works well and results in nudged treatment sequences with higher values of $Y$ on average.

\textbf{Context-dependent scenario.}
We now consider a case with a binary covariate~$X$ and independent error~\(\xi\sim\mathcal N(0,4)\). The data-generating process is as follows:
\begin{enumerate}
    \item When $X = 0$, we have $Y = 10 + 5N_a(A_{1:L}) + 10N_b(A_{1:L}) + 15N_c(A_{1:L}) + \xi$.
    \item When $X = 1$, we have $Y = 100 - 5N_a(A_{1:L}) - 10N_b(A_{1:L}) - 15N_c(A_{1:L}) + \xi$.
\end{enumerate}
\textbf{Nudging results.}
We now learn a feature-scoring function \(\widehat f(A_{1:L}, X) := \mathbb{P}(W_{\widehat f}=1 \mid A_{1:L}, X)\) maximizing the raw causal contrast \(\Delta_X(\widehat f)=\Psi_{\widehat f}(1)-\Psi_{\widehat f}(0)\) and perform nudging as before (while keeping $X$ fixed). The average $Y$ values for the nudged treatment corresponding to $X = 0$ and $X = 1$ are shown in \Cref{tab:avg-y-v2}. We observe that nudging in this case also works reasonably well.
\begin{table}[!htbp]
\centering
\caption{Nudging raises the endpoint average outcome from \(45.5\) to \(64.7\) in Scenario 1 and from \(64.5\) to \(74.5\) in Scenario 2. The table gives every iteration for both no-covariate scenarios. Higher is better.}
\label{tab:avg-y}
\begin{tabular}{ccccccc}
\hline
Iteration  & 0      & 1      & 2      & 3      & 4      & 5      \\ \hline
Scenario 1 & $45.5$ & $47.2$ & $49.5$ & $54.7$ & $60.1$ & $64.7$ \\ \hline
Scenario 2 & $64.5$ & $65.9$ & $66.7$ & $68.9$ & $71.5$ & $74.5$ \\ \hline
\end{tabular}
\end{table}
\begin{table}[!htbp]
\centering
\caption{Nudging raises the endpoint average outcome for both contexts, although the \(X=1\) trajectory first falls from \(64.5\) to \(62.2\). The table gives every iteration. Higher is better.}
\label{tab:avg-y-v2}
\begin{tabular}{ccccccc}
\hline
Iteration & 0      & 1      & 2      & 3      & 4      & 5      \\ \hline
$X = 0$   & $45.5$ & $47.7$ & $57.4$ & $64.2$ & $72.3$ & $77.5$ \\ \hline
$X = 1$   & $64.5$ & $64.2$ & $62.8$ & $62.2$ & $65.6$ & $68.1$ \\ \hline
\end{tabular}
\end{table}

\textbf{When nudging is informative.} \textit{Since there is no unmodifiable content and the treatment can be freely changed, nudging is not too useful here. A more practical strategy would be to simply examine the learned \(\widehat f(A_{1:L},X)\) or \(\widehat f(A_{1:L})\) (depending on the parameterization used).}

\FloatBarrier

\subsubsection{Sentence representation}
\label{sec:sequence-sentence}
Alternatively, we can represent treatment sequences using natural language. For example, \textit{a b c} becomes \textit{The first treatment is a. The second treatment is b. The third treatment is c.} For both scenarios, we again calculate the average $Y$ value for the nudged treatment at each iteration, as shown in \Cref{tab:treat-seq}.

\begin{table}[!htbp]
\centering
\caption{Sentence-based nudging raises the endpoint average outcome from \(36.2\) to \(48.2\) in Scenario 1 and from \(73.8\) to \(89.8\) in Scenario 2; Scenario 2 peaks at \(89.9\) one iteration earlier. Higher is better.}
\label{tab:treat-seq}
\begin{tabular}{ccccccc}
\hline
Iteration  & 0      & 1      & 2      & 3      & 4      & 5      \\ \hline
Scenario 1 & $36.2$ & $36.2$ & $36.2$ & $37.8$ & $44.3$ & $48.2$ \\ \hline
Scenario 2 & $73.8$ & $73.8$ & $74.0$ & $78.6$ & $89.9$ & $89.8$ \\ \hline
\end{tabular}
\end{table}

\FloatBarrier

\subsection{Treatments that include administration times}
\label{sec:sequence-time}
\textbf{Design.}
We now add administration times to the treatment sequences, motivated by longitudinal medical regimes in which care is delivered at recurring visits---such as prenatal visits---that occur at varying intervals. The causal question is which treatment types and spacings most influence the patient's final outcome. For example, we can have \textit{0 a}, \textit{0 a 18 c 32 a}, or \textit{0 c 5 b 9 b}. We assume the following:
\begin{enumerate}
    \item There can be at most three treatments given.
    \item The first treatment is administered at time~$t = 0$.
    \item The outcome \(Y\), measured at time~\(t=100\), is
    \[
        Y=5N_a(A_{1:L})+10N_b(A_{1:L})+15N_c(A_{1:L})+0.2\,t_p+\xi,
    \]
    where \(t_p\) denotes the shortest interval between treatments and \(\xi\sim\mathcal N(0,4)\) is independent error.
\end{enumerate}
\textbf{Budget behavior.}
Here, we consider learning \(\widehat f(A_{1:L}) := \mathbb{P}(W_{\widehat f} = 1 \mid A_{1:L})\). On the test set, we observe a clear positive correlation between \(\mathbb{E}(Y \mid A_{1:L})\) and \(\widehat f(A_{1:L})\) for any given budget, as shown in \Cref{fig:scatter}. Moreover, as the budget decreases, we further prioritize setting \(W_{\widehat f} = 1\) for treatments with larger values of $Y$.

\textbf{Prefix-preserving search.}
Suppose we have obtained \(\widehat f\) from the data. Given a \textit{prefix treatment sequence}---the care already administered at previous visits---we can modify only the \textit{suffix}---the remaining treatments and administration times still under our control---and evaluate \(\widehat f(A_{1:L})\) for each candidate continuation. This search ranks feasible future regimes while leaving the patient's observed treatment history unchanged.

\end{document}